\documentclass[onecolumn]{IEEEtran}

\usepackage[numbers]{natbib}
\usepackage[utf8]{inputenc} % allow utf-8 input
\usepackage[T1]{fontenc}    % use 8-bit T1 fonts
\usepackage{hyperref}       % hyperlinks
\usepackage{url}            % simple URL typesetting
\usepackage{booktabs}       % professional-quality tables
\usepackage{amsfonts}       % blackboard math symbols
\usepackage{nicefrac}       % compact symbols for 1/2, etc.
\usepackage{microtype}      % microtypography
\usepackage{xcolor}         % colors

\usepackage{hyperref}

\usepackage{wrapfig}
\usepackage{graphicx} % 
\usepackage{wrapfig}
\usepackage{bbm}
\usepackage{amsmath}
\usepackage{amssymb}
\usepackage{amsfonts}
\usepackage{amsthm}
\usepackage{arydshln}
\usepackage{xcolor}
\usepackage{comment}
\usepackage{microtype}
\usepackage{graphicx}
\usepackage{epstopdf}
\usepackage{subcaption}
\usepackage{algorithm}
\usepackage{algorithmic}

\usepackage{tikz}
\usetikzlibrary{arrows.meta,positioning}

\usepackage{amsmath}
\usepackage{amssymb}
\usepackage{mathtools}
\usepackage{amsthm}
\usepackage{enumitem}
\usepackage[capitalize,noabbrev]{cleveref}

\newcommand{\eqrefn}[1]{Eq.~(\ref{#1})}
\newcommand{\attn}{\mathbf{attn}}

\newcommand{\Rb}{\mathbb{R}}
\newcommand{\Pb}{\mathbb{P}}
\newcommand{\Sc}{{\mathcal{S}}}

\newcommand{\Eb}{\mathbb{E}}

\newcommand{\Gc}{\mathcal{G}}
\newcommand{\RM}[1]{\left(\romannumeral#1\right)}
\newcommand{\norm}[1]{\left\|#1\right\|}

\newcommand{\name}{kernel-guided }

\newcommand{\Name}{Kernel-Guided }

\newcommand{\abname}{KG-MI }

\DeclareSymbolFont{ssfletters}{OT1}{cmss}{m}{n}
\DeclareMathSymbol{\ssfPi}{0}{ssfletters}{'005}
\theoremstyle{plain}
\newtheorem{theorem}{Theorem}[section]

\newtheorem{lemma}[theorem]{Lemma}
\newtheorem{corollary}[theorem]{Corollary}
\newtheorem{definition}{Definition}[section]
\newtheorem{assumption}{Assumption}[section]
\newtheorem{property}{Property}
\theoremstyle{remark}
\newtheorem{remark}{Remark}

\newcommand{\yl}[1]{{\color{orange}[Yingbin: #1]}}

\title{Transformers Provably Learn Directed Acyclic Graphs via Kernel-Guided Mutual Information}

\author{\textbf{Yuan Cheng}\IEEEauthorrefmark{1}\IEEEauthorrefmark{2}, 
\textbf{Yu Huang}\IEEEauthorrefmark{1}\IEEEauthorrefmark{3},
\textbf{Zhe Xiong}\IEEEauthorrefmark{4}, \textbf{Yingbin Liang}\IEEEauthorrefmark{5},
and \textbf{Vincent Y. F. Tan}\IEEEauthorrefmark{2}\\
\IEEEauthorblockA{
\IEEEauthorrefmark{2}  ISEP and Department of Mathematics, National University of Singapore, Singapore \\[0.5mm]
\IEEEauthorrefmark{3} Department of Statistics and Data Science, University of Pennsylvania, USA 
\IEEEauthorrefmark{4} Independent Researcher \\[0.5mm]
\IEEEauthorrefmark{5} Department of Electrical and Computer Engineering, The Ohio State University, USA\\[0.5mm]
}
{Emails:\, yuan.cheng@u.nus.edu, yuh42@wharton.upenn.edu, 
bearx6666@gmail.com, liang.889@osu.edu, vtan@nus.edu.sg}
\thanks{\IEEEauthorrefmark{1} Equal contribution.}
}

\begin{document}

\maketitle

\begin{abstract}
Uncovering hidden graph structures underlying real-world data is a critical challenge with broad applications across scientific domains. Recently, transformer-based models leveraging the attention mechanism have demonstrated strong empirical success in capturing complex dependencies within graphs. However, the theoretical understanding of their training dynamics has been limited to tree-like graphs, where each node depends on a single parent. Extending provable guarantees to more general directed acyclic graphs (DAGs)---which involve multiple parents per node---remains challenging, primarily due to the difficulty in designing training objectives that enable different attention heads to separately learn multiple different parent relationships.

In this work, we address this problem by introducing a novel information-theoretic metric: the {\em  kernel-guided mutual information} (KG-MI), based on the $f$-divergence. Our objective combines KG-MI with a multi-head attention framework, where each head is associated with a distinct marginal transition kernel to model diverse parent-child dependencies effectively. 
We prove that, given sequences generated by a $K$-parent DAG, training a single-layer, multi-head transformer via gradient ascent converges to the global optimum in polynomial time. Furthermore, we characterize the attention score patterns at convergence. In addition, when particularizing the $f$-divergence to the KL divergence, the learned attention scores accurately reflect the ground-truth adjacency matrix, thereby provably recovering the underlying graph structure. Experimental results validate our theoretical findings.
\end{abstract}
\section{Introduction}

Latent structures and interdependencies within data are increasingly common in real-world applications across diverse fields such as neuroscience, machine learning, and beyond~\cite{levin1998ecosystems, friston2010free, sutton2018reinforcement}. Uncovering these hidden structures is a crucial research challenge, as understanding the underlying graph relationships not only offers deeper insights into the data but also significantly boosts the performance of downstream tasks~\cite{wu2020comprehensive, gori2005new, scarselli2008graph, velivckovic2018graph, brody2022how}. 
Previous studies have modeled unknown latent dependencies using graphs, framing \emph{structure learning} as the task of identifying the true underlying graph from data generated from the graph. Classical approaches---such as constraint-based algorithms and continuous optimization methods---have achieved remarkable progress in this area~\cite{spirtes2000causation, chickering2002optimal, peters2017elements, koller2009probabilistic, zheng2018dags}.

% Recently, {\bf transformers}~\citep{vaswani2017attention}, which have revolutionized artificial intelligence across many domains~\cite{achiam2023gpt,dosovitskiy2020image,chen2021decision}, %following their remarkable success in natural language processing, 
% have emerged as a promising alternative for structure discovery tasks. Their attention-based mechanism naturally evaluates node relationships, offering a built-in approach to represent causal or dependency structures. Moreover, their flexible attention modules are highly effective at capturing long-range interactions, making them well-suited for recognizing complex data dependencies. Furthermore, self-attention computes pairwise interactions among all nodes in parallel, unlike traditional graph learning approaches that depend on sequential, iterative algorithm design. The self-attention thus results  in lower computational overhead and better scalability on modern hardware.

Recently, \textbf{transformers}~\cite{vaswani2017attention} have revolutionized artificial intelligence across numerous domains~\cite{achiam2023gpt,dosovitskiy2020image,chen2021decision}, and their success has inspired growing interest in applying them to structure discovery tasks. The attention-based mechanism of transformers inherently evaluates relationships between nodes, providing an elegant approach to representing causal or dependency structures. Additionally, their flexible attention modules excel at capturing long-range interactions, making them highly effective in recognizing complex data dependencies. Unlike traditional graph learning methods, which rely on sequential and iterative algorithms for optimization, self-attention computes pairwise interactions among all nodes in parallel. This parallelism leads to lower computational overhead and improved scalability on modern hardware, further enhancing their suitability for demanding applications.

Despite empirical successes demonstrating that transformers perform well across a variety of graph-related tasks~\citep{shehzad2024graph, kim2022pure, dwivedi2020generalization}, the theoretical understanding of why transformers are effective for such tasks remains limited. Recent work has begun to establish the expressive power of transformers in representing simple graph structures in Markovian settings~\citep{rajaraman2024transformers, hu2024limitation, zhou2024transformers}. However, expressive power only shows that there exist transformers capable of capturing graph structures; it does not guarantee that transformers trained via standard procedures will acquire such capabilities. A few more recent studies have investigated the training dynamics of transformers to explain how models trained on graph-structured data can learn to recover the underlying graph structure. 
However, these insights are largely restricted to specific settings: (i) tree-structured graphs with single-parent nodes~\citep{nichani2024transformers}, where analysis focused on single-head attention suffices, and (ii) specialized $n$-gram Markov models~\citep{chenunveiling}, where node dependencies are limited to fixed-distance $n$-parent nodes and the learning process relies solely on feed-forward networks, thereby bypassing the core attention mechanism. 

%As a result, the existing theoretical studies on tree-structured graphs and Markov models do not address Directed Acyclic Graph (DAG) structures, which are more general as follows. (i) In DAGs, nodes can have multiple parents (whereas nodes in tree structures can have only single parents), essential for representing shared dependencies or common submodules. As a salient example, in computational graphs (e.g., TensorFlow), DAGs are used because many operations reuse intermediate results—a tree cannot capture this. (ii) In DAGs, nodes can encode complex long-range causal relationships, whereas Markov models capture only immediate history dependencies.

Existing theoretical studies on the training dynamics of transformers for learning graph structures have predominantly focused on {\em tree-structured} graphs and {\em Markov} models. However, these frameworks fall short when applied to the more general and expressive class of {\bf Directed Acyclic Graphs (DAGs)}. DAGs are a crucial extension, offering a richer and more versatile representation of directed relationships. Specifically, unlike trees where each node has at most one parent, nodes in DAGs can have multiple parents, enabling the modeling of shared dependencies and common substructures---an essential feature for many real-world applications. An example is computational graphs in frameworks like TensorFlow, where DAGs naturally encode complex computation pipelines with reusable intermediate results. Such a structure cannot be adequately captured by trees. Furthermore, DAGs support long-range, intricate causal dependencies, whereas Markov models are limited to immediate or fixed-length historical dependencies. Recognizing and leveraging these unique features of DAGs is vital for advancing the understanding of transformer training dynamics on complex, real-world graph structures.

In this paper, we explore more general DAG  structures and examine the role of the attention mechanism in learning such graph representations. Specifically, our goal is to understand how transformers acquire this capability during training by analyzing their training dynamics through gradient ascent on a suitably chosen objective function. To model DAGs, a natural approach involves leveraging the multi-head attention mechanism, where multiple heads enable each node to attend to several potential parent nodes. However, this framework introduces two significant challenges that have not been fully addressed in prior studies focused on trees and Markov models:
\begin{enumerate}[wide, labelindent=0pt]
    \item[(i)] {\bf Head collapse:} Multiple attention heads may converge to attend to redundant or overlapping parents, leading to negligible attention on some parent nodes. The key question is: {\em How can we ensure that different attention heads attend to distinct parents rather than redundantly focusing on the same one?}
    \item[(ii)]  {\bf Performance metric:} While the mutual information has been a useful metric in previous work for identifying parent nodes, directly applying it to learning general DAGs can lead to head collapse. This raises the important question: {\em Which other information-theoretic metrics could serve as effective training objectives to guide transformers in accurately identifying and attending to all parent nodes within a graph?}
    %How can we design a training objective that effectively guides transformers to accurately identify and attend to all parent nodes within the graph?}
\end{enumerate}

\subsection{Main contributions}
To address the critical challenges outlined above, we introduce a novel framework for learning DAGs with proven performance guarantees. Our approach synergistically combines a multi-head attention mechanism with a newly proposed {\em kernel-guided mutual information} (KG-MI) metric, designed to ensure precise identification of all parent nodes while effectively preventing head collapse. Our key contributions are summarized as follows:
\begin{enumerate}[wide, labelindent=0pt]
    \setlength{\parskip}{0pt} 
    \setlength{\topsep}{0pt}
    \item {\bf New Performance Metric to Prevent Head Collapse:} We introduce $f$-KG-MI (\Cref{Def: Modified MI}), a novel mutual information measure that integrates marginal transition kernel-dependent information into the traditional $f$-mutual information derived from $f$-divergences. This metric explicitly enforces diversity among attention heads by aligning each head with a distinct marginal transition kernel, thereby ensuring each attention head captures unique parent-child relationships within the DAG. This innovation significantly enhances the model's capability to uncover comprehensive DAG structures without the collapse of attention heads into redundant or overlapping information.
    %To ensure different attention heads learn distinct parents in DAGs, without head collapse, we propose $f$-KG-MI (\Cref{Def: Modified MI}), which incorporates the marginal transition kernel-dependent information into the conventional $f$-mutual information (which is derived from the $f$-divergence). Using $f$-KG-MI, we design an objective that explicitly associates different attention heads to distinct marginal transition kernels, ensuring that a multi-head attention layer effectively captures diverse parent-child dependencies.
    %we design an objective that explicitly promotes diversity among attention heads, ensuring that they capture unique parent-child dependencies.
    \item {\bf Training Dynamics and Convergence Guarantee:} %Using estimated versions of the KG-MIs,   we train a one-layer \emph{softmax} attention model with the proposed objective function using gradient ascent. We analyze the training dynamics, prove that the objective function converges to the global optimum with polynomial-time efficiency, and characterize the patterns of attention scores upon convergence (\Cref{Thm: loss convergence,Thm: Attention Concentration}). In particular, when the $f$-divergence is particularized to the Kullback--Leibler (KL) divergence, leading to the use of the KL-KG-MI, the attention scores of the trained transformer converge to the adjacency matrix of the graph, indicating that the model successfully learns the structure of the underlying DAG. Our experiments corroborate our convergence results on the attention scores as in \Cref{fig:atten-score heatmap}, and show that our methods outperform—or at least match—classical baselines on certain DAGs as in \Cref{tab:Performance comparison}.
    By utilizing estimators of the $f$-KG-MI, we train a single-layer \emph{softmax} attention model via gradient ascent. We rigorously analyze the training process, proving that the objective function converges to the global optimum within polynomial time. Moreover, we characterize the resulting attention score patterns at convergence (\Cref{Thm: loss convergence,Thm: Attention Concentration}). Notably, when specializing the $f$-divergence to the Kullback--Leibler (KL) divergence (yielding KL-KG-MI), the trained transformer’s attention scores converge to the adjacency matrix of the underlying DAG. Our empirical results validate these theoretical insights---demonstrating convergence properties (\Cref{fig:atten-score heatmap}) and outperforming or at least matching various baselines (\Cref{tab:Performance comparison}).

    % \item {\bf Impact of Choice of $f$ on Convergence Rate:} \Cref{Thm: meta-population,Thm: meta-empirical} characterize how the convergence rate depends on the sequence length $T$ and \emph{information gap} $\Delta$, where $\Delta$ quantifies the  range at which the $f$-KG-MI varies across different node pairs. Notably, under the same data model, the choice of the function $f$ in KG-MI influences $\Delta$, thereby affecting the convergence rate. These results provide valuable insights into selecting an appropriate $f$-KG-MI to accelerate convergence during training. Our experiments further validate these theoretical findings, as demonstrated in \Cref{fig: compreh f mi}. 

\item 
% {\bf New Technical Developments:} {\bf (i) New concentration analysis:} DAGs are significantly more complex than tree and Markov models studied previously. Consequently, the concentration properties of DAGs—for example, the empirical estimates converging to a stationary distribution—depend heavily on these connection patterns. Our concentration analysis (see Appendix \ref{App: concentration}) is specifically developed to handle these intricate dependency structures.
{\bf Broader $f$-divergence framework:} %Our study leverages the general $f$-divergence, while previous studies (e.g., \citep{nichani2024transformers}) are only applicable to the Pearson's $\chi^2$-mutual information. Since the choice of $f$ also influences the convergence rate, this generalization also provides practical insights on choosing the best possible $f$ among a candidate set of  $f$'s to speed up training. We find that  the $\chi^2$-mutual information is often not the best choice in practice. 
 Unlike prior studies limited to specific divergences such as Pearson's $\chi^2$-divergence~\citep{nichani2024transformers}, our framework adopts the broader class of $f$-divergences. This generalization not only broadens theoretical applicability but also provides practical benefits: by analyzing various functions $f$, practitioners can identify the optimal one to speed up training. Our findings suggest that the $\chi^2$-divergences may not be the most effective choice, emphasizing the importance of tailoring $f$-divergence selection in practice.

    % \yh{I suggest removing this paragraph as a single point of contribution, which may be incorporated in previous points}
\end{enumerate}

\subsection{Related Works}
\textbf{Graph Learning.} Learning DAGs from data is an NP-hard problem due to the combinatorial nature of the acyclicity constraint~\citep{chickering2004large,chickering1996learning}. Through iterative conditional independence testing with expanding conditioning sets and incremental edge orientation, the Peter--Clark (or PC) algorithm progressively constructs the DAG \cite{spirtes1991algorithm}. Traditional score-based methods optimize a score over the set of DAGs, framing it as a \emph{combinatorial optimization} problem. Previous works have proposed various scoring functions, including BDe(u)~\citep{heckerman1995learning}, BGe~\citep{kuipers2014addendum}, and BIC~\citep{maxwell1997efficient}, and various different search algorithms~\citep{tsamardinos2006max,cooper1992bayesian,chickering2002optimal}. Zheng et al.~\citep{zheng2018dags} addressed this challenge by reformulating it as a \emph{continuous optimization problem} using a linear structural equation model. Subsequent work extended this approach to sparse optimization~\citep{yu2019dag}.  

% \textbf{Graph Learning:} Learning directed acyclic graphs (DAGs) from data without any assumptions is generally an NP-hard problem, due to combinatorial acyclicity constraint that is difficult to enforce efficiently~\citep{chickering2004large,chickering1996learning}. Tranditional score-based learning optimize a discrete score over the set of DAGs, which is a combinatorial optimization problem, previous work selected different kinds of score functions including BDe(u)~\citep{heckerman1995learning}, BGe~\citep{kuipers2014addendum}, and BIC~\citep{maxwell1997efficient}. \cite{zheng2018dags} convert it to a \textbf{continuous optimization problem} via the linear structure equation model. Follow-ups extend it to nonlinear model~\citep{zheng2020learning},  sparse optimization~\citep{yu2019dag}. Empirically, graph autoencoders have been widely used to discover the graph structure.   

\textbf{Learning Graph Structure with Transformers.} Graph autoencoders have been widely utilized for empirically uncovering graph structures~\cite{kipf2016variational}, yet their theoretical foundations remain limited. Recent research has examined the expressive capabilities of transformers in modeling simple graph structures with Markovian data~\cite{rajaraman2024transformers, hu2024limitation, zhou2024transformers}. However, these studies do not analyze training dynamics or provide convergence guarantees. A few works have explored how transformers learn graph structures, identifying the mutual information as a crucial factor influencing attention patterns~\cite{nichani2024transformers, chenunveiling, edelman2024evolution}. Notably, Nichani et al.~\citep{nichani2024transformers} analyzed gradient descent dynamics in single-parent tree structures. While Chen et al.~\cite{chenunveiling} extended this analysis to multi-parent structures, their focus was limited to $n$-gram Markov models, and the graph structures considered were modeled by feed-forward networks rather than attention layers.

\textbf{Training Dynamics of Transformers.} A growing body of work, initiated by  Jelassi et al.~\citep{jelassi2022vision}, aims to explore the training dynamics of transformers from diverse perspectives. Jelassi et al.~\citep{jelassi2022vision} investigated the inductive biases of Vision Transformers (ViTs) with a focus on specialized positional attention. Building on this, Li et al.~\citep{li2023theoretical} analyzed the training process of shallow ViTs in supervised classification settings and extended their study to in-context learning~\cite{linonlinear}, under strict assumptions on network initialization. Huang et al.~\citep{DBLP:conf/icml/HuangCL24} established the in-context convergence of one-layer softmax transformers trained with gradient descent, providing insights into attention dynamics throughout the training process. Huang et al.~\citep{huang2024transformers} further examined attention dynamics in self-supervised learning settings. Additionally, Yang et al.~\citep{yang2024context} extended the study of in-context learning to multi-head transformers with non-linear task functions. More recently, \citep{wen2024sparse,kim2024cot,Liang_neurips2025_cot,Liang_icml2025_parity} investigated transformer dynamics in the chain-of-thought reasoning setting.

\section{Problem Formulation}\label{sec:form}
In this section, we formally introduce our data model and the transformer architecture used to learn the underlying structure.

\subsection{Data Model}\label{subs: data model}

\textbf{Graph Structure.} 
We consider a {\em directed acyclic graph} (DAG) with $T$ nodes, denoted as $\mathcal{G} = ([T], E)$, where $[T] = \{1, 2, \ldots, T\}$ represents the set of nodes, and $E$ is the set of directed edges that captures the dependence structure of among the nodes. For any two nodes $i$ and $j$, the presence of a directed edge $(j, i) \in E$ indicates that $j$ points to $i$. Here, we assume all directed edges are ordered such that $j$ points to $i$ when $j <i$. In this case,  $j$ is referred to as a \emph{parent} of $i$, and the set of all parents of $i$ is denoted by $p(i) \subset [T]\setminus \{i\}$. %gather all parents of $i$ into the set $p(i)$. 
The \emph{in-degree} of node $i$ is $|p(i)|$, representing the number of directed edges pointing towards $i$. For any node $i$ with in-degree $K>0$, we denote all its parents as $p(i)^1, \ldots, p(i)^K$ and we assume the order  $p(i)^1< \ldots< p(i)^K$. If the in-degree of node $i$ is 0, then $i$ is a \emph{root} node. We denote the set of root nodes as $\mathcal{R} \subset [T]$. Otherwise, $i$ is a non-root node with $|p(i)|$ parents. A \emph{tree} is  a DAG in which $|p(i)| \leq 1$ for all $i \in [T]$. 

Any DAG is in one-to-one correspondence with an adjacency matrix $A \in \{0,1\}^{T \times T}$, in which $A_{j,i}=1$ if $(j,i) \in E$ and $A_{j,i}=0$, otherwise. An illustrative example of such a DAG is shown on the top left of \Cref{fig:graph-structure}.

\textbf{Transition Kernels over DAGs.} For any two finite sets $\Sc,\Sc'$, the set of \emph{probability mass functions}  (PMFs) supported on $\Sc$ is denoted as $\mathcal{P}(\Sc)=\{q \in \Rb_{+}^{|\Sc|}:q(s) \geq 0, \forall s \in \Sc\}$, and the set of conditioned PMFs on $\Sc'$ given $\Sc$ is denoted as $\mathcal{P}(\Sc|\Sc')=\{\{q(\cdot|s')\in \mathcal{P}(\Sc)\}_{s'\in \Sc'}\}$. 

All nodes in a DAG are in one-to-one correspondence with random variables, each taking values in a finite set $\mathcal{S} = \{1,2,\ldots, S\}$. Each directed edge represents a conditional dependence: a child node depends on its parent nodes. We assume that each node's in-degree is either $0$ or $K$.\footnote{Here we assume the in-degree of all non-root nodes is exactly $K$ for ease of analysis. Our results can be easily extended to the general case in which the in-degrees of the root nodes are not necessarily the same. We discuss this extension in detail in Appendix \ref{app: ext to non uni K}.} Given $K$   parents $(S_{p(i)^1},\ldots, S_{p(i)^K})$ of any node $S_i$, a transition kernel $\pi\in \mathcal{P}(\Sc|\Sc^K)$
assigns a probability of transitioning to the state $S_{i}$ from  $(S_{p(i)^1},\ldots, S_{p(i)^K})$. Formally, for any $s' \in \Sc, (s_1,\ldots,s_K) \in \Sc^K$, $\pi(s'|s_1,\ldots,s_K):=\Pb(S_i=s'|S_{p(i)^1}=s_1,\ldots, S_{p(i)^K}=s_K)$. 

The transition kernel $\Pi$ is randomly drawn according to a distribution $P_\Pi$. For any realization $\pi$ of $\Pi\sim P_\Pi$, we assume that the transition kernel $\pi$ is {\em stationary} in the sense that it remains identical across parent-child tuples, and all elements of the transition kernel $\pi$ are {\em positive}, i.e., $\pi(s' | s_1,\ldots, s_K)>0$ for all $s', s_1, \ldots, s_K$.

The stationary distribution associated with any $\pi \sim P_\Pi$ is defined using the following procedure. First, we extend $\pi$ to another transition kernel $\widetilde{\pi} \in \mathcal{P}(\Sc^K|\Sc^K)$ by setting $\widetilde{\pi}\bigl(s'_1, s'_2, \ldots, s'_K   |  s_1, s_2, \ldots, s_K \bigr)= \pi\bigl(s'_K | s_1, s_2, \ldots, s_K\bigr)$ for any tuple of states $(s'_1, s'_2, \ldots, s'_K)$ and $(s_1, s_2, \ldots, s_K)\in \Sc^K$ if $s'_i = s_{i+1}$ for $i = 1,\ldots,K-1$, and $0$, otherwise. From the positivity of $\pi$, any tuple of states $(s'_1, s'_2, \ldots, s'_K)\in\Sc^K$ is accessible from any other tuple  $(s_1, s_2, \ldots, s_K)\in\Sc^K$ under the newly constructed transition kernel $\widetilde{\pi}$. This guarantees that  $\widetilde{\pi}$ admits a unique {\em stationary distribution} $M_{\widetilde{\pi}} \in \mathcal{P}(\Sc^K)$. We denote $M_{\pi}:=M_{\widetilde{\pi}} \in \mathcal{P}(\Sc^K)$ as the \emph{stationary distribution   of  $\pi$}. Furthermore, for each $i \in [K]$, the  {\em marginal distribution} of $M_{\pi}$ on the $i$-th coordinate is defined as $$\mu_{\pi,i}(s)=\sum_{j\neq i}\sum_{s_j\in\Sc}M_{\pi}(s_1,\ldots,s_{i-1},s,s_{i+1},\ldots,s_K), \qquad\forall\, s\in \Sc.$$ Since all marginal distributions $\mu_{\pi,i}$ of $M_{\pi}$ are identical across $i \in [K]$, a fact we formally show  in Appendix~\ref{sec: MC concentration}, we omit the subscript~$i$, and denote the \emph{marginal stationary distribution} of $\pi$ simply as  $\mu_{\pi}$. We assume that the random variables corresponding to the $K$ parent root nodes are sampled from the joint distribution $M_{\pi}$, and as a result, each random variable of the root node has a marginal distribution of $\mu_\pi$. Given the stationary distribution $\mu_\pi$, we define the {\em marginal transition kernel} conditioned on the $\ell$-th parent as \begin{align}\pi^\ell(s'|s_\ell):=\sum_{j\neq\ell}\sum_{s_j\in\Sc}\frac{\pi(s'|s_1,\ldots,s_K)M_\pi(s_1,\ldots,s_K)}{\mu_\pi(s_\ell)}.\label{eqn:pi_ell}
\end{align} Accordingly, $\pi^\ell(s_{i}|s_{p(i)^\ell})$ characterizes the probability of transition from the $\ell$-th parent $s_{p(i)^\ell}$ to the child $s_{i}$ (marginalized over all the other $K-1$ parents).\footnote{The proof of our claims regarding the properties of $\tilde{\pi}$, the existence of the stationary distribution $M_{\tilde{\pi}}$  and the identity of the marginals $\mu_{\pi,i}$ can be found in Appendix~\ref{sec: MC concentration}.}   

%We note that an order-$K$ Markov chain is a special case of our DAG model, in which each non-root node $i$ has parents with indices $i -1, i-2, \ldots i - K$. % has exactly $K$ {\em immediately} preceding nodes as parents. % In the general DAG setting, however, the parent set need not consist of consecutive predecessors.

\textbf{Generation of Random Sequences.}
Our goal is to estimate the structure of the DAG using random sequences generated by the aforementioned DAGs. Without loss of generality, we assume that the first $K$ nodes $1, \ldots, K$ are roots. Then a length-$T$ random sequence $S_{1:T}$ is generated by 
\begin{enumerate}[wide, labelindent=0pt]
\setlength{\itemsep}{0pt}
    \setlength{\parskip}{0pt} 
    \setlength{\topsep}{0pt}
    \item First, sample $\Pi=\pi$ from a probability distribution over transition kernels $ P_\Pi$.
    \item For $i = 1, \dots, T $, if $i_1<\ldots<i_K$ are $K$ parent root nodes,    sample $(S_{i_1},\ldots,S_{i_K}) \sim M_\pi$, the stationary distribution. Else, $i$ is not a parent node, sample  $S_i \sim \pi(\cdot | S_{p(i)^{1}},\ldots,S_{p(i)^{K}})$, the transition kernel given the parents.
    % \item \textit{Draw } $S_{T+1},\ldots,S_{T+K}, S_{T+k} \sim \text{Unif}([S])$ for any $k$. \textit{and } $S_{T+K+1} \sim \pi(\cdot | S_{T+1:T+K})$. 
    \item {Output:} Sequence $S_{1:T} = (S_1,\ldots,S_T)$.
\end{enumerate}

To learn the DAG structure, we make the following mild assumption that the marginal transition kernels are distinct. This ensures that the dependencies between a child and its various parents can be effectively captured and distinguished.
 .
%\yl{explain why this assumption is made}\yc{done}
%make some assumptions on the non-symmetry of marginal transition kernel.
\begin{assumption}[Distinct Marginal Transition Kernels]\label{Assp: MC-non sym} 
For any transition kernel $\pi$ in the support $ \mathrm{supp}(P_\Pi)$  of $P_\Pi$, its marginal transition kernels are distinct, i.e., \(\pi^i \neq \pi^j\) for all $i\ne j$.
\end{assumption}

\subsection{Transformer Architecture}\label{subs: transformer architecture}

In this work, we consider the causal self-attention head mechanism as follows. 
\begin{definition}[Causal Self-Attention Head] Consider a $K$-head causal self-attention layer. For the $\ell$-th head with $\ell=1,\ldots,K$, let $W^\ell_V \in \Rb^{d_{\mathrm{o}} \times d_{\mathrm{e}}}$, $W^\ell_K \in \Rb^{d_{\mathrm{e}} \times d_{\mathrm{e}}}$, and $W^\ell_Q \in \Rb^{d_{\mathrm{e}} \times d_{\mathrm{e}}}$ be the value, key and query matrices, respectively. We   collect all these parameters into $\theta^\ell=(W^\ell_V,W^\ell_K,W^\ell_Q)$. Given any input embedding $E \in \Rb^{d_{\mathrm{e}} \times T}$, the output of the $\ell$-th head is
    \begin{align*}
        & \attn(E;\theta^\ell)=W^\ell_V E\cdot \mathrm{softmax}\left((\mathrm{Mask}((W^\ell_K E)^\top W^\ell_Q E))\right) \in \Rb^{d_{\mathrm{o}} \times T}, 
    \end{align*}
    where the $\mathrm{softmax}(\cdot)$ function is applied in a column-wise manner to a matrix, and for a vector $v$, $\mathrm{softmax}(v)_i=e^{v_i}/\sum_{j}e^{v_j}$, and the $\mathrm{Mask}(\cdot)$ function masks the lower triangular part of the matrix, i.e., for a matrix $A$, $\mathrm{Mask}(A)_{j,i}=A_{j,i}$ if $i > j$; $\mathrm{Mask}(A)_{j,i}=-\infty$, otherwise.   
\end{definition}

\textbf{Embeddings.} Given a random sequence $S_{1:T}=(S_1,\ldots,S_T)$, we consider natural one-hot embeddings for the state space, i.e., for any state $S_i$, the embedding is $e_{S_i}\in \Rb^S$, which takes 1 at the $S_i$-th entry and 0 elsewhere. We further add a natural one-hot position embedding $e_i \in \Rb^T$ to each token and obtain the embedded sequence: 
\begin{align*}
    E=E(S_{1:T})=\begin{bmatrix}
        X_1&\ldots&X_{T}\\
        e_1&\ldots&e_{T}
    \end{bmatrix} \in \Rb^{(S+T) \times T},
\end{align*}
where $X_i$ represents the embedding of  $S_i$. % We use $X_i$  and  $S_i$  interchangeably   in the following sections.

\textbf{Reparameterization.} For each head $\ell \in [K]$, we consolidate the query and key matrices into a single matrix denoted as $W^\ell_{KQ} \in \Rb^{(S+T)\times (S+T)}$, as in previous works for analyzing transformers \citep{DBLP:conf/icml/HuangCL24,zhang2024trained,jelassi2022vision}. Since our focus is on positional information, we restrict $W^\ell_{KQ}$ to a block sparse version whose only nonzero block is $Q^\ell \in \Rb^{T \times T}$ which takes the form
\begin{equation}
W^\ell_{KQ}=\begin{bmatrix}
0_{S \times S} & 0_{S \times T} \\
0_{T \times S} & Q^\ell
\end{bmatrix}. \label{eqn:WKQ}    
\end{equation}
We further set the output dimension $d_{\mathrm{o}}=T$ and the value matrix $W^\ell_V=(0_{T\times S}| I_{T})\in \Rb^{T \times (T+S)}$ such that $W^\ell_V E= I_T$ for any embedding $E$. With the reparameterization above, the overall transformer model can be recast using the parameterization  $\theta=\{\theta^\ell\}_{\ell=1}^K=\{Q^\ell\}_{\ell=1}^K$:
\begin{align*}
\textstyle   F_{\mathrm{CSA}}(E,\theta)
=\big(\attn(E;\theta^1),\ldots,\attn(E;\theta^K)\big) \in \Rb^{T \times KT},
\end{align*}
where $\attn(E;\theta^\ell)=\mathrm{softmax}\left(\mathrm{Mask}(Q^\ell)\right) \in \Rb^{T \times T}$. We abbreviate $\attn(E;\theta^\ell)$ as $\attn^{\ell}$ for simplicity.

%We aim to find the underlying causal structure of the input sequence, i.e., the correlation dependence of different positions, which can be captured by the attention score of the self-attention layer. 
%We aim to use the attention scores of the self-attention layer to estimate the adjacency matrix of the DAG as shown in \Cref{fig:graph-structure}. 

We aim to uncover the underlying DAG structure that generated the data sequence $S_{1:T}$. The DAG can be represented by the attention scores in the self-attention layer, as detailed in \Cref{sec:obj_alg}. In \Cref{fig:graph-structure}, we illustrate the use of these attention scores to estimate the adjacency matrix of the DAG.

\section{Novel Attention-Enabled Information-Theoretic Objective}

In this section, we design a parameterized objective function that guides the training of attention models to accurately identify the underlying graph structure.
%we design a parameterized objective function that can be optimized to identify the graph structure. 

\subsection{Preliminaries}\label{sec:preliminary}
We first define the (conditional) $f$-mutual information.
\begin{definition}
Let $f: \mathbb{R}^+ \to \mathbb{R}$ be a convex function with $f(1) = 0$. % We have the following definitions:
\begin{enumerate}[topsep=0pt, left=1pt]
\itemsep 0pt
    \item  The {\em   \(f\)-divergence} between two probability mass functions $P$ and $Q$ on a finite sample space  $\mathcal{S}$ is given by
\begin{align*}
     D_f(P \| Q) = \sum_{x \in \mathcal{S}}Q(x) f\Big(\frac{P(x)}{Q(x)}\Big) .
\end{align*}
\itemsep 0pt
    \item  The {\em  \(f\)-mutual information} between   two random variables $X$ and $Y$ with joint distribution $P_{X,Y}$ is
\begin{align*}
    &I_f(X; Y) = D_f(P_{X,Y} \| P_X \otimes P_Y).
\end{align*}
\end{enumerate}
For random sequences generated as described in \Cref{subs: data model},   we define the {\em $f$-mutual information} between two random variables \(S_i\) and \(S_j\) in $S_{1:T}$  conditioned on $\Pi$ as\footnote{Note that Eq.~\eqref{eqn:f_MI} (and similar $f$-mutual information-like quantities) should be written in their conditional forms as $I_f(S_i; S_j |\Pi)$ but for brevity, we suppress the dependence of various $f$-mutual information quantities on $\Pi$ in the sequel.} 
\begin{equation}
I_f(S_i; S_j ):= \Eb_{\Pi \sim P_\Pi} \left[ \sum_{s,s'\in\mathcal{S}} P_{S_i| \Pi}(s)P_{S_j| \Pi}(s) f\bigg(\frac{P_{S_i, S_j| \Pi}(s,s')}{P_{S_i| \Pi}(s)P_{S_j| \Pi}(s')} \bigg) \right],     \label{eqn:f_MI}
\end{equation}
where $P_{S_i, S_j|\Pi}$ is the joint probability distribution of nodes $S_i$ and $S_j$ given the transition kernel $\Pi$. %For simplicity, we omit conditioning on $\Pi$ in $P_{S_i|\Pi}$ and the subscript $\Pi \sim P_\Pi$ in the expectation whenever there is no ambiguity. 
\end{definition}

Conventionally, maximizing the sum of mutual information quantities has served as a fundamental principle for identifying parent nodes in tree structures. The celebrated Chow--Liu algorithm \cite{chow1968approximating} exemplifies this approach by leveraging the maximum likelihood principle to learn an optimal tree. It is demonstrated that this process is equivalent to constructing a maximum-weight spanning tree, where the edge weights correspond to pairwise mutual information terms. Furthermore, the data processing inequality \cite{cover1999elements} ensures that mutual information is non-increasing along directed paths. Consequently, for each node $i$, its direct parent is the node with which it shares the highest mutual information. This property similarly extends to the \(f\)-mutual information \cite{csiszar1972class}, indicating that maximizing \(f\)-mutual information provides a robust and principled criterion for identifying the set of parents of each node $i$.

%Conventionally, maximizing the (sum of) mutual information quantities has been used as a general principle to identify parent nodes in trees. Indeed, the celebrated Chow--Liu algorithm \citep{chow1968approximating} utilizes the maximum likelihood principle to learn  an optimal tree structure and it is shown that this is tantamount to constructing a maximum-weight spanning tree where the edge weights are the  pairwise \textit{mutual information} terms. In addition, by the data processing inequality~\citep{cover1999elements}, mutual information decreases along directed paths. Therefore, for each node $i$, its direct parent  has the largest mutual information with $S_i$. The same property extends to the $f$-mutual information~\citep{csiszar1972class}, suggesting that maximizing $f$-mutual information provides a principled criterion for identifying parent nodes.  

% the node that achieves the highest mutual information with $S_i$ must belong to the parent set $p(i)$. The data processing property also holds for the $f$-mutual information~\citep{csiszar1972class}, which suggests that using maximal $f$-mutual information can also be a principle to identify the parents for a child. 

For a tree structure where each node has at most a single parent, the above principle naturally suggests the following objective that leverages a single-head attention model for learning the parent of each node:
\begin{align}
    L_f(\theta)=\sum_{i,j \in [T]} I_f(S_i; S_j) \attn_{j,i}(\theta), \label{eq:losS_singleparent}
\end{align}
where we have omitted the head index $\ell$ on $\attn$ for the single-head attention model. 
% \yc{Their obj is not exactly the mutual information but is similar type}
Clearly, for each node $i$, maximizing the above objective function will drive the attention score for target $i$ to concentrate on $j$ that maximizes $I_f(S_i; S_j)$, and hence $j$ will be learned as the parent of $i$. Indeed, an objective function akin to that in Eq.~\eqref{eq:losS_singleparent} has been implicitly used in \citep{nichani2024transformers} for learning a {\em tree} structure (where each node has only one parent).\footnote{Nichani et al.~\citep{nichani2024transformers} studies a problem that combines tree-structure learning and regression prediction. By singling out their tree-learning task, the objective function can be seen to take a similar form to \eqrefn{eq:losS_singleparent}.}
% A more detailed discussion can be found in Appendix \ref{appd: Discussion with Previous Work}.

\subsection{Novel \Name Mutual Information for Multi-Parent Learning}\label{sec: Definition of KG-MI}

\begin{figure*}
    \centering
    \includegraphics[width=0.95\textwidth]{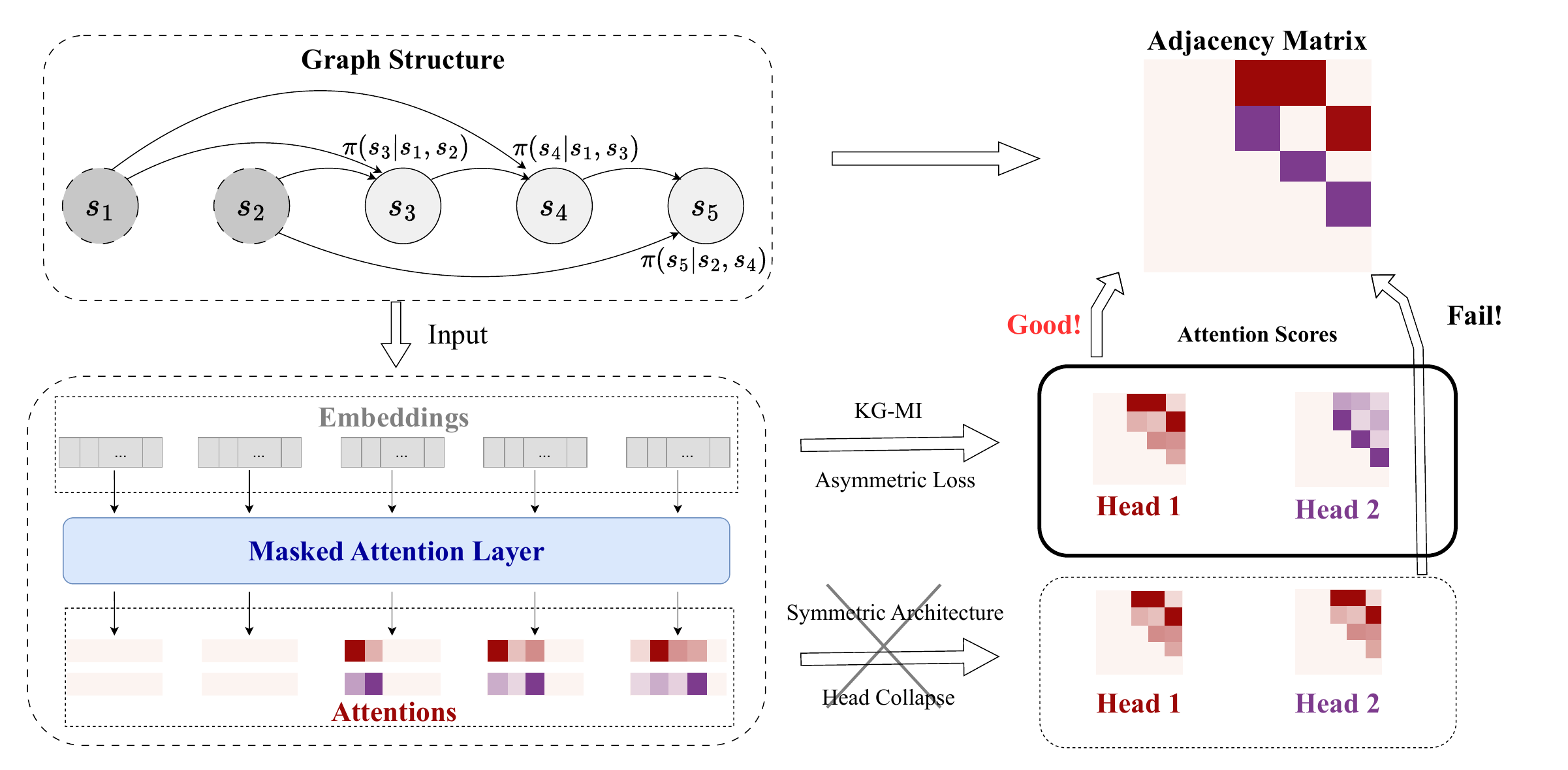}
    \caption{The process of using a transformer to learn the graph structure proceeds as follows. Random sequences generated by a DAG (top left) are fed into the transformer (bottom left), which produces multi-head attention scores (bottom right). These scores are then used to estimate the adjacency matrix of the DAG (top right). The DAG illustrated in the top left consists of five nodes with the edge set $E = \{(1,3), (1,4), (2,3), (2,5), (3,4), (4,5)\}$. It features two root nodes, labeled 1 and 2, and three non-root nodes labeled 3, 4, and 5, each with an in-degree of 2. The adjacency matrix at the top right highlights parent-child relationships, with positions marked in dark red and purple. The attention scores from the transformer's multi-head layers are visualized in the bottom right, where darker colors indicate higher attention values. Note that using a standard symmetric multi-head architecture results in head collapse, leading to erroneous learning. Conversely, our KG-MI objective effectively mitigates this issue, enabling successful graph structure learning.}
\label{fig:graph-structure}
\end{figure*}

% A straightforward extension of the above idea to multi-parent learning is as follows. For any node $i$, given the number of parents $K$, select the set of nodes exhibiting the largest mutual information with $i$. To accommodate multi-parent learning, multi-head transformer can be used, where each head identifies one parent for node $i$. However, such an approach is often not efficient\textcolor{blue}{Explain a bit more, why inefficient}. 

%{\bf Challenge of head collapse:} 

A natural extension of the idea in \Cref{sec:preliminary} to learn multi-parent DAGs  encounters a significant issue known as {\em  head collapse}, as discussed below. Consider a node $i$ with $K>1$ parents. If we employ a $K$-head attention model combined with the standard $f$-mutual information $ I_f(S_i; S_j)$  in a loss function similar to that in Eq.~\eqref{eq:losS_singleparent}, we have
\begin{align}
    L_f(\theta)=\frac{1}{KT}\sum_{\ell=1}^K\sum_{i,j \in [T]} I_f(S_i; S_j) \attn^\ell_{j,i}(\theta). \label{Eq: stra obj}
\end{align} 
In this setup, all attention heads indexed by $\ell$ tend to home in on and thus select the same parent node because they rely on the {\em identical} $f$-mutual information quantity $I_f(S_i; S_j)$ (which does not vary with the head $\ell$). Consequently, the heads fail to distinguish multiple distinct parents, as confirmed by our experiments. As illustrated in \Cref{fig:head collapse}, both KL and $\chi^2$-mutual information lead to head collapse: the multiple heads learn identical attention patterns, focusing exclusively on the same parent and thus not capturing other true parents.

\begin{figure}[ht]
    \centering
    \begin{tabular}{c@{\hspace{5pt}}c@{\hspace{5pt}}c@{\hspace{5pt}}c@{\hspace{5pt}}}
    DAG $\mathcal{G}$ & \quad ~~ Heatmap of KL (Head 1) & \quad ~~ Heatmap of KL (Head 2) &~~ \\ \includegraphics[width=0.26\textwidth]{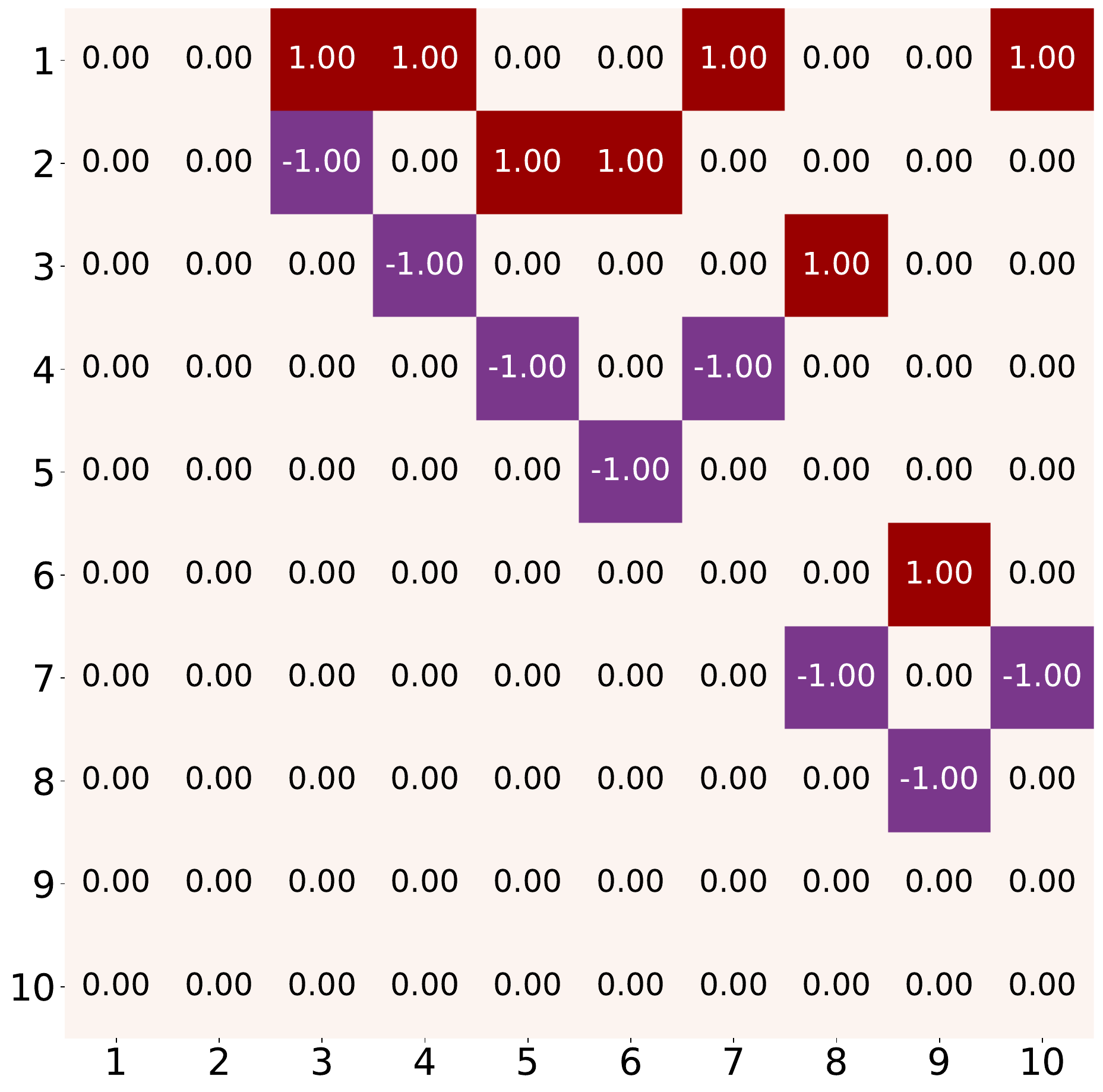} &\quad ~
    \includegraphics[width=0.26\textwidth]{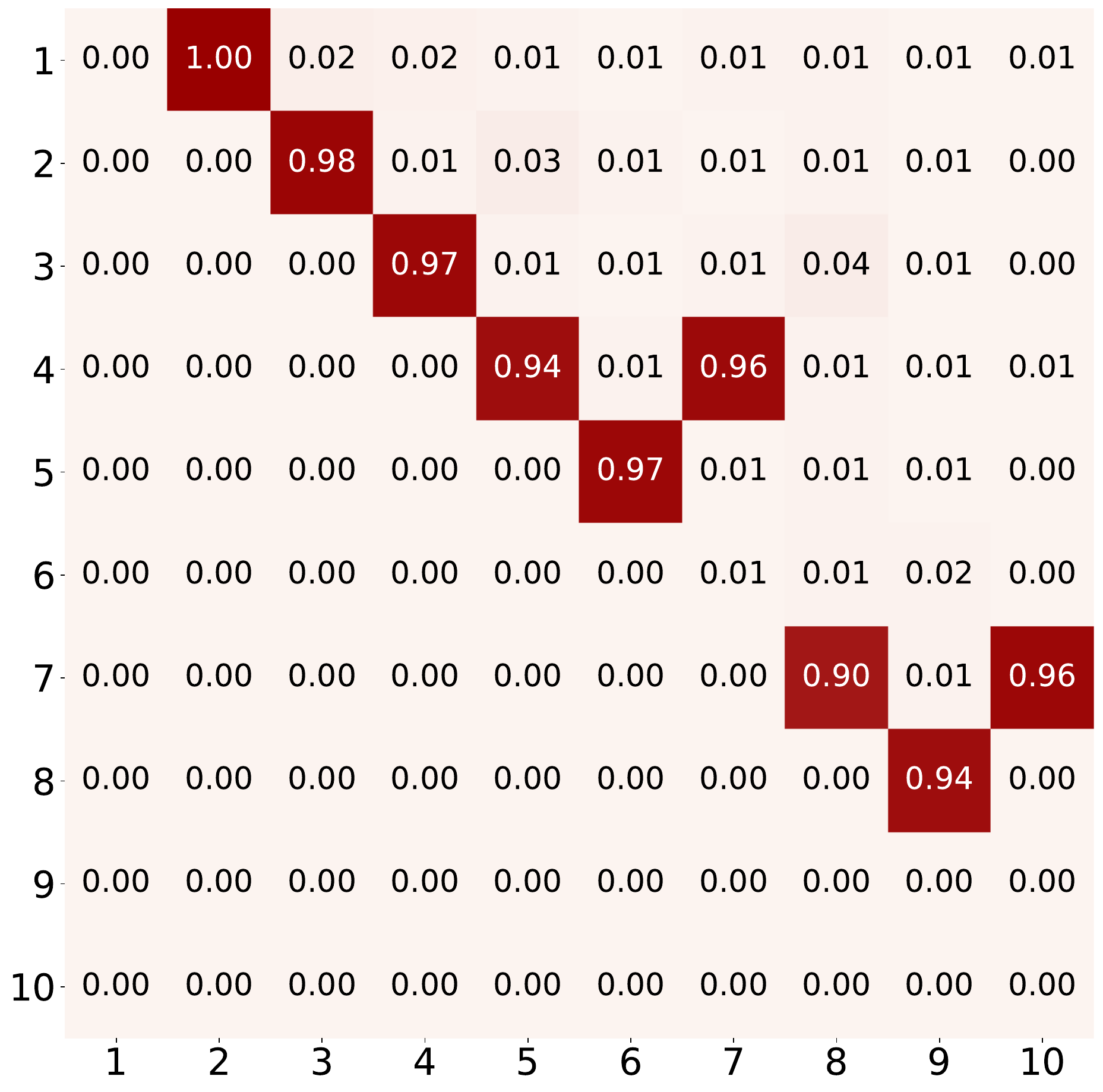} &\quad~
    \includegraphics[width=0.26\textwidth]{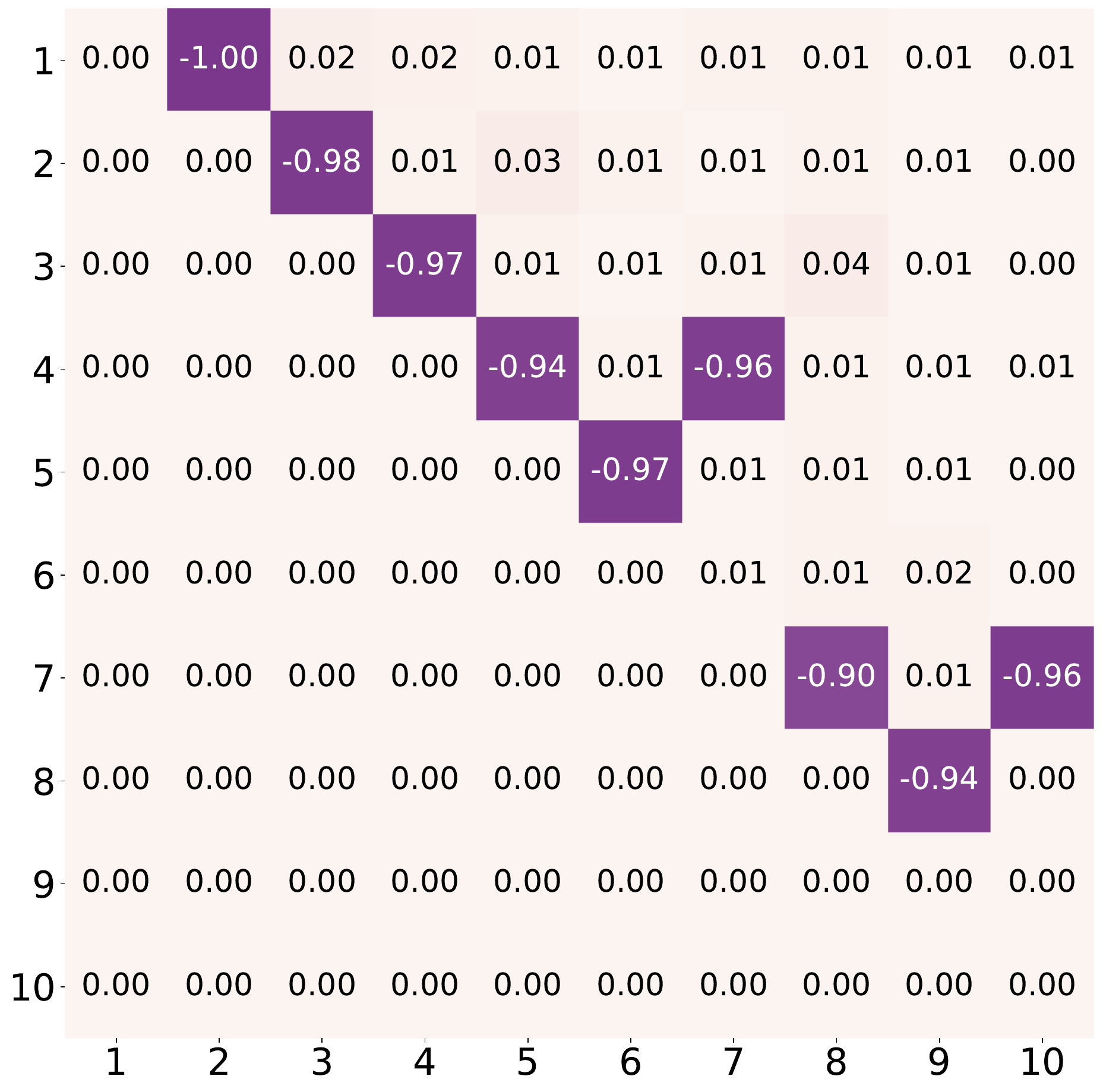} &~~
    \put(0,12){\includegraphics[width=0.05\textwidth]{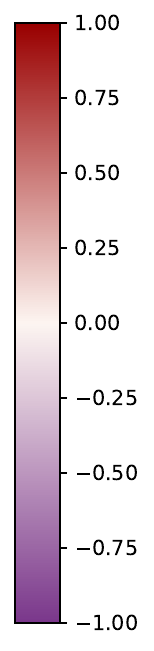}}
    \\
    DAG $\mathcal{G}$ & \quad ~~ Heatmap of $\chi^2$ (Head 1) & \quad ~~ Heatmap of $\chi^2$ (Head 2) &~~ \\ \includegraphics[width=0.26\textwidth]{IT/fig/Heatmap/heatmap_of_graph_w_num.pdf} &\quad ~
    \includegraphics[width=0.26\textwidth]{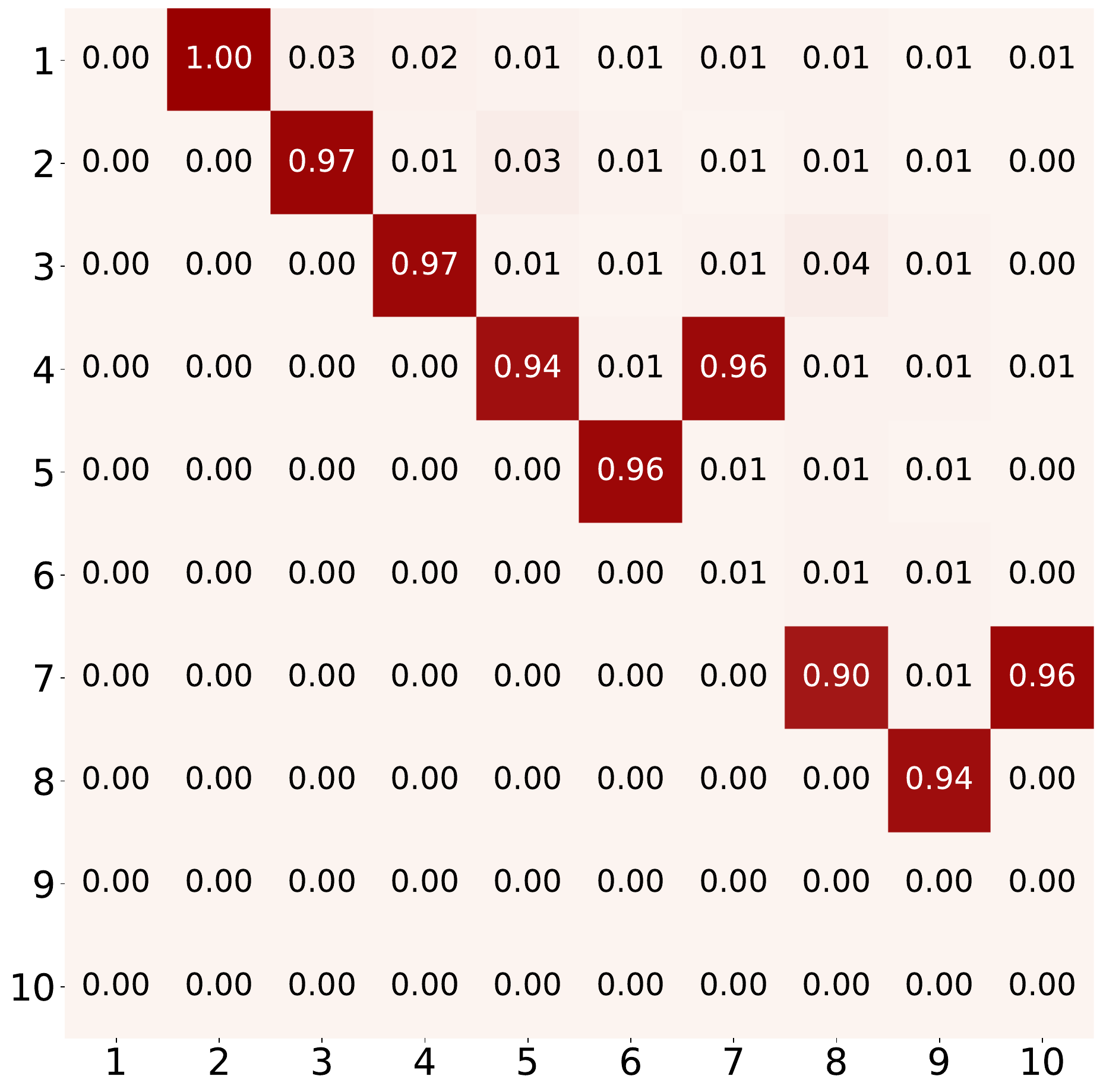} &\quad~
    \includegraphics[width=0.26\textwidth]{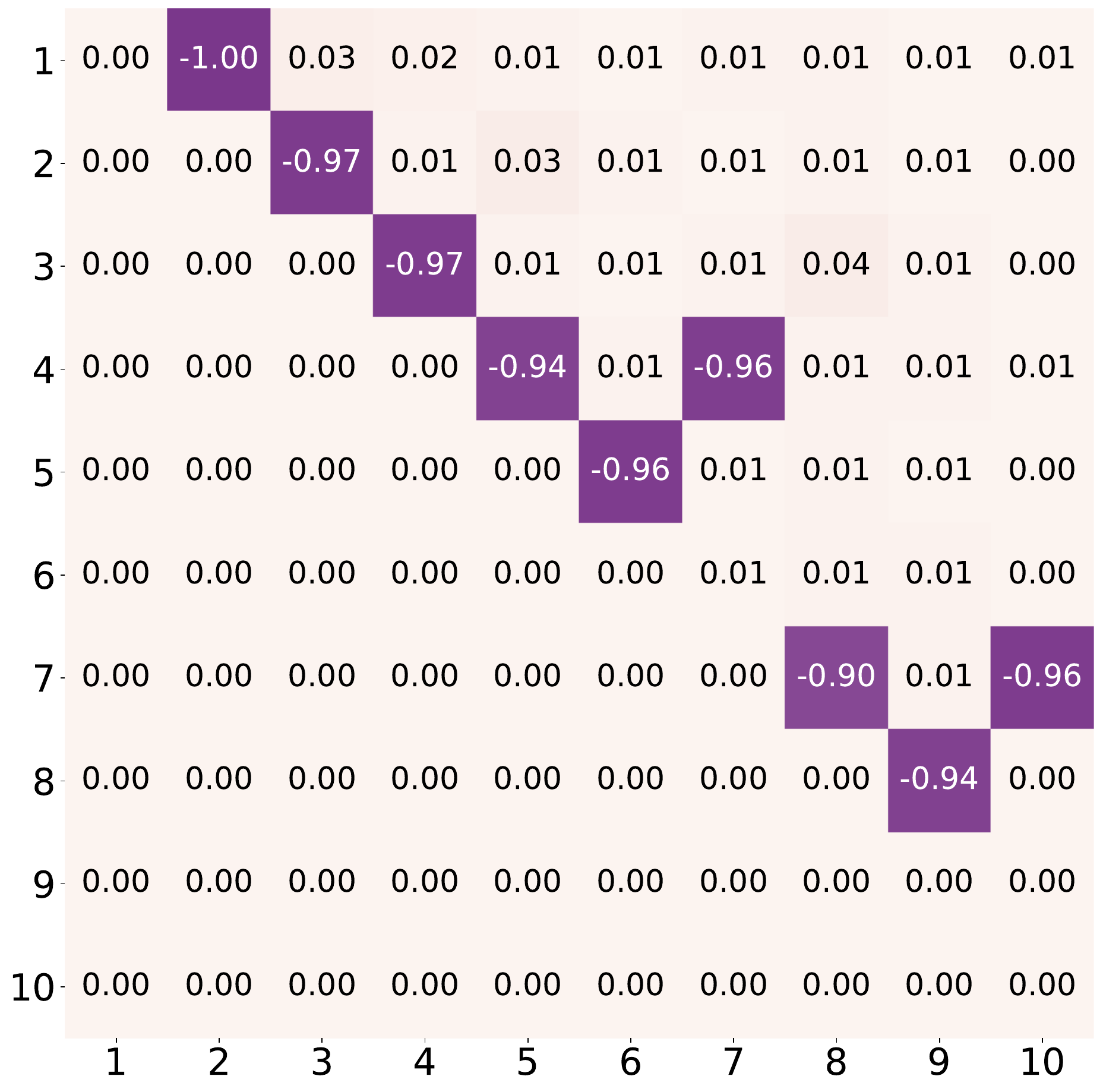} &~~
    \put(0,12){\includegraphics[width=0.05\textwidth]{IT/fig/Heatmap/colorbar_w_num.pdf}}
    \\
    \end{tabular}
    \caption{Heatmaps of the true adjacency matrices of  the DAG  $\mathcal{G}$ (first column) are shown alongside the attention patterns learned by two heads using different mutual information measures. The top row displays results obtained with KL mutual information (second and third columns), while the bottom row presents results with the $\chi^2$-mutual information (second and third columns). In this experiment, the data is sampled from a graph with 10 nodes, with the first two nodes being roots.
To visualize the learned multi-head attention scores, different colors (red and blue) are used to represent the attention patterns of each head. When trained with the naive objective from Eq.~\eqref{Eq: stra obj}, both KL and $\chi^2$-mutual information lead to head collapse: the two heads produce identical heatmaps and only recover a single parent node in \(\mathcal{G}\).}
\label{fig:head collapse}
\end{figure}

% For any node $i$, given the number of parents $K$, one straightforward way of extension is to select the node set which has the largest mutual information with $i$. However, this is not efficient as a multi-head transformer xx xx. Another efficient way is to directly select the top $K$ random variables with the largest mutual information with $i$ by the $K$ heads, which is not effective as the $K$ heads will select the same node with the same defined mutual information.
To overcome the above difficulty, we propose a novel modified mutual information as follows, which encourages different attention heads to learn distinct parents.
% If $P_{S_i,S_j|\pi}(s',s)\!>\!0$ for all $(s,s')$ and $\pi \in \mathrm{supp}(P_\Pi)$, then
\begin{definition}[\Name Mutual Information]\label{Def: Modified MI} Let $f: \mathbb{R}_+ \to \mathbb{R}$ be a convex function with $f(1) = 0$. For any $\ell  \in  [K]$ and $i, j \in  [T]$, define the {\em kernel-guided mutual information} (KG-MI) as 
    \begin{align*}
        \tilde{I}^\ell_f(S_i; S_j) = \Eb_{ \Pi\sim P_\Pi}\left[\sum_{s,s'} \frac{P_{S_j|\Pi}(s)\Pi^\ell(s'|s) P_{S_i|\Pi}(s')P_{S_j|\Pi}(s)}{P_{S_i,S_j|\Pi}(s',s)}f\Big(\frac{P_{S_i,S_j|\Pi}(s',s)}{P_{S_i|\Pi}(s')P_{S_j|\Pi}(s)}\Big)\right].
        \end{align*}
\end{definition}
% We note that for any $f$,  then we can rewrite the mutual information between token $X_i$ and $X_j$ as 
% \begin{align}
%     I_f(X_i; X_j) &= \Eb_{\pi\sim P_\pi,(x_i,x_j)\sim P_{X_i,X_j}}\left[\frac{P_{X_i}(x_i) P_{X_j}(x_j)}{P_{X_i,X_j}(x_i,x_j)}f(\frac{P_{X_i,X_j}(x_i,x_j)}{P_{X_i(x_i)}P_{X_j}(x_j)})\right]\nonumber\\
%     &= \Eb_{ \pi\sim P_\pi}\Eb_{x_j\sim P_{X_j}}\left[\Eb_{x_j\sim P_{X_i|X_j}}\left[\frac{P_{X_i}(x_i) P_{X_j}(x_j)}{P_{X_i,X_j}(x_i,x_j)}f(\frac{P_{X_i,X_j}(x_i,x_j)}{P_{X_i}(x_i)P_{X_j}(x_j)})\right]\right].\label{Eq: super popular f-mutual info}
% \end{align}
% \textcolor{red}{VT: Please explain why the quantity above is termed ``kernel-guided''. }
% \textcolor{blue}{done}

%The {\em kernel-guided mutual information} (KG-MI) 
The KG-MI derives its name from the incorporation of the marginal transition kernel  $\Pi^\ell$ within the standard mutual information framework to guide DAG learning. By incorporating this kernel and reweighting the standard mutual information with the factor ${P_{S_j|\Pi}(s)\Pi^\ell(s'|s)}/{P_{S_i,S_j|\Pi}(s',s)}$, we observe that the KG-MI has several desirable properties, as outlined below.

% \Name mutual information (KG-MI) gets its name from the introduction of marginal transition kernel to guide learning in the standard mutual information. By reweighting the standard mutual information with the factor $\frac{P_{S_j}(s)\pi^\ell(s'|s)}{P_{S_i,S_j}(s',s)}$, PG-MI is equipped several desirable properties, outlined as follows.
\begin{itemize}
    \item \textbf{Variation with $\ell$:} The definition of the KG-MI  $\tilde{I}^\ell_f(S_i; S_j)$ explicitly depends on $\ell$. Consequently, when designing objective functions for multiple heads $\ell$, each head leverages a different $\tilde{I}^\ell_f$, enabling the discovery of all distinct parents by the different heads.
\item \textbf{Specialization to the standard MI:} If $j=p(i)^\ell$, then $\tilde{I}^\ell_f(S_i; S_j)$ coincides with the $f$-mutual information between $S_i$ and $S_{p(i)^\ell}$. Furthermore, $\tilde{I}^\ell_f(S_i; S_j)=0$ if $S_j$ is independent of $S_i$. 

\end{itemize}

% Different choices of the convex function $f$ may lead to varying theoretical and empirical performance, highlighting the importance of a judicious choice of $f$ for a given application, which we will discuss in detail in \Cref{sec: thm,sec: experiment}.

%\subsection{Training setting}\label{subs: training setting}

\subsection{Objective Function and Training Algorithm}\label{sec:obj_alg}
Inspired by \eqrefn{eq:losS_singleparent}, which is the objective function for learning a tree structure, 
we define the following objective function based on the $f$-KG-MIs $\{\tilde{I}^\ell_f(S_i; S_j)\}$ in \Cref{Def: Modified MI}:
\begin{align}
   L_f(\theta)= \frac{1}{KT}\sum_{\ell=1}^K\sum_{i,j \in [T]} \tilde{I}^\ell_f(S_i; S_j) \attn_{j,i}^{\ell}(\theta), \label{Eq: obj f}
\end{align}
%where $\tilde{I}^\ell_f(S_i; S_j)$ is defined as in \Cref{Def: Modified MI},and 
where $\attn^{\ell}$ is the attention score of head $\ell$ and $\attn^{\ell}_{j,i}$ is its   $(j,i)$-th entry.

We explain how the objective function in \eqrefn{Eq: obj f} guides the learning of the structure of DAGs. Specifically, \eqrefn{Eq: obj f} uses {\em multiple} attention heads to capture the {\em  multiple} parent nodes that each node in a DAG may have.
%Similarly to the single-parent case, 
For each attention head $\ell$ and each node $i$, maximizing the above objective function drives the attention score for target $i$ to  {\em concentrate} on the $j$ that maximizes $\tilde{I}^\ell_f(S_i; S_j)$, and hence, $j$ will be learned as the $\ell$-th parent of $i$. Since $\tilde{I}^\ell_f(S_i; S_j)$ varies across different $\ell$'s (i.e., the KG-MIs $\tilde{I}^\ell_f(S_i; S_j)$ are distinct across $\ell$'s), each attention head is maximized by a different $j$, enabling the model to learn {\em distinct} parent nodes.

\textbf{No loss in optimality.} For any non-root node $i \notin \mathcal{R}$, and head $\ell \in [K]$, we denote $j(i)^{\ell,*}:=\arg\max_{j\in[K]}\tilde{I}^\ell_f(S_i; S_j)$ as the parent of $i$ with the largest $f$-KG-MI. Then the maximum of the objective function is given by $$L^*=\frac{1}{KT}\sum_{\ell=1}^K\sum_{i \notin \mathcal{R}} \tilde{I}^\ell_f(S_i; S_{j(i)^{\ell,*}}),$$ which is equal to the optimal objective that one would obtain without any reparameterization in the transformer architecture (see Section~\ref{subs: transformer architecture}). In other words, the specific form of the weight matrices $\{W^\ell_{V},\, W^\ell_{KQ}\}$ (adopted in, say, Eq.~\eqref{eqn:WKQ}) does not prevent the transformer from being trained to achieve the optimal objective value.

In practice, the true $f$-KG-MIs are  unknown. Thus,  we need to estimate them for training. Under additional assumptions and if the number of sequences and the sequence lengths are sufficiently large,   it is possible to obtain an accurate estimate of the $f$-KG-MIs. Thus, we assume access to a reliable estimate with a small error. 
\begin{assumption}\label{assp: est of F}
    For each $i,j \in  [T]$ and any $\epsilon>0$, there exists an empirical version of $\tilde{I}_f^\ell(S_i; S_j)$, denoted as $\hat{I}_f^\ell(S_i; S_j)$ and constructed using sufficiently many length-$T$  sequences $S_{1:T}$, satisfying that for each  $ \ell \in [  K]$,  $$\big|\tilde{I}^\ell_f(S_i; S_j)-\hat{I}^\ell_f(S_i; S_j)\big| \leq \epsilon.$$  
\end{assumption}
To demonstrate that the assumption can be satisfied in practice, we consider the case where $f(x) = x^2 - x$, which corresponds to  Pearson's $\chi^2$-KG-MI. Under mild assumptions on the graph structure and transition kernel as stated precisely in \Cref{lem: est chi2 MI} and is also assumed in previous work~\citep{nichani2024transformers}, a reliable practical estimate satisfying \Cref{assp: est of F} can be obtained when both the length and the number of random sequences $S_{1:T}$ are sufficiently large. The detailed estimation procedure for $\tilde{I}^\ell_{\chi^2}(S_i; S_j)$ and  the proof that the estimated $\hat{I}^\ell_{\chi^2}(S_i; S_j)$ satisfies \Cref{assp: est of F} are provided in Appendix~\ref{sec:example_pearsons}.

\begin{algorithm}[H]
   \caption{Learning a multi-parent DAG via Gradient Ascent}
   \label{alg:meta}
\begin{algorithmic}[1]
   \STATE {\bfseries Input:} Learning rate $\eta >0 $; number of iterations $\tau$; in-degree $K$; convex function $f$ in the $f$-divergence. % to be used in the $f$-divergence.
   \STATE Estimates of the $f$-KG-MIs $\{\hat{I}^\ell_f(S_i; S_j)\}$   satisfying \Cref{assp: est of F}.
  \STATE Objective  function as in \eqrefn{Eq: obj f}
  \STATE Initialize $\theta(0)=\{Q^\ell=\mathbf{0}_{T\times T}\}_{\ell=1}^K$. \label{line: init}. 
   % \begin{align*}
   %     \textstyle L_f(\theta)=\frac{1}{KT}\sum_{\ell=1}^K\sum_{i,j=1}^T \tilde{I}^\ell_f(S_i; S_j) \attn_{j,i}^{\ell}(\theta), \theta=Q^\ell. 
   % \end{align*}
   % \textcolor{red}{VT: remove the above equation}
   \FOR{$t=1$ {\bfseries to} $\tau$}
   \STATE $\theta(t) = \theta(t-1) + \eta \widehat{\nabla}_\theta L_f(\theta(t-1))$. 
   \ENDFOR
   \STATE {\bfseries Output:} $\hat{\theta}=\theta(\tau)$.
\end{algorithmic}
\end{algorithm}
% \end{minipage}
% \end{wrapfigure}
\textbf{Training Algorithm:} We first obtain   estimates $\{\hat{I}^\ell_f(S_i; S_j)\}$ satisfying \Cref{assp: est of F}. Next, we maximize the objective in \eqrefn{Eq: obj f} via gradient ascent (GA) with a constant learning rate $\eta>0
$ and the estimate of the gradient $\widehat{\nabla}_\theta L_f(\theta(t-1))$ which     is obtained by replacing the population $f$-KG-MI in the population gradient $\nabla_\theta L_f(\theta(t-1))$ with its empirical estimate $\hat{I}^\ell_f(S_i; S_j)$. At $t=0$, parameters $\theta(0)$ are initialized as zero matrices. The   pseudocode of the structure learning algorithm is provided in \Cref{alg:meta}.
 
\section{Main Theorems}\label{sec: thm} 
In this section, we characterize the training dynamics of the multi-head attention model by GA. The  proofs of the theoretical results  stated in this section are provided in Appendix \ref{App: Proof}.

\subsection{Convergence of Objective Function}

Since the algorithm's performance depends on the $f$-KG-MIs, we introduce some relevant statistics of $f$-KG-MIs as follows. For any non-root node $i$, and head $\ell \in [K]$, we denote $$\Delta^{\ell}(i) :=\tilde{I}_f^{\ell}(S_i;S_{j(i)^{\ell,*}})-\max_{j \neq j(i)^{\ell,*},j<i}\tilde{I}^\ell_f(S_i; S_j),$$ and then define the {\it information gap} of the DAG, together with its parameters, as $$\Delta:=\min_{\ell \in [K],i \in [T]}\Delta^\ell(i).$$ The information gap quantifies the variation of the $f$-KG-MIs across different node pairs and impacts the convergence rate of the GA algorithm, as we characterize in the following theorem. 

\begin{theorem}[Convergence of Objective Function] \label{Thm: loss convergence}  
Suppose Assumptions \ref{Assp: MC-non sym} and \ref{assp: est of F} hold, and every estimate for $f$-KG-MI chosen in \Cref{alg:meta} is bounded above by $I_{\max}$, i.e.,   $\max_{i,j \in [T]}\max_{\ell\in [K]} \hat{I}^\ell_f(S_i; S_j)\leq I_{\max}$. Let $\epsilon<\frac{1}{4}\Delta$, and for any attention error $\varepsilon_\mathrm{attn}>0$, define $$\tau_*=\frac{4KTI_{\max}\log{\frac{I_{\max}}{\varepsilon_\mathrm{attn}}}}{\varepsilon_\mathrm{attn}\eta\Delta}+\frac{4KT^2\log T}{\eta\Delta}.$$  Then, for any iteration $t\geq \tau_*$, the output of \Cref{alg:meta}, $\theta(t)$, satisfies $L^*-L(\theta(t)) \leq \max\{I_{\max}\varepsilon_{\mathrm{attn}},\epsilon\}.$
\end{theorem}

\textbf{Objective Convergence:} \Cref{Thm: loss convergence} demonstrates that when $\Delta/4$ exceeds the permissible estimation error $\epsilon$, training a one-layer transformer using the objective function defined in Eq.~\eqref{Eq: obj f} converges to the global maximum of the objective in the reparameterized space via GA. This convergence occurs with time complexity polynomial  in $T$, $K$, $\frac{1}{\Delta}$, $\frac{1}{\eta}$ and $\frac{1}{\varepsilon_{\mathrm{attn}}}$. We experimentally validate our theoretical convergence rate using synthetic data, as presented in \Cref{fig:convergence rate}. The results indicate that the actual convergence rate scales linearly with both $T^2 \log T$ and $\frac{1}{\Delta}$, aligning well with the upper bounds stated in \Cref{Thm: loss convergence}. This consistency suggests that our theoretical result presented in Theorem~\ref{Thm: loss convergence} is tight w.r.t.\  $T$ and $\Delta$. % in practical scenarios. 

% The convergence rate is inversely proportional to the learning rate $\eta$, and if the learning rate is large enough, the training can end in one step. Such a phenomena is because of the special landscape of 

%\Cref{Thm: loss convergence} shows that, if the information gap $\Delta$
%dominates the estimation error $\epsilon$, training a one-layer transformer with an objective function defined in \eqrefn{Eq: obj f} converges to the global maximum of the objective function in the reparameterization space via GA, with polynomial time
%efficiency with respect to $T,K,\frac{1}{\Delta}$ and $\frac{1}{\varepsilon_{\mathrm{attn}}}$. We validate our theoretical result of convergence rate with experiment on synthetic data in \Cref{fig:convergence rate}. It can be observed that the actual convergence rate scales linearly in both $T^2\log T$ and $\frac{1}{\Delta}$, which matches the upper bounds given in \Cref{Thm: loss convergence}, verifying that our upper bounds are rather tight in some practical scenarios.

\textbf{Impact of different choices of $f$ in \abname:} As \Cref{Thm: loss convergence} suggests, the convergence rate is inversely proportional to the information gap $\Delta$. Consequently, the choice of the function $f$ in the $f$-KG-MI can be considered as a design parameter to enhance the rate of convergence. Specifically, selecting an $f$ that produces a larger gap $\Delta$ in $\tilde{I}^\ell_f(S_i; S_j)$ accelerates the training process. We further validate this observation experimentally, as shown in \Cref{fig: compreh f mi} and discussed in \Cref{sec: experiment}.

%As \Cref{Thm: loss convergence} implies, the convergence rate is inversely proportional to the information gap $\Delta$. Hence, 
%the choice of $f$ in KG-MI can be treated as a design parameter to optimize the convergence. Namely, by selecting an $f$ such that $\tilde{I}^\ell_f(S_i; S_j)$ has a larger gap, we can accelerate the convergence during training. We demonstrate this insight through an experiment, with results shown in \Cref{fig: relation to gap}. In the experiment, we evaluate four $f$-KG-MIs, including KL, Hellinger, Pearson $\chi^2$ and Neyman $\chi^2$ \abname, on the same data model. Clearly, these choices of $f$ result in different convergence rates. Among them, Neyman KG-MI, which has the largest information gap 
%$\Delta$, achieves the fastest convergence. This empirically validates that the choice of $f$ directly influences the convergence rate via its impact on $\Delta$.

\subsection{Concentration of Attention  Scores}
The next theorem characterizes the attention patterns for both root and non-root nodes at convergence. 
\begin{theorem}[Attention Concentration]\label{Thm: Attention Concentration}  
Under the same conditions of \Cref{Thm: loss convergence}, for any attention error $\varepsilon_\mathrm{attn}>0$ and any $i \in [T]$, the output of \Cref{alg:meta} satisfies 
    \begin{enumerate}[wide, labelwidth=!, labelindent=0pt]
        \itemsep 0pt
        \item Let $i$ be a non-root node  and $\epsilon<\frac{1}{4}\Delta^\ell(i)$ for all~$\ell$ and~$i$. Let $$\tau^{\ell,i}_*=\frac{4KT\log{\frac{1}{\varepsilon_\mathrm{attn}}}}{\varepsilon_\mathrm{attn}\eta\Delta^\ell(i)}+\frac{4KTi\log i }{\eta\Delta^\ell(i)}.$$ Then for any iteration $t\geq \tau^{\ell,i}_*$, $\attn_{j(i)^{\ell,*},i}^\ell(t) > 1-\varepsilon_{\mathrm{attn}}.$
    \item Let $i>1$ be a root node. Then for any iteration $1\leq t \leq \tau^{\ell,i}_*$, any $j < i$,
    $$\Big|\attn_{j,i}^\ell(t) -\frac{1}{i-1}\Big| 
   \leq 64\bigg(\frac{1}{\varepsilon_\mathrm{attn}}\log{\frac{1}{\varepsilon_\mathrm{attn}}}+{i\log i}\bigg) \frac{\epsilon}{(i-1)^2\Delta^\ell(i)}.$$
    \end{enumerate}
\end{theorem}

\textbf{Non-root nodes concentrate.} For each non-root node \( i \), as long as $\Delta/4$
exceeds the estimation error $\epsilon$, each attention head \( \ell \) predominantly attends to the node \( j(i)^{\ell,*} \), which maximizes the KG-MI \( \tilde{I}^\ell_f(S_i; S_j) \) among all \( j < i \). Although \Cref{Thm: Attention Concentration} does not guarantee that, for a general 
$\tilde{I}^\ell_f(S_i; S_j)$, the maximizer $j(i)^{\ell,*}$ coincides with the $\ell$-th parent $p(i)^\ell$ of node $i$, our experiments in \Cref{fig:atten-score heatmap} with some widely used $f$, namely KL-KG-MI, and $\chi^2$-KG-MI, show that the corresponding heatmaps consistently recover the true adjacency matrix. This empirical evidence suggests that our proposed metric is practically effective in certain real-world scenarios.

For the KL-KG-MI, we can further establish the following theorem, showing that each attention head indeed attends to its correct parent, and thus the structure of the DAG can be  successfully learned.

\begin{theorem}\label{coro: KL conver}
If the convex function $f: \mathbb{R}_+\to\mathbb{R}$ is chosen to be  $f(x)=x\log x$, the $f$-KG-MI particularizes to the KL-KG-MI, which can be expressed as $$\tilde{I}^\ell_{\mathrm{KL}}(S_i; S_j)=\Eb_{ \Pi\sim P_\Pi}\left[\sum_{s,s'}{P_{S_j|\Pi}(s)\Pi^\ell(s'|s)}\log\bigg(\frac{P_{S_i,S_j|\Pi}(s',s)}{P_{S_i|\Pi}(s')P_{S_j|\Pi}(s)}\bigg)\right].$$ For the KL-KG-MI, $j(i)^{\ell,*}=p(i)^\ell$. As a result, under the same conditions as \Cref{Thm: Attention Concentration}, for any non-root node $i$, if the number of iterations $t\geq \tau^{\ell,i}_*$, the output of \Cref{alg:meta} satisfies $\attn_{p(i)^\ell,i}^\ell(t) > 1-\varepsilon_{\mathrm{attn}}.$
\end{theorem}
The above attention concentration phenomenon arises due to the unique properties of the KL divergence and the KL-KG-MI. Intuitively, for any non-root node $i$ and head $\ell$, identifying the node $j$ that maximizes $\tilde{I}^\ell_{\mathrm{KL}}(S_i; S_j)$ is equivalent to identifying the node $j$ whose joint distribution with $i$ is the closest to $P_{S_i|\Pi}(\cdot)\Pi^\ell(\cdot| \cdot)$ under the KL divergence. The node that optimally aligns its distribution with $P_{S_i|\Pi}(\cdot)\,\Pi^\ell(\cdot|\cdot)$   corresponds to the $\ell$-th parent of $i$, namely $p(i)^\ell$. The proof of Theorem~\ref{coro: KL conver} can be found in Appendix~\ref{appd: subs: KL}.

\textbf{Root nodes diverge.} In contrast, for root nodes, each attention head \( \ell \) diverges slightly. This occurs because, for any node $j$ preceding $i$, the true KG-MI $\tilde{I}^\ell_f(S_i; S_j)$ is identically zero. Consequently, the attention updates are dominated by  the estimation error $\epsilon$.

If $\epsilon$ is small, this irregularity does not significantly affect the convergence of the objective function. Intuitively, for root nodes $i$, the estimated $f$-\abname $\hat{I}^\ell_f(S_i; S_j)$ is always upper bounded by $\epsilon$, making it negligible compared to the \abname of non-root nodes and, consequently, almost not influencing the convergence behavior of the objective function.

\section{Experiments}\label{sec: experiment}

In this section, we present experiments using synthetic data to support our theoretical findings, and show that our methods outperform—or at least match—classical baselines on certain DAGs. 

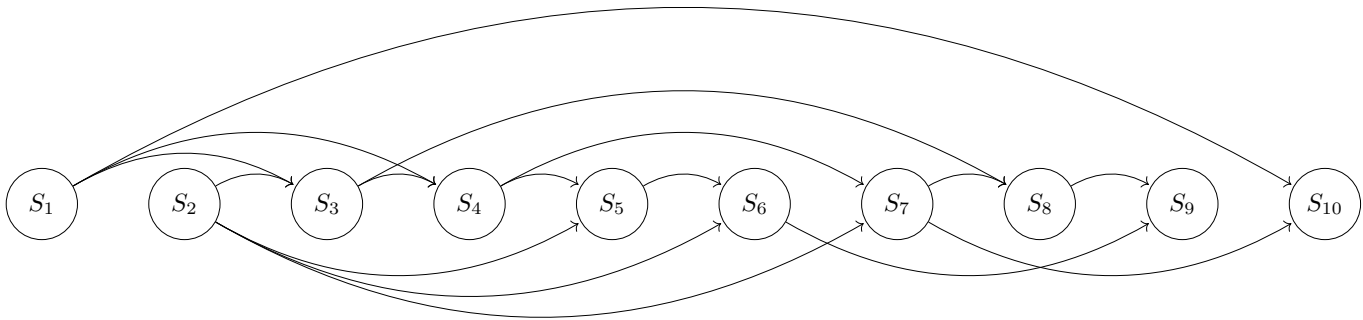
\begin{figure}[t]
    \centering
\resizebox{\textwidth}{!}{
\begin{tikzpicture}[
    ->, % Define arrow format
    node distance=2cm, % Distance of node
    every node/.style={circle, draw, minimum size=1cm}, % Node format
    shorten >=1pt % shorten arrow
]

% Define node
\node (S1) at (0, 0) {$S_1$};
\node (S2) [right of=S1] {$S_2$};
\node (S3) [right of=S2] {$S_3$};
\node (S4) [right of=S3] {$S_4$};
\node (S5) [right of=S4] {$S_5$};
\node (S6) [right of=S5] {$S_6$};
\node (S7) [right of=S6] {$S_7$};
\node (S8) [right of=S7] {$S_8$};
\node (S9) [right of=S8] {$S_9$};
\node (S10) [right of=S9] {$S_{10}$};

% Draw Edge
\draw (S1) edge[bend left] (S3);
\draw (S2) edge[bend left] (S3);

\draw (S1) edge[bend left] (S4);
\draw (S3) edge[bend left] (S4);

\draw (S4) edge[bend left] (S5);
\draw (S2) edge[bend right] (S5);

\draw (S5) edge[bend left] (S6);
\draw (S2) edge[bend right] (S6);

\draw (S4) edge[bend left] (S7);
\draw (S2) edge[bend right] (S7);

\draw (S7) edge[bend left] (S8);
\draw (S3) edge[bend left] (S8);

\draw (S8) edge[bend left] (S9);
\draw (S6) edge[bend right] (S9);

\draw (S1) edge[bend left] (S10);
\draw (S7) edge[bend right] (S10);

\end{tikzpicture}}
\caption{The structure of the meta-graph with 10 nodes. It includes two root nodes, 1 and 2, and eight non-root nodes, 3-10, each with an in-degree of 2. The edge set $E=\{(1,3),(1,4),(1,10),(2,3),(2,5),(2,6),(2,7),(3,4),(3,8),(4,5),(4,7),(5,6),$ $(6,9),(7,8),(7,10),(8,9)\}$.}
\label{fig:structure of 10 length graph}
\end{figure}

\subsection{Experiment Settings} \label{data gen}
We first introduce the \emph{data model}. In all the tasks, we consider order-2 Markov chains with a state space consisting of 3 states. We introduce meta-graphs, which are used to generate larger graphs by combining the meta-graphs. In the meta-graphs, the first two nodes are root nodes, and the remaining are non-root nodes with in-degree 2.  

\textbf{Graph Structure.} We have two meta-graphs, whose numbers of nodes are 5 and 10, and the graph with 5 nodes is a subgraph of the graph with 10 nodes. \Cref{fig:structure of 10 length graph} shows the structure of the meta-graph with 10 nodes, and the meta-graph with 5 nodes consists of the nodes from $S_1$ to $S_5$ with the edges only depending on $S_1$ to $S_5$.

\textbf{Transition Kernel.} For simplicity, we set the distribution  $P_\pi$  to be a Dirac delta distribution centered at a single transition kernel, i.e., we fix the transition kernel. The sample space $\mathcal{S} = \{0,1,2\}$, and the transition kernel $\pi$ are defined as follows (where the format below indicates the transition probabilities from the pair of parent states $(S_1,S_2)$ to each state $\{0,1,2\}$ as  $(S_1,S_2): \{0: \pi(0|S_1,S_2), \ 1: \pi(1|S_1,S_2), \ 2: \pi(2|S_1,S_2)\}$):
\begin{align}
& (0, 0): \{0: 0.1, \ 1: 0.5 - p_0, \ 2: 0.4 + p_0\}, \nonumber\\
& (0, 1): \{0: 0.2 + p_1, \ 1: 0.3, \ 2: 0.5 - p_1\}, \nonumber\\
& (0, 2): \{0: 0.3, \ 1: 0.4 - p_2, \ 2: 0.3 + p_2\}, \nonumber\\
& (1, 0): \{0: 0.5 - p_3, \ 1: 0.3, \ 2: 0.2 + p_3\}, \nonumber\\
& (1, 1): \{0: 0.4 + p_4, \ 1: 0.4 - p_4, \ 2: 0.2\}, \nonumber\\
& (1, 2): \{0: 0.2 + p_5, \ 1: 0.3, \ 2: 0.5 - p_5\}, \nonumber\\
& (2, 0): \{0: 0.6 - p_6, \ 1: 0.2 + p_6, \ 2: 0.2\}, \nonumber\\
& (2, 1): \{0: 0.1 + p_7, \ 1: 0.5 - p_7, \ 2: 0.4\}, \nonumber\\
& (2, 2): \{0: 0.2 + p_8, \ 1: 0.3, \ 2: 0.5 - p_8\}. \label{Eq: emp_transit} 
\end{align}
Here  $\{p_i\}_{i=0}^8$ represents a set of tunable probabilities which can be adjusted in different cases. For each sample, we follow the process introduced in Section \ref{subs: data model} to obtain the state of each position step by step according to the structure in Figure \ref{fig:structure of 10 length graph} and the above transition kernels.

%We next introduce our network model and training algorithm.

\textbf{Transformer Structure and Training Settings.} We use a transformer consisting of one attention layer with $2$ heads. We adopt  the GA-based procedure outlined in \Cref{alg:meta}. The initialization of the attention layer follows the settings in Section~\ref{subs: transformer architecture}.  We set the learning rate $\eta=10$.
%In accordance to   \Cref{alg:meta} (gradient ascent), we use a one-layer transformer consisting of only one attention layer with $2$ heads, and the initialization of the attention layer follows the settings in Section~\ref{subs: transformer architecture}.   We set the learning rate $\eta=10$.

\subsection{Experiment Results}
We first present our experiment results of \emph{attention score concentration}, \emph{convergence rate}, and \emph{impact of different choices of~$f$} on the meta-graph with 10 nodes as follows to validate our theoretical findings of \Cref{Thm: Attention Concentration,Thm: loss convergence}. 

\textbf{Attention Score Concentration.}
In \Cref{fig:atten-score heatmap}, we present the attention heatmaps generated by the trained transformer via GA for different KG-MIs. The heatmaps combine the outputs from both heads 1 and 2, with random sequences sampled using the meta-graph with 10 nodes. For the KL-\abname, we use the perfect (population) ground truth. For $\chi^2$-\abname, we estimate it using the method detailed in Appendix \ref{sec:example_pearsons}. For both KG-MIs, the heatmaps converge to the true adjacency matrix of the meta graph, corroborating our theoretical result in \Cref{Thm: Attention Concentration,coro: KL conver}. 
% While we do not provide theoretical guarantees for $\chi^2$-\abname like KL-\abname in \Cref{coro: KL conver}, \yl{do we want to explicitly say the above sentence?} in practice, it still effectively guides the transformer to learn the correct adjacency matrix.}
%\yh{rearrange the fig} 

\begin{figure*}[htbp]
    \centering
    \begin{tabular}{c@{\hspace{5pt}}c@{\hspace{5pt}}c@{\hspace{5pt}}c@{\hspace{5pt}}}
    DAG $\mathcal{G}$ & \quad ~~ Heatmap of KL (Population) & \quad ~~ Heatmap of $\chi^2$ (Estimate) &~~ \\ \includegraphics[width=0.26\textwidth]{IT/fig/Heatmap/heatmap_of_graph_w_num.pdf} &\quad ~
    \includegraphics[width=0.26\textwidth]{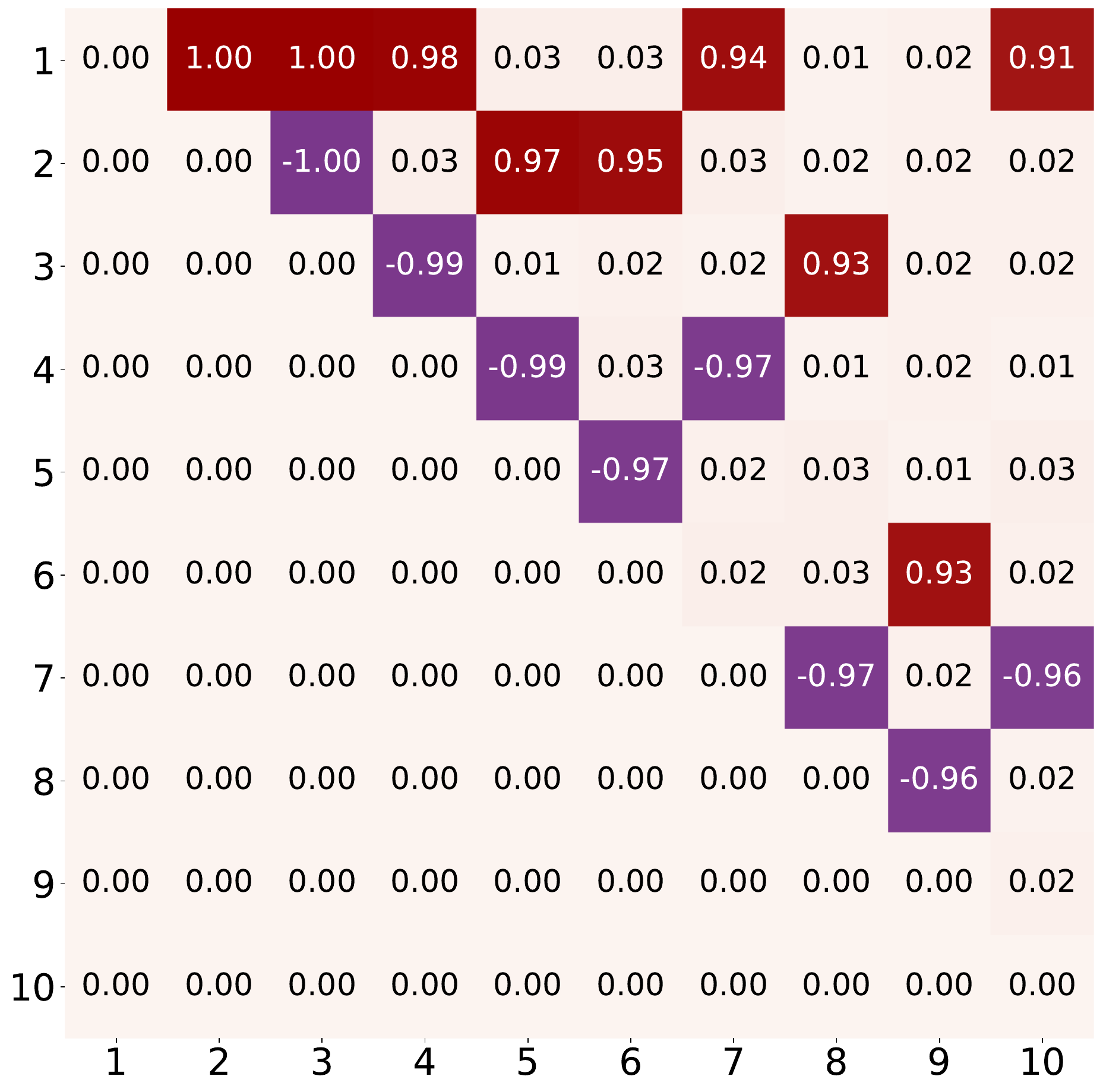} &\quad~
    \includegraphics[width=0.26\textwidth]{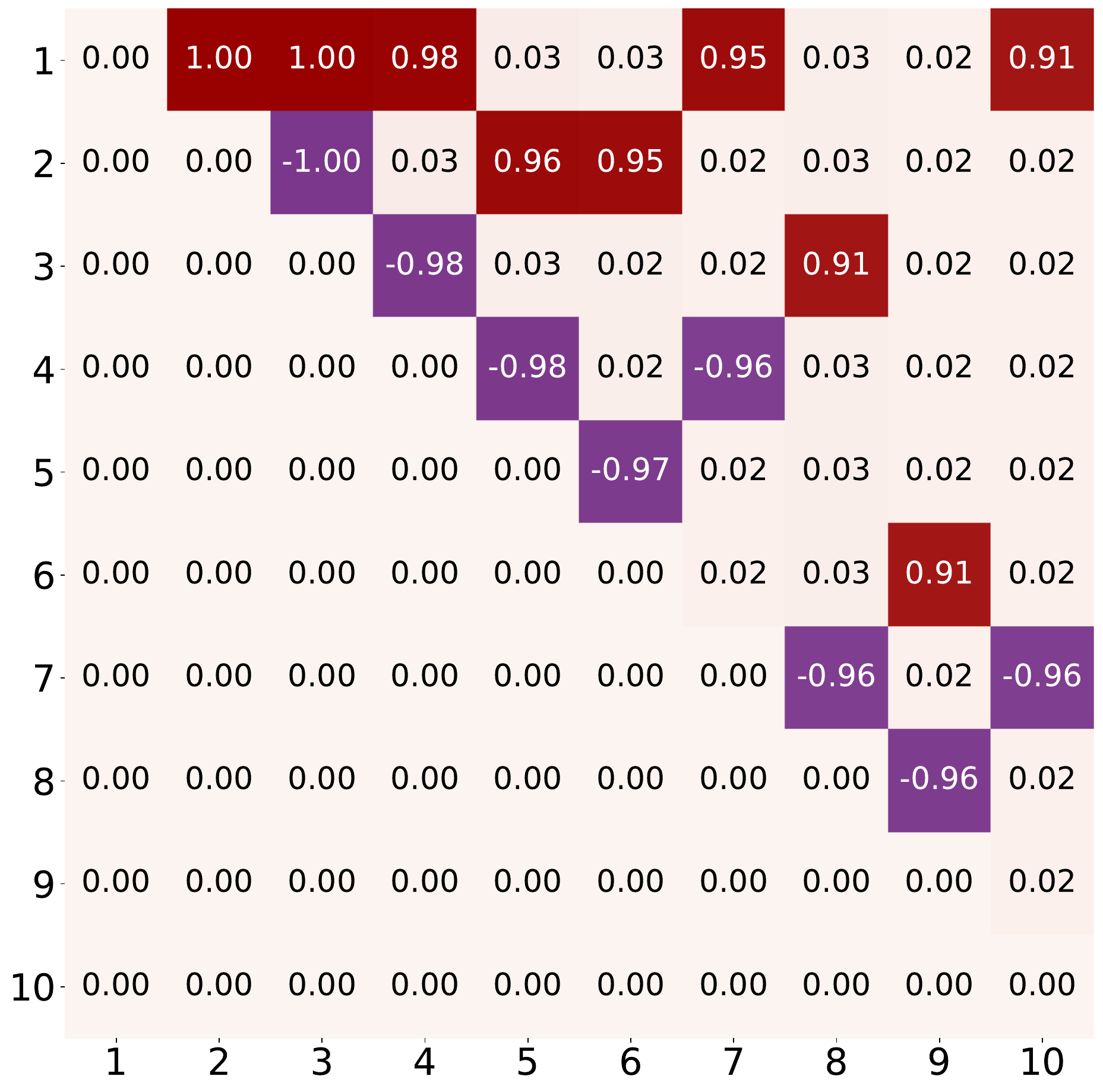} &~~
    \put(0,12){\includegraphics[width=0.05\textwidth]{IT/fig/Heatmap/colorbar_w_num.pdf}}
    \\
    \end{tabular}
    \caption{\textbf{The attention scores of the trained transformer using KG-MI.} We compare the heatmaps between the adjacency matrix of graph $\Gc$ (left), and the trained attention scores trained with true KL-\abname (middle) and with estimated $\chi^2$-\abname (right). To present the learned multi-head attention scores, for non-root nodes (nodes $3$ to $10$ above), we add the attention score of both heads $1,2$ in our transformer into one heatmap and change the highest score of head $2$  to be negative to distinguish it from head $1$. In particular, the first two nodes of the graph are root nodes. As a result, for node $1$, heads $1$ and $2$ will not attend to anything and for node $2$, heads $1$ and $2$ will attend to node $1$. In conclusion, both heatmaps converge to the true adjacency matrix of $\Gc$.}
    % for each column, the two patches with value 1 represent the two parent nodes of the corresponding position. For example, the parents of position 6 are position 2 and position 5.
    \label{fig:atten-score heatmap}
\end{figure*}

\textbf{Convergence Rate.} In \Cref{fig:convergence rate}, we plot how the convergence rates of the objective function depend on the length of the random sequences $T$ (\Cref{fig:relation to sequence length}) and information gap $\Delta$ (\Cref{fig: relation to gap}), where the convergence rate is evaluated by the stopping epoch to achieve $\varepsilon_{\mathrm{attn}}$. For the length $T$, error $\varepsilon_{\mathrm{attn}}=0.1$, we generate a meta-sequence by the meta-graph and identical transition kernel. We obtain sequences of different lengths by appending meta-sequences repeatedly, which ensures that the information gap $\Delta$ of sequences with different lengths is identical. For the information gap $\Delta$, error $\varepsilon_{\mathrm{attn}}=0.001$, we fix the meta-graph and change the transition kernel to get different information gaps $\Delta$. \Cref{fig:convergence rate} illustrates that the stopping epoch scales linearly in both $T^2\log T$ and $\frac{1}{\Delta}$, which matches the upper bounds given in \Cref{Thm: loss convergence}, verifying that our upper bounds are rather tight in some practical scenarios.
%, which specifies the number of epochs required to achieve  $\varepsilon_{\mathrm{attn}}=0.1$. Thus, even though we do not have a lower bound, the experiments suggest that our upper bounds are rather tight in some practical scenarios.
% Form Based on the lower bound derived in Theorem \ref{Thm: meta-population}, for the reliance on $T$, we provide the relationship between convergence epoch and $T^2\log(T)$ and for gap we show the relation between the convergence epoch and $\frac{1}{\Delta}$. Both of the relation are built on the graph structure with 10 basic positions. As we can see, both of the relationship with regard to $T$ and $\Delta$ are linear, which coincides with the lower bound $\tau_{*}$.

\begin{figure*}[htbp]
    \centering
\begin{subfigure}[t]{0.3\textwidth}
    % \begin{subfigure}{0.48\textwidth}
    \centering
    \includegraphics[width=0.95\linewidth]{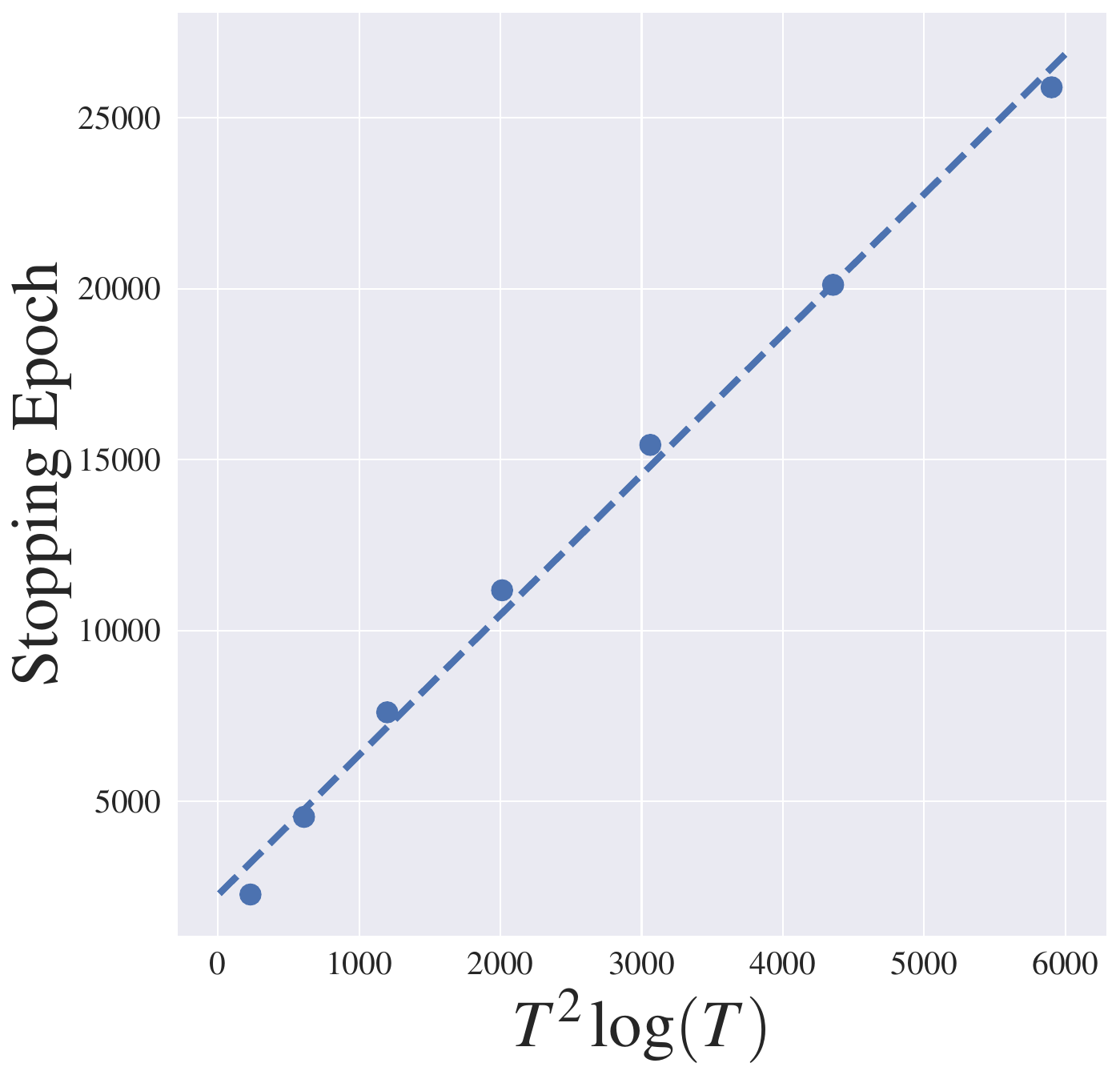}
    \caption{Random Sequences Length $T$}
    \label{fig:relation to sequence length}
    % \end{subfigure}
    % \begin{subfigure}{0.48\textwidth}
        % \includegraphics[width=\linewidth]{icml2025/figures/wrt_gap_aux.pdf}
        % \caption{Information gap $\Delta$}
        % \label{fig: relation to gap}
            
    % \end{subfigure}
\end{subfigure}
\hspace{-0em}
\begin{subfigure}[t]{0.3\textwidth}
\centering
\includegraphics[width=0.95\linewidth]{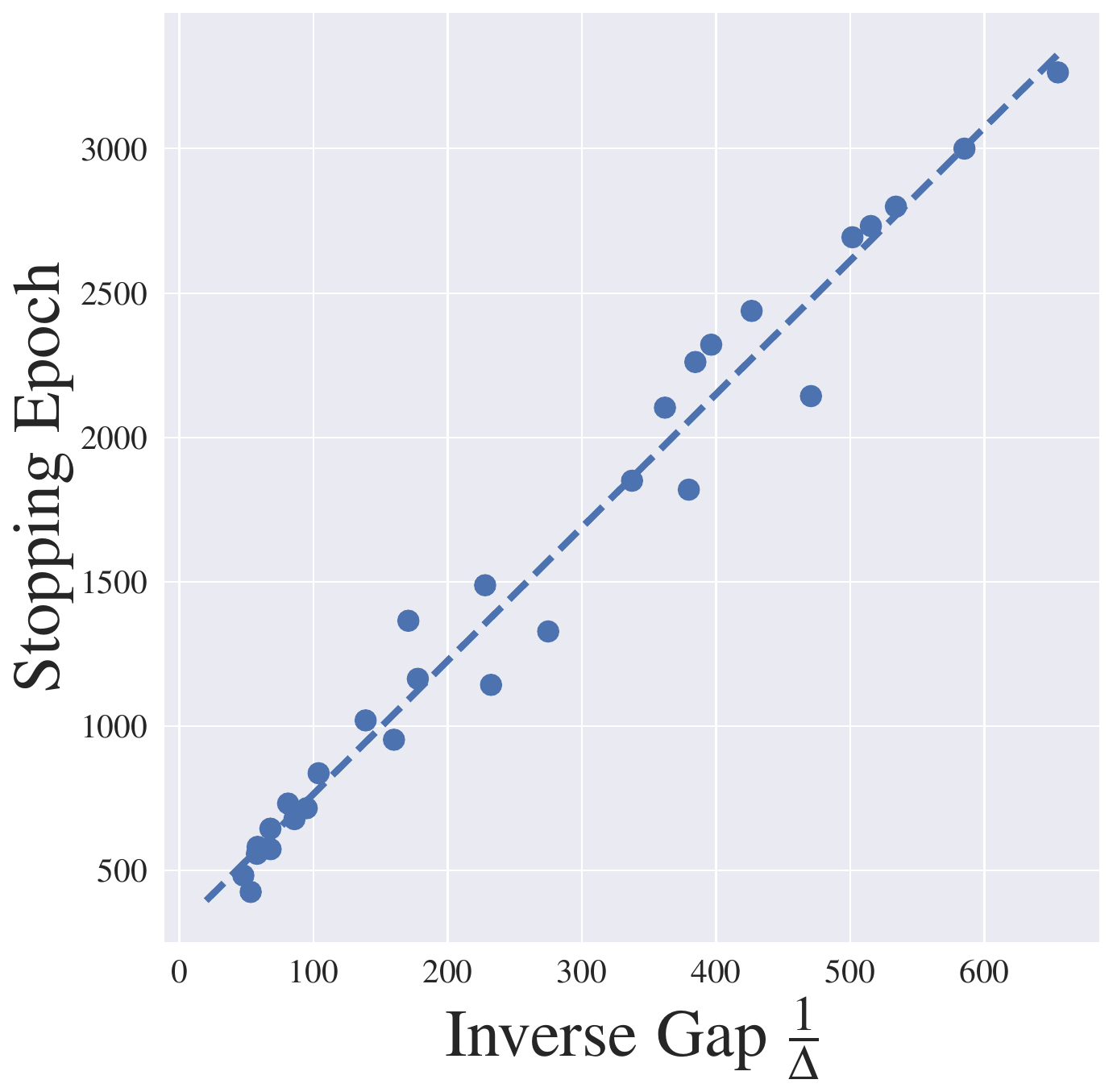}
    \caption{Information gap $\Delta$}
        \label{fig: relation to gap}
\end{subfigure}
\hspace{-0em}
\begin{subfigure}[t]{0.3\textwidth}
\includegraphics[width=0.95\linewidth]{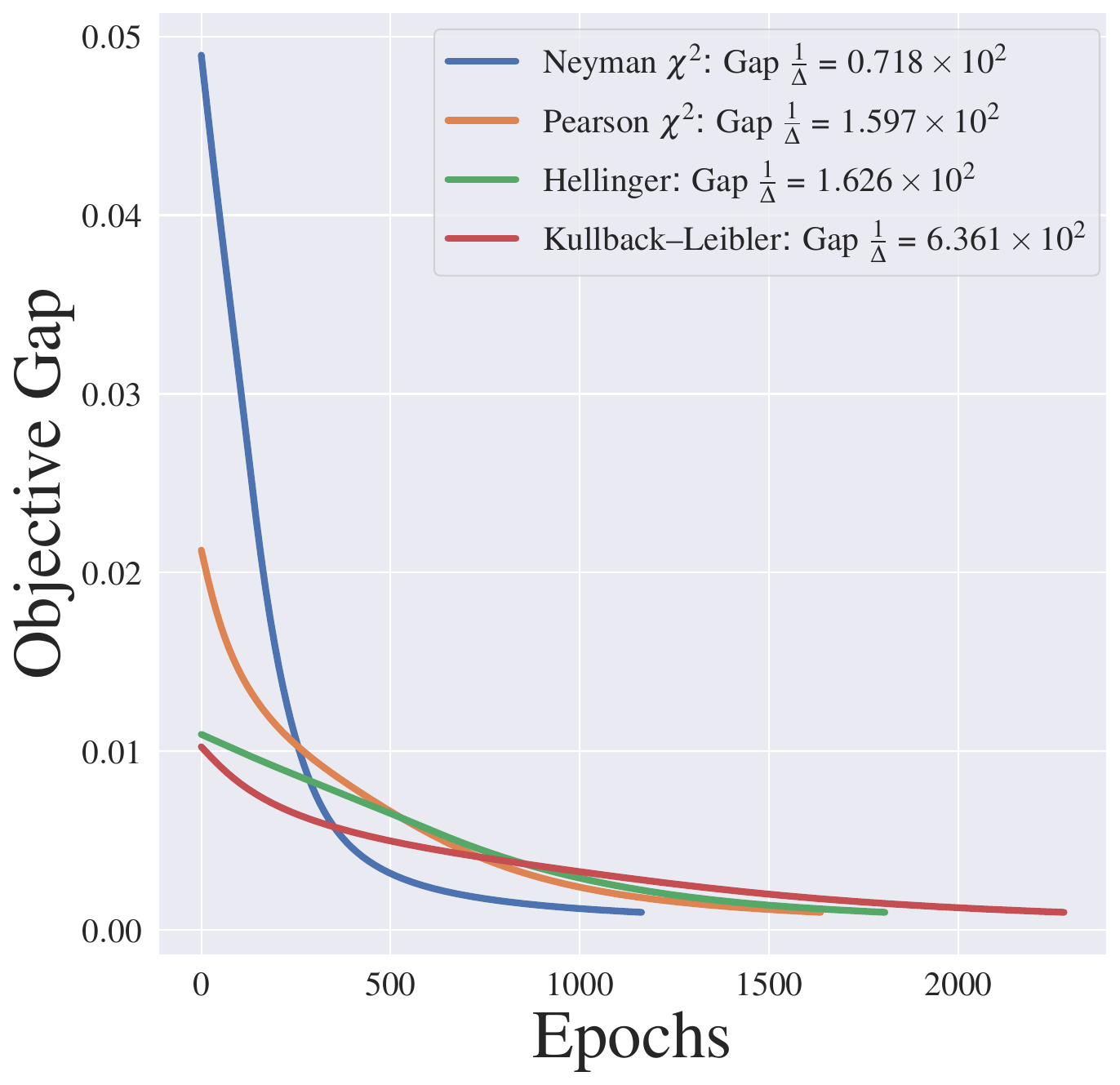}
    % \hspace{0.416em}
    \caption{$f$-\abname for different $f$'s.}
    \label{fig: compreh f mi}
\end{subfigure}
\caption{Convergence rates w.r.t.\ the length of the random sequences $T$ (subfigure a) and information gap $\Delta$ (subfigure  b). Comprehensive study of different $f$-\abname's for the sample random sequences (subfigure  c). }
    \label{fig:convergence rate}
\end{figure*}

\textbf{Impact of different choices of $f$ in \abname.} In \Cref{fig: compreh f mi}, we demonstrate how the suboptimality gap (the gap between the objective function and its maximum value) decays with respect to the number of epochs (iterations) for different choices of $f$ that result in KL, Hellinger's, Pearson's $\chi^2$- and Neyman's $\chi^2$-\abname. For the same data model, where the random sequences are of length 10, and are generated by the meta-graph with an identical transition kernel, different choices of $f$-\abname have different $\Delta$, leading to different convergence rates. Among the four KG-MIs presented in \Cref{fig: compreh f mi}, the larger $\Delta$ of the $f$-\abname, the faster the objective converges. We also observe that $\Delta$ of Pearson's $\chi^2$-\abname and Hellinger's \abname are close and have similar convergence rates, which 
corroborates the upper bounds given in \Cref{Thm: loss convergence} and is consistent with the empirical results in \Cref{fig:convergence rate}.

\textbf{Scalability with length $T$.} To analyze the scalability with length $T$, we have extended our experiments to larger values of $T$, up to 100. As in \Cref{fig: large T}, we conduct a comprehensive study of different $f$-KG-MI's for $T=100$. We observe that a larger information gap $\Delta$ in $f$-KG-MI still leads to a faster convergence rate. Specifically, $\chi^2$-KG-MI, which has the largest $\Delta$ among the three KG-MIs, converges the fastest, while KL-KG-MI, with the smallest $\Delta$, converges the slowest.

\subsection{Comparison with Baseline Algorithms}

\begin{figure*}[ht]
\centering
\begin{minipage}{0.3\textwidth}
    \centering
\includegraphics[width=0.95\linewidth]{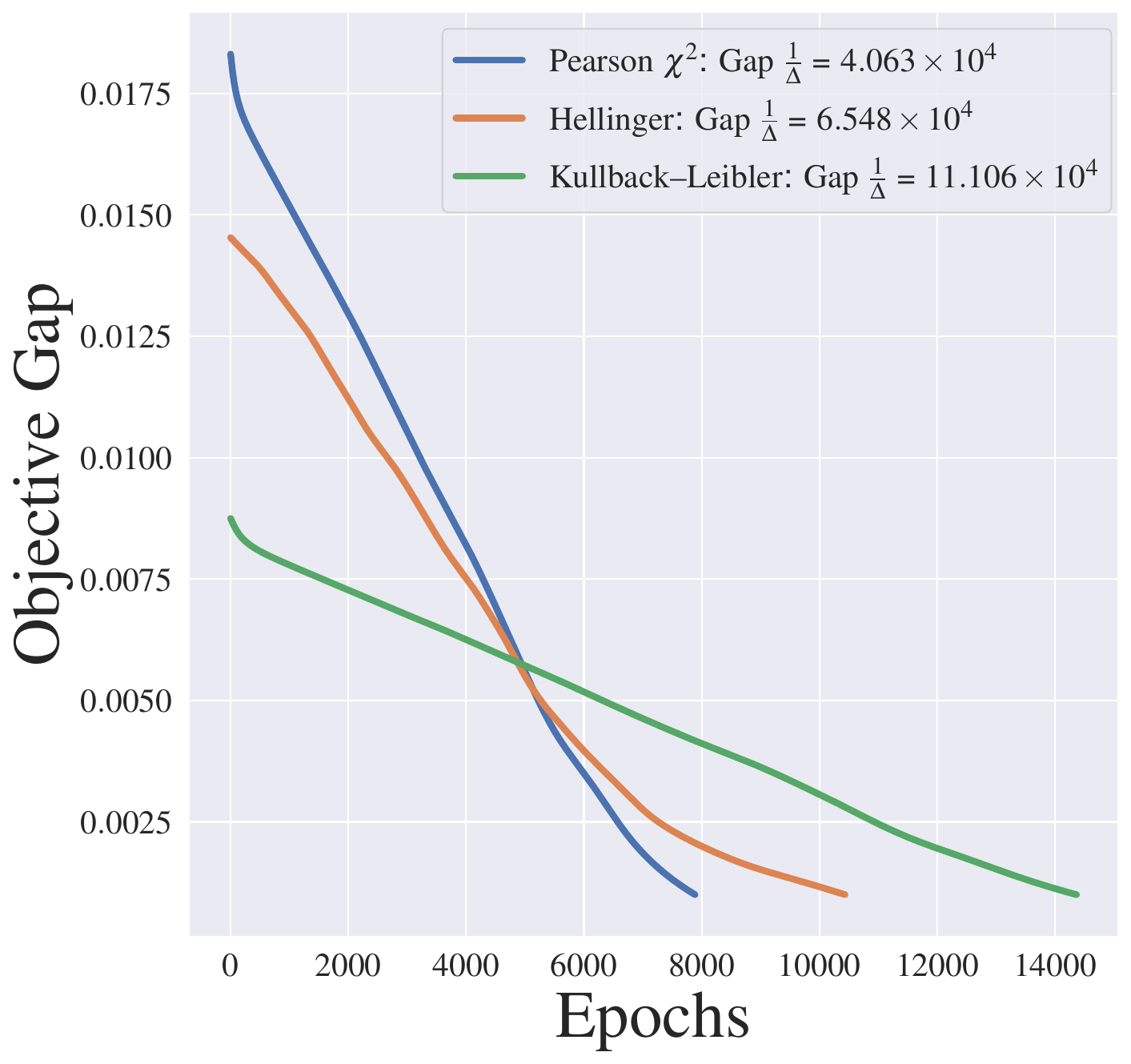}
    \caption{Comprehensive Study of different $f$-\abname's with large $T=100$.}
    \label{fig: large T}
\end{minipage}%
\hfill
\begin{minipage}{0.7\textwidth}
    \centering
\begin{tabular}{lccc}
\toprule
\textbf{Method} & \textbf{F1 ($\uparrow$)} & \textbf{SHD ($\downarrow$)} & \textbf{Running Time (s)} \\
\midrule
NOTEARS~\citep{zheng2018dags} & 0.210 $\pm$ 0.103 & 22.20 $\pm$ 2.588 & 7.279 $\pm$ 0.867 \\
PC~\citep{spirtes1991algorithm}     & 0.267 $\pm$ 0.033 & 11.60 $\pm$ 0.548 & 0.096 $\pm$ 0.012 \\
MMHC~\citep{tsamardinos2006max}    & 0.581 $\pm$ 0.050 & 6.000 $\pm$ 1.000 & 1.830 $\pm$ 0.043 \\
HC~\citep{chickering2002optimal}      & 1.000 $\pm$ 0.000 & 0.000 $\pm$ 0.000 & 1.134 $\pm$ 0.045 \\
A*~\citep{hart1968formal,yuan2013learning}     & 1.000 $\pm$ 0.000 & 0.000 $\pm$ 0.000 & 1.113 $\pm$ 0.035 \\
\textbf{Ours (KG-KL-MI)} & \textbf{1.000 $\pm$ 0.000} & \textbf{0.000 $\pm$ 0.000} & \textbf{0.057 $\pm$ 0.001} \\
\bottomrule
\end{tabular}
\captionof{table}{Comparison of different methods.  The \textbf{bold} represents the best.}
\label{tab:Performance comparison}
\end{minipage}
\end{figure*}

%\textbf{Comparison with baselines.}
We next show that our transformer-based method in \Cref{alg:meta} outperforms—or at least matches—classical baselines on the same data model. Specifically, we evaluate the performance across five independently sampled datasets, each containing $10,000$ samples. Our comparison includes five established methods: NOTEARS~\cite{zheng2018dags}, PC~\cite{spirtes1991algorithm}, MMHC~\cite{tsamardinos2006max}, HC~\cite{chickering2002optimal}, and A$^*$ search~\cite{hart1968formal,yuan2013learning}. We employ both the  F1 score~\cite{van2004geometry} and the Structural Hamming Distance (SHD)~\cite{tsamardinos2006max} as evaluation metrics.  NOTEARS, HC, and A$^*$ search were adapted to incorporate the provided topological ordering. As shown in \Cref{tab:Performance comparison}, our transformer-based method, together with HC and A$^*$ search, achieves comparable performance, accurately recovering the entire graph structure in terms of both the F1 score and the SHD. Importantly, our method also outperforms others in computational efficiency, completing the task in a mere 0.057 seconds---significantly faster than all competing techniques. This remarkable speed underscores the efficiency gains enabled by our transformer-based framework. The key advantage stems from our parameterized model, which facilitates effective graph learning within a principled optimization paradigm, leveraging well-established algorithms to achieve both accuracy and speed.

\section{Conclusion}\label{sec: conclusion}

In this work, we address the critical challenge of training transformers to accurately learn multi-parent DAGs. We introduce a novel metric, KG-MI, and develop an objective function that integrates KG-MI with a multi-head attention mechanism. Our approach demonstrates that optimizing this objective via gradient ascent effectively enables the model to distinguish multiple parent nodes for each node. Moreover, by leveraging estimates of the KG-MIs, we establish that our framework guarantees convergence to the global optimum within polynomial time. We further analyze the attention patterns at convergence, showing that when the $f$-divergence is set to the KL divergence, the trained transformer's attention scores precisely recover the DAG's adjacency matrix, and hence, the structure of the underlying DAG.

Our theoretical findings provide a firm theoretical foundation for the use of attention-based models in graph structure learning. Experimental results and analyses indicate that our approach not only enhances understanding of attention mechanisms but also offers practical guidance for selecting suitable $f$-divergences to improve training efficiency and accuracy. While our graph structure learning guarantees are currently specific to the KL divergence setting, our results open promising avenues for extending these insights to other $f$-divergences, paving the way for broader applicability in complex graph learning tasks.

\allowdisplaybreaks

\begin{appendices}

\section{Extension to DAGs with non-uniform in-degrees.}\label{app: ext to non uni K}

The proposed method can be extended to DAGs with non-uniform in-degrees, as long as the transition kernels are compatible in a marginal sense. Formally, suppose the maximum in-degree is $K$, and some node $i$ has a smaller in-degree $K_i \leq K$. We denote the transition kernel corresponding to $K_i$ parents as $\pi_{K_i}$, and the transition kernel corresponding to $K$ parents as $\pi_{K}$. If the transition kernel $\pi_{K_i}(s'|s_1,\ldots,s_{K_i})$ shares the same marginal stationary distribution $\mu_\pi$ as $\pi_K$, and is the ``\emph{marginal distribution}'' of $\pi^K$ in the sense that 
\begin{align}
    \pi_{K_i}(s'|s_1,\ldots,s_{K_i})=\sum_{s_{K_{i+1}},\ldots,s_K}\pi_K(s'|s_1,\ldots,s_{K})\mu_\pi(s_{K_{i+1}})\ldots\mu_\pi(s_K). \label{Eq: condition of margin}
\end{align}

Under this setup, our method using $K$-head attention and the KG-MI loss remains applicable. Take KG-KL-MI as an example: for a node $i$ with $K_i$ parents, there are two cases. 
\begin{enumerate}
    \item The $K_i$ heads responsible for those marginals can successfully recover the true $K_i$ parents.
    \item The remaining $K-K_i$ heads will either (1) \textbf{redundantly concentrate} on the same parent set, or (2) fail to concentrate and assign attention to irrelevant non-parent nodes diffusely, similar to the behavior observed in the attention scores of root nodes described in \Cref{Thm: Attention Concentration}.
\end{enumerate}

In both scenarios, \textbf{the true parent set is identifiable}, by selecting the concentration heads out. To support this claim, we provide an empirical study on synthetic data with non-uniform in-degrees, demonstrating that our framework still recovers the correct structure in such settings.

We next provide an experiment on synthetic data to validate our claim above.

\begin{figure}[t]
    \centering
\resizebox{\textwidth}{!}{
\begin{tikzpicture}[
    ->, % Define arrow format
    node distance=2cm, % Distance of node
    every node/.style={circle, draw, minimum size=1cm}, % Node format
    shorten >=1pt % shorten arrow
]

% Define node
\node (S1) at (0, 0) {$S_1$};
\node (S2) [right of=S1] {$S_2$};
\node (S3) [right of=S2] {$S_3$};
\node (S4) [right of=S3] {$S_4$};
\node (S5) [right of=S4] {$S_5$};
\node (S6) [right of=S5] {$S_6$};
\node (S7) [right of=S6] {$S_7$};
\node (S8) [right of=S7] {$S_8$};
\node (S9) [right of=S8] {$S_9$};
\node (S10) [right of=S9] {$S_{10}$};

% Draw Edge
\draw (S1) edge[bend left] (S3);
% \draw (S2) edge[bend left] (S3);

\draw (S2) edge[bend left] (S4);
\draw (S3) edge[bend left] (S4);

\draw (S1) edge[bend left] (S5);
\draw (S4) edge[bend right] (S5);

\draw (S3) edge[bend left] (S6);
\draw (S4) edge[bend right] (S6);

\draw (S5) edge[bend left] (S7);
% \draw (S2) edge[bend right] (S7);

\draw (S6) edge[bend left] (S8);
\draw (S7) edge[bend left] (S8);

\draw (S8) edge[bend left] (S9);
% \draw (S6) edge[bend right] (S9);

\draw (S6) edge[bend left] (S10);
\draw (S9) edge[bend right] (S10);

\end{tikzpicture}}
\caption{The structure of the meta-graph with 10 nodes which includes two root nodes, 1 and 2, and eight non-root nodes, 3-10, among which nodes $3,7,9$ have in-degrees $1$, and nodes $4,5,6,8,10$ have in-degrees $2$. The edge set $\widetilde{E}=\{(1,3),(2,4),(3,4),(1,5),(4,5),(3,6),(4,6),(5,7),(6,8),(7,8),(8,9),(6,10),(9,10)\}$.}
\label{fig:structure of 10 length graph_non unifrm}
\end{figure}
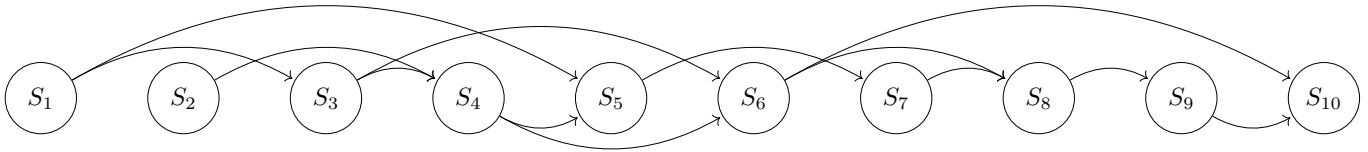

\textbf{Data Model.}
As in \Cref{fig:structure of 10 length graph_non unifrm}, we consider a different DAG $\widetilde{\mathcal{G}}=([T],\widetilde{E})$ from  the previous ${\mathcal{G}}$ with $T=10$ and maximum in-degree of $2$, and the in-degree can be $0,1,2$. For the transition kernel, we consider the same joint kernel as in \eqrefn{Eq: emp_transit} with  marginal kernels satisfying \eqrefn{Eq: condition of margin}.

\begin{figure*}[t]
    \centering
    \begin{tabular}{c@{\hspace{5pt}}c@{\hspace{5pt}}c@{\hspace{5pt}}c@{\hspace{5pt}}}
    DAG $\mathcal{G}$ & \quad ~~ Heatmap of KL (Head 1) & \quad ~~ Heatmap of KL (Head 2) &~~ \\ \includegraphics[width=0.26\textwidth]{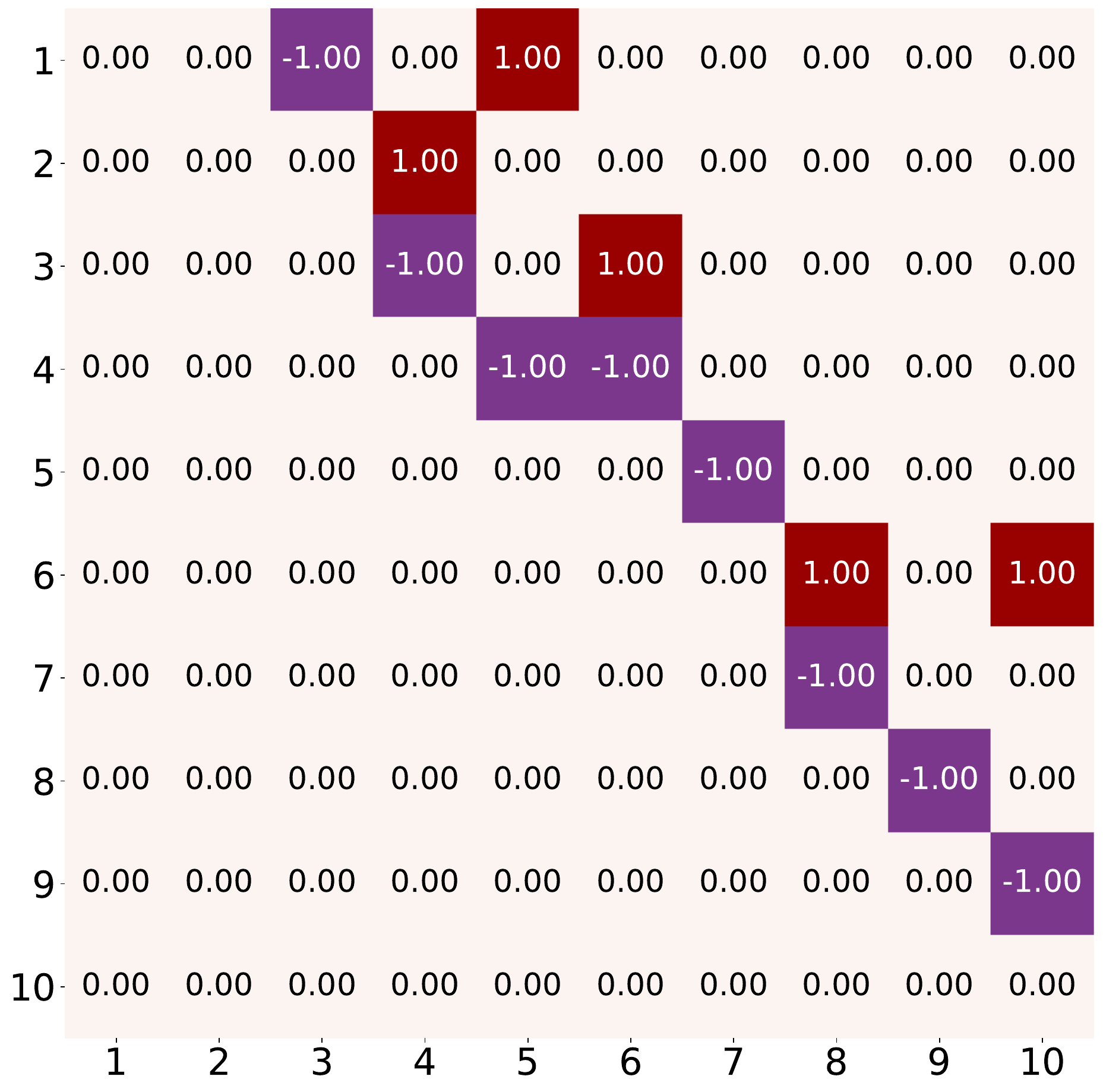} &\quad ~
    \includegraphics[width=0.26\textwidth]{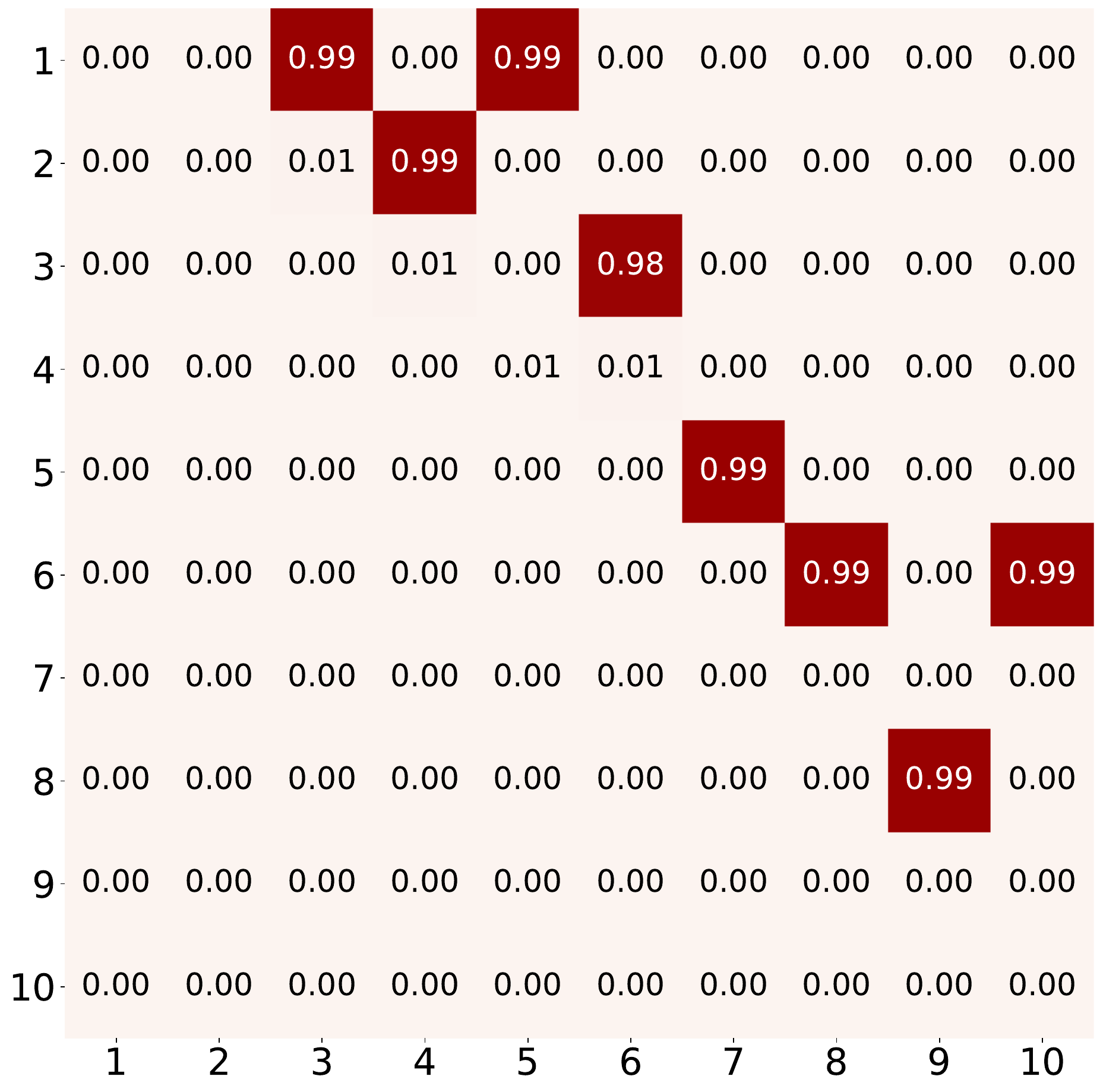} &\quad~
    \includegraphics[width=0.26\textwidth]{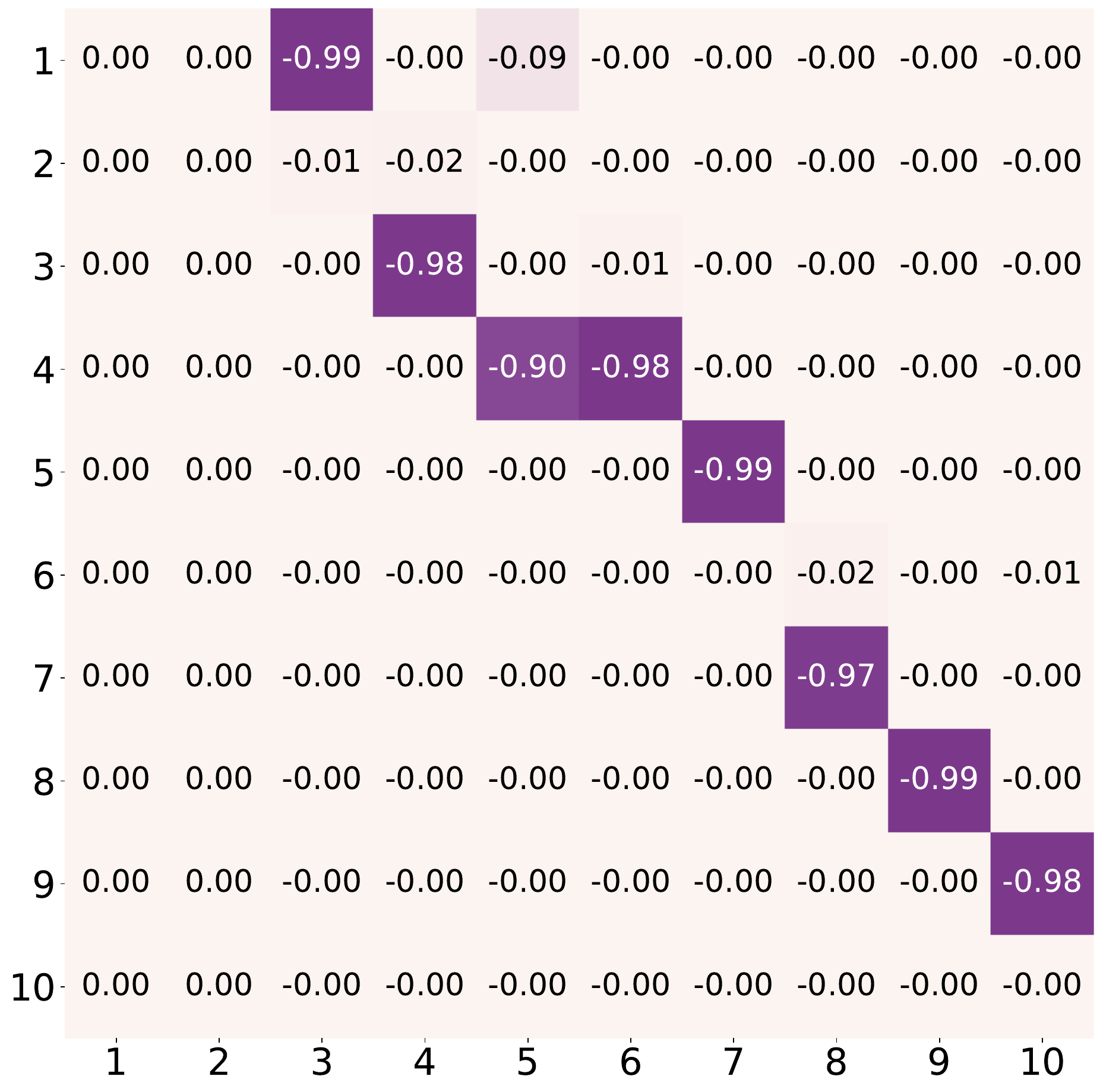} &~~
    \put(0,12){\includegraphics[width=0.05\textwidth]{IT/fig/Heatmap/colorbar_w_num.pdf}}
    \\
    \end{tabular}
    \caption{\textbf{Learning general DAG with non-uniform in-degrees.}  
\Cref{fig:general dag} compares heatmaps of the true adjacency matrix of graph $\mathcal{G}$ (first column) with the attention patterns learned by two heads using KL mutual information under a general DAG with non-uniform in-degrees.
}
    % for each column, the two patches with value 1 represent the two parent nodes of the corresponding position. For example, the parents of position 6 are position 2 and position 5.
    \label{fig:general dag}
\end{figure*}

As  shown in \Cref{fig:general dag}, the attention scores successfully capture the structure of the underlying graph $\mathcal{G}$. In particular, for nodes with in-degree 1, there are two cases. 
\begin{enumerate}
    \item \textbf{Nodes 3 and 9:} both attention heads \textbf{redundantly concentrate on the same single parent}, with a combined attention score of two heads close to 2. This corresponds exactly to the first case discussed above.
    \item \textbf{Node 7:} one head correctly concentrates on its single parent (node 5), while the other head fails to concentrate and instead diffusely assigns attention to irrelevant, non-parent nodes—matching the second case described above.
\end{enumerate}

\section{Discussion of Properties of Transition Kernels}\label{sec: MC concentration}
We start by introducing notation to be used in this section and the remainder of the appendix. Let $I_{i,j}$ be the $T \times T$ matrix with 1 at the $(i,j)$-th position, and 0 elsewhere, and let $1_T$ be the $T$-dimension vector with each entry being 1. We write $a \lesssim b$ if there exists an absolute constant $C > 0$ not depending on $a,b$ such that $a \leq C \cdot b$. We write $a=\mathrm{poly}(b)$ if $a=O(\mathrm{poly}(b))$.

In this section, we prove certain statistical properties of $K$-parent DAGs in our paper. To prove these lemmas, it suffices to consider the special case of a $K$-th-order Markov chain, where each non-root node has exactly its $K$ immediate predecessors as parents, and to apply standard results from Markov chain theory. The extension to a general DAG then follows.
 
\subsection{Discussion on the Stationary Distribution}\label{Appd: 3}
We first establish the existence and uniqueness of the stationary distribution and its marginals, which ensures that the definition of the stationary distribution and its marginals in main body is well-founded.
\begin{lemma}[Existence of a Stationary Distribution and Identical Marginals]\label{lem:sta dis}
Consider a stationary transition kernel $\pi\in\mathcal{P}(\Sc|\Sc^K)$, and its corresponding $\tilde{\pi}$ (constructed according to Section~\ref{subs: data model}),   defined as  $\widetilde{\pi}\bigl( s'_1, s'_2, \ldots, s'_K | s_1, s_2, \ldots, s_K \bigr) =\pi\bigl(s'_K | s_1, s_2, \ldots, s_K\bigr)$, if $s'_i = s_{i+1} \text{ for } i = 1,\dots,K-1$, and 0, otherwise. If the transition kernel $\pi$ is positive in the sense that $\pi(s'|s_1,\ldots,s_K)>0$ for all $(s',s_1,\ldots,s_K)$, then 
\begin{itemize}
    \item \textbf{Existence and uniqueness of a stationary distribution.} There exists a unique stationary distribution 
    $M_{\widetilde{\pi}} \in \mathcal{P}(\Sc^K)$ of $\widetilde{\pi}$.
    \item \textbf{Identical marginals.}
    Define the stationary distribution of $\pi$ as $M_\pi=M_{\widetilde{\pi}}$. Define the marginal distribution of $M_{\pi}$ on the $i$-th coordinate as $\mu_{\pi,i}(s)=\sum_{j\neq i}\sum_{s_j\in\Sc}M_{\pi}(s_1,\ldots,s_{i-1},s,s_{i+1},\ldots,s_K), s\in \Sc$. The marginal distributions $\mu_{\pi,i}$ are identical across all $i \in [K]$. 
\end{itemize}
\end{lemma}

\begin{proof}
\textbf{Result (1): Existence and uniqueness of a stationary distribution.}

It suffices to show that any tuple of states $(s'_1, \ldots, s'_K)$ is accessible from any other tuple of states $(s_1, \ldots, s_K)$ under the transition kernel $\widetilde{\pi}$. This establishes the irreducibility of the induced Markov chain on  the expanded state space $\mathcal{S}^K$ and, consequently, the existence and uniqueness of its stationary distribution.

We next prove that any tuple of states  $(s'_1, \ldots, s'_K)$ is accessible from any other $(s_1, \ldots, s_K)$. From the positivity of $\pi$, for any $2\leq j \leq K-1$, we have  $\widetilde{\pi}(s_{j+1},\ldots,s_K,s'_1,\ldots,s'_j|s_{j},\ldots,s_K,s'_1,\ldots,s'_{j-1})={\pi}(s'_j|s_{j},\ldots,s_K,s'_1,\ldots,s'_{j-1}) > 0$. In addition, $\widetilde{\pi}(s'_1,\ldots,s'_K|s_K,s'_1,\ldots,s'_{K-1})={\pi}(s'_K|s_K,s'_1,\ldots,s'_{K-1}) > 0$, and $\widetilde{\pi}(s_{2},\ldots,s_K,s'_1|s_{1},\ldots,s_K)={\pi}(s'_1|s_{1},\ldots,s_K) > 0$. Let $\widetilde{\pi}^{(K)}$ represent the $K$-step transition kernel, i.e., the probability distribution of being in a new state from a given initial state after $K$ steps. We only need to prove that $\widetilde{\pi}^{(K)}(s'_1, \ldots, s'_K|s_1, \ldots, s_K) > 0$. For this, consider,
\begin{align*}
    &
    \widetilde{\pi}^{(K)}(s'_1, \ldots, s'_K|s_1, \ldots, s_K)  \\
    & \quad
    \geq 
    \widetilde{\pi}(s_{2},\ldots,s_K,s'_1|s_{1},\ldots,s_K)\prod_{j=2}^{K-2}\widetilde{\pi}(s_{j+1},\ldots,s_K,s'_1,\ldots,s'_j|s_{j},\ldots,s_K,s'_1,\ldots,s'_{j-1}) \widetilde{\pi}(s'_1,\ldots,s'_K|s_K,s'_1,\ldots,s'_{K-1})\\
    & 
    \quad
    =
    {\pi}(s'_1|s_{1},\ldots,s_K)\prod_{j=2}^{K-2}{\pi}(s'_j|s_{j},\ldots,s_K,s'_1,\ldots,s'_{j-1}) {\pi}(s'_K|s_K,s'_1,\ldots,s'_{K-1})
    > 0,
\end{align*}
where the last inequality follows from the strict positivity of $\pi$. This completes the proof.

\textbf{Result (2): Identical marginals.}

From its definition, the stationary distribution $M_\pi$ satisfies 
\begin{equation}\label{Eq: multi-parent stationary}
M_\pi\bigl(s'_1,\dots,s'_K\bigr)
    = 
    \sum_{s_1,\dots,s_K}
    \widetilde{\pi}\bigl(\,s'_1,\dots,s'_K\,\big|\,s_1,\dots,s_K\bigr)\,
    M_\pi\bigl(s_1,\dots,s_K\bigr).
\end{equation}
By definition, $\mu_{\pi,1}(s_1') 
 :=
\sum_{s_2',\dots,s_K'} 
M_\pi\bigl(s_1',\dots,s_K'\bigr).$
Summing over $s_2',\ldots,s_K'$ on both two sides of \eqrefn{Eq: multi-parent stationary}, we obtain
\begin{align*}
    \mu_{\pi,1}(s_1')&=\sum_{s_2',\ldots,s_K'}M_\pi(s_1',\ldots,s_K')\\
    &=\sum_{s_2',\ldots,s_K'}\sum_{s_1,\ldots,s_K}\widetilde{\pi}\left( s'_1, s'_2, \ldots, s'_K \,\big|\, s_1, s_2, \ldots, s_K \right)M_\pi(s_1,\ldots,s_K)\\
    &=\sum_{s_2',\ldots,s_K'}\sum_{s_1,\ldots,s_K}\left\{\widetilde{\pi}\left( s'_1, s'_2, \ldots, s'_K \,\big|\, s_1, s_2, \ldots, s_K \right)M_\pi(s_1,\ldots,s_K)\mathbbm{1}_{ \{ s'_1=s_2,\ldots,s'_{K-1}=s_K \} }\right\}\\
    &=\sum_{s_K'}\sum_{s_1,s_3\ldots,s_K}\widetilde{\pi}\left( s'_1, s_3, \ldots,s_K, s'_K \,\big|\, s_1, s_2, \ldots, s_K \right)M_\pi(s_1,s'_1,s_3\ldots,s_K)\\
    &=\sum_{s_K'}\sum_{s_1,s_3\ldots,s_K}\pi\left(s'_K \,\big|\, s_1, s'_1, \ldots, s_K \right)M_\pi(s_1,s'_1,s_3\ldots,s_K)\\
    &=\sum_{s_1,s_3\ldots,s_K}M_\pi(s_1,s'_1,s_3\ldots,s_K)\\
    &= \mu_{\pi,2}(s'_1).
\end{align*}
We have $\mu_{\pi,1}(s_1')=\mu_{\pi,2}(s_1')$ for all $s_1' \in \Sc$. Similarly, we obtain 
\[
\mu_{\pi,1}  = \mu_{\pi,2}  =\ldots = \mu_{\pi,K} 
    =: \mu_\pi .
\]
Hence, each marginal distribution of $M_\pi$ is identical to $\mu_\pi$.
This completes the proof.
\end{proof}

We next state and prove some properties of stationary distributions.
\begin{lemma}[Property of Stationary Distribution]\label{lem: prop of stat dis}
Under the same condition of \Cref{lem:sta dis}, we have 
    \begin{itemize}
        
    \item \textbf{Stationarity condition.}
    The stationary distribution $M_\pi$ satisfies
    \begin{align}
    \sum_{s_1,\dots,s_K \in \Sc} 
    M_\pi\bigl(s_1,\dots,s_K\bigr)\,
    \pi\bigl(s' \mid s_1,\dots,s_K\bigr)
    \;=\;
    \mu_\pi(s'),
    \quad
    \forall\,s' \in \Sc. \label{Eq: Stationarity condition}
    \end{align}
    \item \textbf{Distribution of any nodes.}
    If all the $K$ parent root nodes are sampled with the stationary distribution $M_\pi$ on $\Sc^K$, then the distribution of any node is $\mu_\pi$.
    \item \textbf{Stationarity of Marginals $\mu_\pi$.} If $K$ parent root nodes are sampled from $M_\pi$, for any $\ell \in [K]$, recall that the   marginal transition kernel is defined as $\pi^\ell(s'|s_\ell):=\sum_{j\neq\ell}\sum_{s_j\in\Sc}\big[\pi(s'|s_1,\ldots,s_K)M_\pi(s_1,\ldots,s_K)/\mu_\pi(s_\ell)\big]$ per Eq.~\eqref{eqn:pi_ell}. Then for any $s_\ell,s' \in \Sc$ and any $k \in [K]$,
    $\mu_\pi(s')
        \;=\;
        \sum_{s_\ell \in \Sc}\mu_\pi(s_\ell)\,\pi^{\ell}(s' \mid s_\ell).$
    \end{itemize}
\end{lemma}

\begin{proof}
    \textbf{Result (1): Stationarity condition on $\mu_\pi$.}\\
Summing Eq.~\eqref{Eq: multi-parent stationary} over $s_1',\dots,s_{K-1}'$ yields
\begin{align*}
    \mu_{\pi}(s'_K)&=\sum_{s_1',\ldots,s_{K-1}'}M_\pi(s_1',\ldots,s_K')\\
    &=\sum_{s_1',\ldots,s_{K-1}'}\sum_{s_1,\ldots,s_K}\widetilde{\pi}\left( s'_1, s'_2, \ldots, s'_K \,\big|\, s_1, s_2, \ldots, s_K \right)M_\pi(s_1,\ldots,s_K)\\
    &\overset{\RM{1}}{=}\sum_{s_1,\ldots,s_K}\widetilde{\pi}\left( s_2,  \ldots, s_K, s'_K \,\big|\, s_1, s_2, \ldots, s_K \right)M_\pi(s_1,\ldots,s_K)\\
    & = \sum_{s_1,\ldots,s_K}\pi\left( s'_K \,\big|\, s_1, s_2, \ldots, s_K \right)M_\pi(s_1,\ldots,s_K),
\end{align*}
where $\RM{1}$ follows from that $\widetilde{\pi}\left( s'_1, s'_2, \ldots, s'_K \,\big|\, s_1, s_2, \ldots, s_K \right)$ is nonzero only when $s'_1=s_2,\ldots,s'_{K-1}=s_K$. 

\textbf{Result (2): Long-term stationarity in time.}\\
From Eq.~\eqref{Eq: multi-parent stationary}, for any integer $n \ge 1$, we have
\[
    M_\pi\bigl(s'_1,\dots,s'_K\bigr)
    \;=\;
    \sum_{s_1,\dots,s_K}
    \widetilde{\pi}^{(n)}\bigl(\,s'_1,\dots,s'_K\,\big|\, s_1,\dots,s_K \bigr)\,
    M_\pi\bigl(s_1,\dots,s_K\bigr),
\]
where recall that $\widetilde{\pi}^{(n)}$ represents the $n$-step transition kernel.  Therefore, if an order-$K$ Markov chain with transition kernel $\pi$ is initialized at $M_\pi$ on $\Sc^K$, then the joint distribution for any subsequent $K$ consecutive time steps   remains at $M_\pi$. Moreover, the marginal distribution at any single step is $\mu_\pi$.

\textbf{Result (3): Stationarity of Marginals $\mu_\pi$.}\\
By definition, for any $s_\ell,s' \in \Sc$,
\begin{align*}
    \sum_{s_\ell \in \Sc}\mu_\pi(s_\ell) \pi^\ell(s'|s_\ell)
    &
    =\sum_{s_\ell \in \Sc}\mu_\pi(s_\ell)\sum_{j\neq\ell}\sum_{s_j\in\Sc}\big[\pi(s'|s_1,\ldots,s_K)M_\pi(s_1,\ldots,s_K)/\mu_\pi(s_\ell)\big]\\
    &
    =
    \sum_{s_\ell \in \Sc}\sum_{j\neq\ell}\sum_{s_j\in\Sc}\pi(s'|s_1,\ldots,s_K)M_\pi(s_1,\ldots,s_K)\\
    &
    =\sum_{j=1}^K\sum_{s_j\in\Sc}\pi(s'|s_1,\ldots,s_K)M_\pi(s_1,\ldots,s_K)\\
    &
    =
    \mu_\pi(s'),
\end{align*}
where the last inequality follows from \eqrefn{Eq: Stationarity condition}.
\end{proof}

\subsection{Transition Kernel Preliminaries}
In this subsection, we introduce some definitions and prove some useful lemmas to prove the concentration results of order-$K$ Markov chains, which are further used to prove that our proposed estimate in Appendix \ref{sec:example_pearsons} for the $\chi^2$-KG-MI satisfies   \Cref{assp: est of F}.
% Recall the conditional distribution:
% $$\pi^k(s'|S_k)=\sum_{S_1,\ldots,S_{k-1},S_{k+1},\ldots,S_K}\pi(s'|S_1,\ldots,S_K)M_\pi^{-k}(S_1,\ldots,S_K|S_k),$$
For any $k \in [K]$, we define the  normalized and centered transition matrix for the marginal transition kernel $\pi^k$ (which was defined in Eq.~\eqref{eqn:pi_ell})
\begin{align}
        (B^k_\pi)_{s,s'}:=\sqrt{\frac{\mu_\pi(s)}{\mu_\pi(s')}}\left[\pi^k(s'|s)-\mu_\pi(s')\right]. \label{Eq: normalized and centered transition matrix}
    \end{align}
\begin{definition}[Spectral Gap of Order-$K$ Markov Chains]\label{def: Spectral Gap}
The spectral gap of an order-$K$ Markov chain with transition kernel $\pi$ and stationary marginal distribution $\mu_\pi$ is defined as $1-\lambda(\pi)$, where $\lambda(\pi):=\max_{k  \in [K]}\|B^k_\pi\|_2$.    
\end{definition}

% \begin{assumption}[Assumptions on order-$K$ Markov Chain]\label{Assp: reachabiltiy}
% Let $\pi : \Sc \times \Sc^K  \to \Rb$ be the transition kernel and $\pi^k(\cdot|\cdot): \Sc\times\Sc \to \Rb$ to be the marginal transition kernel conditioned only over the parent $k$. There exists $\gamma > 0$ such that almost surely over the draw of $\pi$:
% \begin{itemize}
%     \item (\textit{Transition lower bounded}): $\min_{S_1,\ldots,S_K,s'} \pi(s' \mid S_1,\ldots,S_K) > \gamma / S$.
%     \item (\textit{Non-degeneracy of chain}): The chain does not immediately mix to the stationary measure $\mu_{\pi}$ in one step:
%     \[
%     \sum_{s \in \Sc} \| \pi^k(\cdot \mid s) - \mu_{\pi}(\cdot) \|_2^2 \geq \gamma^2 / S.
%     \]
%     % \item (\textit{Symmetry}): For any permutation $\sigma$ on $[S]$, $\sigma^{-1} \pi \sigma \overset{d}{=} \pi$.
%     % \item (\textit{Constant mean}): $\Eb_{\Pi\sim P_\Pi}[\pi] = \frac{1}{S} \mathbbm{1}_S \mathbbm{1}_S^{\top}$.

% \end{itemize}
    
% \end{assumption}

\begin{lemma}\label{lem: bound spectral gap}
If $\pi$ satisfies $\min_{s_1,\ldots,s_K,s'} \pi(s' | s_1,\ldots,s_K) > \gamma / S$, and the Markov chain does not immediately mix to the stationary measure $\mu_{\pi}$ in one step in the sense that, for all $k \in [K]$, $\sum_{s \in \Sc} \| \pi^k(\cdot \mid s) - \mu_{\pi}(\cdot) \|_2^2 \geq \gamma^2 / S$, then we have $\lambda(\pi)\leq 1-\gamma/S$.
\end{lemma}
\begin{proof}
    For finite state space $\Sc = \{1,2,\ldots, S\}$, any $k \in [K]$, we can write $\pi^k \in \Rb^{S\times S}$ as
    \begin{align}
        \pi^k=\frac{\gamma}{S}{1}_S\mu_\pi^\top+\left(1-\frac{\gamma}{S}\right)Q_k, \label{Eq: Spectral Gap}
    \end{align}
    where $(\pi^k)_{s,s'}=\pi^k(s'|s)$, and $Q_k= \frac{1}{1-\frac{\gamma}{S}}\left(\pi^k-\frac{\gamma}{S}{1}_S\mu_\pi^\top\right)$. Next, we  verify that $Q_k$ is a stochastic matrix. Since $\pi(s'|s_1,\ldots,s_K)\geq \frac{\gamma}{S}$ for any $s_1,\ldots,s_K$, 
        \begin{align*}
            \pi^k(s'|s_k)&=\sum_{s_1,\ldots,s_{k-1},s_{k+1},\ldots,s_K}\pi(s'|s_1,\ldots,s_K)M_\pi(s_1,\ldots,s_K)/\mu_\pi(s_k)\\
            & \geq \frac{\gamma}{S}\sum_{s_1,\ldots,s_{k-1},s_{k+1},\ldots,s_K}M_\pi(s_1,\ldots,s_K)/\mu_\pi(s_k)=\frac{\gamma}{S}.
        \end{align*}
        As a result, each element in $Q_k$ is non-negative. In addition, we have 
        \begin{align*}
            1_S&=\pi^k 1_S\\
            &= \frac{\gamma}{S}1_S\mu_\pi^\top 1_S+\left(1-\frac{\gamma}{S}\right)Q_k1_S\\
            &= \frac{\gamma}{S}1_S+\left(1-\frac{\gamma}{S}\right)Q_k1_S.
        \end{align*}
        So we have 
        \begin{align*}
            Q_k 1_S=1_S,
        \end{align*}
        which together with the positivity of all elements in $Q_k$ shows that $Q_k$ is a stochastic matrix.
        
        Based on \eqrefn{Eq: Spectral Gap}, we may define
        \begin{align*}
            A_\pi^k:=\pi^k-1_S\mu_\pi^\top=\Big( 1-\frac{\gamma}{S}\Big)(Q_k-1_S\mu_\pi^\top).
        \end{align*}
        Let $\Lambda$ and $\Lambda^{-1}$ be two diagonal matrices in $\Rb^{S\times S}$, where $\Lambda_{s,s}=\sqrt{\mu_\pi(s)}$ and $\Lambda^{-1}_{s,s}=1/\sqrt{\mu_\pi(s)}$. From the definition of $B^k_\pi$  in \eqrefn{Eq: normalized and centered transition matrix}, $B_\pi^k=\Lambda(\pi^k-{1}_S\mu_\pi^\top)\Lambda^{-1}=\Lambda A_\pi^k\Lambda^{-1}$.
        \begin{align*}
            \norm{B^k_\pi}_2 &= \max_{v\in\Rb^S}\frac{\norm{B_\pi^k v}_2}{\norm{v}_2}\\
            & = \max_{v\in\Rb^S}\frac{\norm{\Lambda A_\pi^k\Lambda^{-1} v}_2}{\norm{v}_2}\\
            & = \max_{v\in\Rb^S}\frac{\norm{\Lambda A_\pi^kv}_2}{ \norm{\Lambda v}_2}\\
            & \leq 1-\frac{\gamma}{S},
        \end{align*}
        where the last inequality follows similarly to   \cite[Lemma 13]{nichani2024transformers}.\end{proof}
From the proof of \Cref{lem: bound spectral gap} and similar to    \cite[Lemma 14]{nichani2024transformers}, we have the following lemma,
\begin{lemma}\label{lem: bound of margins}
Under the  conditions as stated in \Cref{lem: bound spectral gap}, for each $k \in [K]$, we have $\min_{(s',s)\in \Sc^2}\pi^k(s'|s) \geq \frac{\gamma}{S}$ and $\min_{s\in\Sc}\mu_\pi(s)\geq \frac{\gamma}{S}$.  
\end{lemma}

% \subsection{Concentration Result on DAGs}\label{App: concentration}

% \begin{assumption}[Assumption on the graph structure]\label{Assump: graph}
% Decompose the graph as $\mathcal{G}=\bigcup_{l=1}^k\mathcal{G}_l$, where $\mathcal{G}_l$ are disjoint graphs. Let $d(i,j)$ be the length of the shortest path between $i,j$ in $\mathcal{G}$. Assume that 
%         %\item
%         for any $l\leq k$, $i\in [T]$, and $m\in \mathbb{N}$,  the number of nodes which have exact distance $m$ with $i$:  \( \#\{(  i,j )\in \mathcal{G}_l : d(i,j) = m \}  \) is upper bounded by some constant $L$. 
%     %\end{itemize}
% \end{assumption}
\subsection{Concentration results on DAGs}
In this subsection, we aim to prove the concentration result of order-$K$ Markov chains. To proceed, we first introduce the following definition of \emph{effective sequence length}. 
\begin{definition}[Effective Sequence Length]
    For \( \lambda \in (0,1) \), we define the effective sequence length \( T_{\mathrm{eff}}(\lambda) \) by:
\[
T_{\mathrm{eff}}(\lambda) := \frac{T^2}{\sum_{i,j=1}^{T} \lambda^{d(i,j)}}.
\]
\end{definition}

\begin{lemma}\label{lem: bound effective length}Decompose the DAG as $\mathcal{G}=\bigcup_{l=1}^k\mathcal{G}_l$, where $\mathcal{G}_l$ are disjoint subgraphs of $\mathcal{G}$. Let $d(i,j)$ be the length of the shortest path between $i,j$ in $\mathcal{G}$. Assume that 
for any $l\leq k$, $i\in \Gc_l$, and $m\in \mathbb{N}$,  the number of nodes in $\Gc_l$, which have exact distance $m$ with $i \in \Gc_l$ denoted as $|\{j\in \mathcal{G}_l : d(i,j) = m \}|$ is upper bounded by some constant $L$. Then,
\[
T_{\mathrm{eff}}(\lambda) \geq \frac{T(1 - \lambda)}{L}.
\]
\end{lemma}
\begin{proof}
     Note that \( T_{\mathrm{eff}}(\lambda)^{-1} \)   decomposes into a sum over  subgraphs as \( d(i,j) = \infty \) when \( i \) and \( j \) belong to disjoint subgraphs. We have
\begin{align*}
\frac{1}{T_{\mathrm{eff}}(\lambda)} &= \frac{1}{T^2} \sum_{l=1}^k \sum_{i,j \in \mathcal{G}_l} \lambda^{d(i,j)} \\
&= \frac{1}{T^2} \sum_{l=1}^k \sum_{i,j \in \mathcal{G}_l} \lambda^{d(i,j)} \\*
&= \frac{1}{T^2} \sum_{l=1}^k \sum_{i \in \mathcal{G}_l} \sum_{m \geq 0} \big|\{ j \in \mathcal{G}_l : d(i,j) = m \}\big| \lambda^m.
\end{align*}

Since \(|\{ j \in \mathcal{G}_l : d(i,j) = m \} | \) can be upper bounded by $L$,  
\begin{align*}
    \frac{1}{T_{\mathrm{eff}}(\lambda)} &\leq \frac{1}{T^2} \sum_{l=1}^k \sum_{i \in \mathcal{G}_l} \sum_{m \geq 0} L \lambda^m\\
    & \leq \frac{L\sum_{l=1}^k \sum_{i \in \mathcal{G}_l}1}{T^2(1-\lambda)}=\frac{L}{T(1-\lambda)}.
\end{align*}
This completes the proof.
\end{proof}

\begin{lemma}\label{lem: concen for joint dis}
    For any transition kernel $\pi$ with stationary marginal  distribution $\mu_\pi$, and any $i,j < T,s,s' \in \Sc$, we have
    $$|P_{S_i,S_j|\Pi=\pi}(s',s)-\mu_\pi(s)\mu_\pi(s')|\leq\sqrt{\mu_\pi(s)\mu_\pi(s')} \lambda(\pi)^{d(i,j)}.$$
\end{lemma}
\begin{proof}
    If $S_i$ and $S_j$ are independent conditionally independent given the event $\{\Pi=\pi\}$, immediately, we have 
    \begin{align*}
        P_{S_i,S_j|\Pi=\pi}(s',s)-\mu_\pi(s)\mu_\pi(s')&=P_{S_j|\Pi=\pi}[s]P_{S_i|\Pi=\pi}[s']-\mu_\pi(s)\mu_\pi(s') = 0.
    \end{align*}
    If $i$ and $j$ are not independent, then there exists a closest common parent $k$ for $i$ and $j$. Denote the distance as $d(k,i)$ and $d(k,j)$ correspondingly, and $d(i,j)=d(k,i)+d(k,j)$. Note that if $d(k,j)=0$, then $k=j$. For simplicity, for any two distinct nodes $i,j$, we denote by $\mathrm{path}(i,j)$ the shortest path from $i$
    to $j$, which consists of an ordered set of directed edges starting from $i$ and ending at $j$. For each edge $(l,m)$ in $\mathrm{path}(i,j)$, the node $l$ is one of the parents of $m$, indexed by some $k \in [K]$. We collect all such indices into an ordered sequence, denoted by $\mathrm{p}\text{-}\mathrm{pa}(i,j)$, whose order corresponds to the order of edges along $\mathrm{path}(i,j)$. Finally, we can decompose the probability from node $i$ to node $j$ as: $\Pb(S_j=s|S_i=s',\Pi=\pi)=\left(\prod_{\tilde{j}\in \mathrm{p}\text{-}\mathrm{pa}(i,j)}\pi^{\tilde{j}}\right)_{s',s}$.  As a result,   
  \begin{align*}
        &P_{S_i,S_j|\Pi=\pi}(s',s)-\mu_\pi(s)\mu_\pi(s')\\
        &\quad =\Eb\big[(\mathbbm{1}_{S_j=s}-\mu_\pi(s))(\mathbbm{1}_{S_i=s'}-\mu_\pi(s')) \, |\, \Pi=\pi\big]\\
        & \quad
        =\sum_{s_k\in \Sc}\mu_\pi(s_k)\Bigg\{\bigg(\prod_{\tilde{j}\in \mathrm{p}\text{-}\mathrm{pa}(k,j)}\pi^{\tilde{j}}\bigg)_{s_k,s}-\mu_\pi(s)\Bigg\}\Bigg\{\bigg(\prod_{\tilde{i}\in \mathrm{p}\text{-}\mathrm{pa}(k,i)}\pi^{\tilde{i}}\bigg)_{s_k,s'}-\mu_\pi(s')\Bigg\}\\
        & \quad  \overset{\RM{1}}{=}\sum_{s_k\in \Sc}\mu_\pi(s_k)\sqrt{\frac{\mu_\pi(s)}{\mu_\pi(s_k)}}  \bigg(\prod_{\tilde{j}\in \mathrm{p}\text{-}\mathrm{pa}(k,j)}B_\pi^{\tilde{j}}\bigg)_{s_k,s}  \sqrt{\frac{\mu_\pi(s')}{\mu_\pi(s_k)}}\bigg(\prod_{\tilde{i}\in \mathrm{p}\text{-}\mathrm{pa}(k,i)}B_\pi^{\tilde{i}}\bigg)_{s_k,s'}  \\
        &\quad = \sqrt{\mu_\pi(s)\mu_\pi(s')}\sum_{s_k\in\Sc}\bigg(\prod_{\tilde{j}\in \mathrm{p}\text{-}\mathrm{pa}(k,j)}B_\pi^{\tilde{j}}\bigg)_{s_k,s} \bigg(\prod_{\tilde{i}\in \mathrm{p}\text{-}\mathrm{pa}(k,i)}B_\pi^{\tilde{i}}\bigg)_{s_k,s'}\\
        &\quad = \sqrt{\mu_\pi(s)\mu_\pi(s')}\left[\bigg(\prod_{\tilde{j}\in \mathrm{p}\text{-}\mathrm{pa}(k,j)}B_\pi^{\tilde{j}}\bigg)^\top \bigg(\prod_{\tilde{i}\in \mathrm{p}\text{-}\mathrm{pa}(k,i)}B_\pi^{\tilde{i}}\bigg)\right]_{s,s'}\\
        &\quad = \sqrt{\mu_\pi(s)\mu_\pi(s')}e_s^\top\left[\bigg(\prod_{\tilde{j}\in \mathrm{p}\text{-}\mathrm{pa}(k,j)}B_\pi^{\tilde{j}}\bigg)^\top \bigg(\prod_{\tilde{i}\in \mathrm{p}\text{-}\mathrm{pa}(k,i)}B_\pi^{\tilde{i}}\bigg)\right]e_{s'}\\
        &\quad \leq \sqrt{\mu_\pi(s)\mu_\pi(s')}\|e_s\|_2\left\|\bigg(\prod_{\tilde{j}\in \mathrm{p}\text{-}\mathrm{pa}(k,j)}B_\pi^{\tilde{j}}\bigg)^\top \bigg(\prod_{\tilde{i}\in \mathrm{p}\text{-}\mathrm{pa}(k,i)}B_\pi^{\tilde{i}}\bigg)\right\|_2\|e_{s'}\|_2\\
        &\quad \leq \sqrt{\mu_\pi(s)\mu_\pi(s')}\prod_{\tilde{j}\in \mathrm{p}\text{-}\mathrm{pa}(k,j)}\left\|B_\pi^{\tilde{j}}\right\|_2 \prod_{\tilde{i}\in \mathrm{p}\text{-}\mathrm{pa}(k,i)}\left\|B_\pi^{\tilde{i}}\right\|_2\\
        &\quad \overset{\RM{2}}{\leq} \sqrt{\mu_\pi(s)\mu_\pi(s')} \lambda(\pi)^{d(i,j)},
    \end{align*}
    where $\RM{1}$ follows from the definition of $B^k_\pi$ in \eqrefn{Eq: normalized and centered transition matrix}, and $\RM{2}$ follows from \Cref{def: Spectral Gap}. 
\end{proof}
We next introduce our concentration results. We introduce a natural estimate for the stationary marginal distribution, and characterize the estimation error. 
\begin{lemma}\label{lem: concen for stat dis}
Under the same conditions as stated in \Cref{lem: bound effective length}, for any $\pi \in \mathrm{supp}(P_\Pi)$, assume  $\min_{s_1,\ldots,s_K,s'} \pi(s' \mid s_1,\ldots,s_K) > \gamma / S$. For  any subset \( I \subset [T ] \) and $s\in \Sc$, we define
\[
\hat{\mu}_{I}(s) := \frac{1}{|I|} \sum_{i \in I} \mathbbm{1}_{S_i=s}.
\]
Then,
\[
\mathbb{E}_{S_{1:T}}[\hat{\mu}_{I}(s)] = \mu_{\pi}(s) \quad \text{and} \quad \mathbb{E}_{S_{1:T}}\left[(\hat{\mu}_{I}(s) - \mu_{\pi}(s))^2\right] \leq \frac{\mu_{\pi}(s) T LS }{ |I|^2\gamma}.
\]
\end{lemma}
\begin{proof} The equality $\mathbb{E}_{S_{1:T}}[\hat{\mu}_{I}(s)] = \mu_{\pi}(s)$ is evident from the second claim of \Cref{lem: prop of stat dis}.

For the inequality, consider
    \begin{align*}
\mathbb{E}_{S_{1:T}}\left[(\hat{\mu}_{I}(s) - \mu_{\pi}(s))^2\right] &= \frac{1}{|I|^2} \sum_{i,j \in I} \mathbb{E}_{S_{1:T}}\left[\mathbbm{1}_{S_i=s} \mathbbm{1}_{S_j=s} - \mu_{\pi}(s)^2\right] \\
&\overset{\RM{1}}{\leq} \frac{\mu_{\pi}(s)}{|I|^2} \sum_{i,j \in I} \lambda(\pi)^{d(i,j)} \\
&\leq \frac{\mu_{\pi}(s)}{|I|^2} \sum_{i,j = 1}^{T} \lambda(\pi)^{d(i,j)} \\
&= \frac{\mu_{\pi}(s) T^2}{T_{\mathrm{eff}}(\lambda(\pi)) |I|^2}\\
& \overset{\RM{2}}{\leq} \frac{\mu_{\pi}(s) T L}{ |I|^2(1 - \lambda(\pi))}\\
& \overset{\RM{3}}{\leq} \frac{\mu_{\pi}(s) T LS }{ |I|^2\gamma},
\end{align*}
where $\RM{1}$ follows from \Cref{lem: concen for joint dis}, and $\RM{2}$ follows from \Cref{lem: bound effective length}, and $\RM{3}$ follows from \Cref{lem: bound spectral gap}. 
\end{proof}

As a direct application of \Cref{lem: concen for stat dis} to set $I=[T]$, we have 
\begin{corollary}\label{Coro: approx stat dis} Under the same conditions as stated in \Cref{lem: concen for stat dis}, for any $s \in \Sc$,
\[
\mathbb{E}_{S
_{1:T}}\left[(\hat{\mu}_{[T]}(s) - \mu_{\pi}(s))^2\right] \leq \frac{\mu_{\pi}(s)  LS }{ T\gamma}.
\]
\end{corollary}

\section{Discussion of $f$-\name Mutual Information}\label{app: f-mutual}
\subsection{Properties of the $f$-\name mutual information} \label{sec:properties_f-mutual}
Following the \Cref{Def: Modified MI}, we show that the properties claimed in \Cref{sec: Definition of KG-MI} are satisfied for all general $f$-KG-MIs.

\begin{property}[Properties of $f$-KG-MI]\label{Prop: 1-2} We summarize the desired properties of $f$-KG-MI.
    \begin{enumerate}
    \item \textbf{Variation with $\ell$:} The KG-MI $\tilde{I}^\ell_f(S_i; S_j)$ takes different values for different $\ell$.
    \item  \textbf{Particularization  to the Standard $f$-Mutual Information:} If $j=p(i)^\ell$, then $\tilde{I}^\ell_f(S_i; S_j)$ coincides with $f$-mutual information between $S_i$ and $S_{p(i)^\ell}$. Furthermore, $\tilde{I}^\ell_f(S_i; S_j)=0$ if $S_i$ is conditionally independent of $S_j$ given $\Pi$.
    % \item \textbf{Degenerate to Standard Mutual Information.}
    % \textcolor{blue}{$F$ is related to $f$ mutual information, furthermore, if $K=1$, then $F$ degenerate to mutual information.}
    %     \item For each $\ell$ and $i$, $\arg\max_{j \in [T]}\tilde{I}^\ell_f(S_i; S_j) \subseteq p(i)$.\newline
    % \textcolor{blue}{The mass of learned $\attn$ concentrates on the parents nodes.} 
    
    % \item $\tilde{I}^\ell_f(S_i; S_j)$ is in a population version, but it can be approximated/computed by an empirical version.\newline
    % \textcolor{blue}{Not our main focus.}
    % \textcolor{red}{\item The objective function is equivalent to some other forms such that it can be expressed as an attention layer.}
\end{enumerate}
\end{property}

\begin{proof}[Verification of \Cref{Prop: 1-2}] We show that the $f$-KG-MI has the above two properties.
\begin{enumerate}
    \item  This follows directly from the definition.

    \item         If $j=p(i)^\ell$, then $P_{S_j|\Pi}(s)\cdot\Pi^\ell(s'|s)=P_{S_i,S_j|\Pi}(s',s)$. As a result,
        \begin{align*}
            \tilde{I}^\ell_{f}(S_i; S_j)&=\Eb_{ \Pi\sim P_\Pi}\bigg[\sum_{s,s'}\frac{P_{S_j|\Pi}(s)\Pi^\ell(s'|s) P_{S_i|\Pi}(s')P_{S_j|\Pi}(s)}{P_{S_i,S_j|\Pi}(s',s)}f\Big(\frac{P_{S_i,S_j|\Pi}(s',s)}{P_{S_i|\Pi}(s')P_{S_j|\Pi}(s)}\Big) \bigg]\\
            &=\Eb_{ \Pi\sim P_\Pi}\bigg[ \sum_{s,s'} P_{S_i|\Pi}(s')P_{S_j|\Pi}(s)f\Big(\frac{P_{S_i,S_j|\Pi}(s',s)}{P_{S_i|\Pi}(s')P_{S_j|\Pi}(s)}\Big)\bigg]= {I}_{f}(S_i; S_j).
        \end{align*}
        If $S_j$ is independent of $S_i$ conditioned on $\Pi$, $P_{S_i|\Pi}(s')\cdot P_{S_j|\Pi}(s)=P_{S_i,S_j|\Pi}(s',s)$ for all $s,s'$. 
 As a result,
        \begin{align*}
            \tilde{I}^\ell_{f}(S_i; S_j)
            &=\Eb_{ \Pi\sim P_\Pi} \bigg[\sum_{s,s'} \frac{P_{S_j|\Pi}(s)\Pi^\ell(s'|s) P_{S_i|\Pi}(s')P_{S_j|\Pi}(s)}{P_{S_i,S_j|\Pi}(s',s)}f\Big(\frac{P_{S_i,S_j|\Pi}(s',s)}{P_{S_i|\Pi}(s')P_{S_j|\Pi}(s)}\Big) \bigg]\\
            &=\Eb_{ \Pi\sim P_\Pi} \bigg[ \sum_{s,s'}\!{P_{S_j|\Pi}(s)\Pi^\ell(s'|s)}f (1 ) \bigg]=0. 
        \end{align*}
\end{enumerate}
This completes the proof.
    \end{proof}

\subsection{Example of an $f$-KG-MI, the Pearson's $\chi^2$-KG-MI} \label{sec:example_pearsons}
In the following, we discuss an example of an $f$-KG-MI, the Pearson's $\chi^2$-KG-MI. We first introduce an estimation procedure for $\tilde{I}^\ell_{\chi^2}(S_i; S_j)$ and  prove that the estimate $\hat{I}^\ell_{\chi^2}(S_i; S_j)$ satisfies \Cref{assp: est of F} under mild assumption as stated in \Cref{lem: est chi2 MI}.
%\subsection{Example: Pearson's $\chi^2$-\name mutual information}\label{appd: subs-chi2}
Note that Pearson's $\chi^2$-mutual information, the convex function  $f(x)=x^2-x$ and $\tilde{I}^\ell_{\chi^2}(S_i; S_j)$  can be expressed as  %Note that there are several equivalent expressions for the Pearson's $\chi^2$-mutual information, which may result in variations in the KG-MI. Here, we select $f(x)=x^2 - x$ for illustrative purposes. 
%The corresponding Pearson's $\chi^2$-KG-MI can be expressed as
\begin{align*}
    \tilde{I}^\ell_{\chi^2}(S_i; S_j)&=\Eb_{ \Pi}\left[\sum_{s,s'}\frac{P_{S_j|\Pi}(s)\Pi^\ell(s'|s)\cdot P_{S_i|\Pi}(s')P_{S_j|\Pi}(s)}{P_{S_i|\Pi}(s',s)}f\bigg(\frac{P_{S_i|\Pi}(s',s)}{P_{S_i|\Pi}(s')P_{S_j|\Pi}(s)}\bigg)\right]\\
    &=\Eb_{ \Pi}\left[\sum_{s,s'}\frac{\Pi^\ell(s'|s)\cdot P_{S_i|\Pi}(s',s)}{P_{S_i|\Pi}(s')}-1\right]\\
    &=\Eb_{ \Pi}\left[\sum_{s,s'}\frac{\Pi^\ell(s'|s)\cdot P_{S_i|\Pi}(s',s)}{\mu_\Pi(s')}-1\right],
\end{align*}
where the last equation follows from the fact that the initial distribution is the marginal distribution.

\textbf{Estimation of Pearson's $\chi^2$-\name mutual information.}
We next discuss a feasible approach for approximating the Pearson's $\chi^2$-KG-MI.
\begin{enumerate}[wide, labelindent=0pt]
    \item The stationary marginal distribution $\mu_\Pi$  can be approximated by $\hat{\mu}_{[T]}$.
    \item The joint distribution $P_{S_i|\Pi}$  can be estimated with the empirical distribution function.
    \item To estimate the marginal transition kernel $\Pi^\ell$, we require additional information. Specifically, we extend the random sequences by appending a short series of random variables, which act as labels for supervised learning of the transition kernel, and the learning regime is similar to semi-supervised learning. This construction mirrors Task 1 in Nichani et al.~\citep{nichani2024transformers}. The random sequences are now generated as follows.
    
    \textbf{Generation of Random Sequences.} 
\begin{enumerate}
    \item First, sample $\Pi= \pi \sim P_\Pi$.
    \item  For  $i = 1, \ldots, T $, if $i_1<\ldots<i_K$ are $K$ parent root nodes,    $(S_{i_1},\ldots,S_{i_K}) \sim M_\pi$. Else, $i$ is not a parent node, sample $S_i \sim \pi(\cdot | S_{p(i)^{1}},\ldots,S_{p(i)^{K}})$.
    \item Sample  $S_{T+1},\ldots,S_{T+K}$, where $S_{T+k} \sim \text{Unif}([S])$ for all $k \in [K]$  \textit{and} $S_{T+K+1} \sim \pi(\cdot | S_{T+1:T+K})$. 
    \item Output:  $S_{1:T+K+1} = (S_1,\ldots,S_T,\ldots,S_{T+K+1})$.
\end{enumerate}
\end{enumerate}
Summarizing the estimation procedure above, if we sample $N$ sequences $\{S_{1:T+K+1}^{(n)}\}_{n=1}^N$, it is natural to propose the following estimator
\begin{align}
    \hat{I}^\ell_{\chi^2}(S_i; S_j)=\frac{1}{N}\sum_{n=1}^N \left[  \sum_{(s',s) \in \Sc^{2}} \frac{S\mathbbm{1}_{S^{(n)}_{T+K+1}=s'}\mathbbm{1}_{S^{(n)}_{T+\ell}=s}}{  \hat{\mu}^{(n)}_{[T]}(s')+\kappa }\mathbbm{1}_{S^{(n)}_i=s'}\mathbbm{1}_{S^{(n)}_j=s}-1\right], \label{Eq: est of chi2} 
\end{align}
where $\hat{\mu}^{(n)}_{[T]}(s) := \frac{1}{T} \sum_{i \in [T]} \mathbbm{1}_{S^{(n)}_i=s}$, and the positive constant $\kappa$ is added to ensure that the denominator is positive. %; it can be chosen to be arbitrarily small.

% $$\tilde{I}^\ell_{\chi^2}(S_i; S_j)=\Eb_{ \pi}\left[\sum_{s,s'}\frac{\pi^\ell(s'|s)\cdot P_{S_i|\Pi}(s',s)}{\mu_\pi(s')}-1\right].$$

In the following lemma, we show that if the number of sequences $N$ and the length of sequence $T$ are sufficiently large, and under the certain DAG conditions (which is common in the literature~\citep{nichani2024transformers,chenunveiling}) the estimate $\hat{I}^\ell_{\chi^2}$ is arbitrarily close to the ground truth $\tilde{I}^\ell_{\chi^2}$.

\begin{lemma}[Estimation Error of  Pearson's $\chi^2$-KG-MI]\label{lem: est chi2 MI}Decompose the graph as $\mathcal{G}=\bigcup_{l=1}^k\mathcal{G}_l$, suppose for any $l \leq k$, $i \in \Gc_l$, we have $|\{j\in \mathcal{G}_l : d(i,j) = m \}| \leq L$, and for any $\pi \in \mathrm{supp}(P_\Pi)$, $\min_{s_1,\ldots,s_K,s'} \pi(s' \mid s_1,\ldots,s_K) > \gamma / S$. For any $\epsilon >0$ and $0<\delta<1$, if we set $\kappa \leq \frac{1}{4}\epsilon$ in \eqrefn{Eq: est of chi2}, the length of the random sequence $T=\mathrm{poly}(S,L,\frac{1}{\gamma},\frac{1}{\kappa})$, and sample complexity $N=\mathrm{poly}(\frac{1}{\epsilon},\log(\frac{1}{\delta}),\frac{1}{\kappa},S)$, with probability at least $1-\delta$, the distance between the empirical estimate and population version of $\chi^2$-\name mutual information satisfies 
\begin{align}
    \big|\hat{I}^\ell_{\chi^2}(S_i; S_j)-\tilde{I}^\ell_{\chi^2}(S_i; S_j)\big| \leq \epsilon \qquad \forall\, i,j\in[T]. \nonumber
\end{align}
    
\end{lemma}
\begin{proof}
    Consider an expectation version of $\tilde{I}^\ell_{\chi^2}(S_i; S_j)$,
    \begin{align*}
        \bar{I}_{\chi^2}(S_i;S_j)=\mathbb{E}_{\Pi\sim P_\Pi}\Eb_{S_{1:T+K+1}|\Pi} \left[  \sum_{(s',s) \in \Sc^{2}} \frac{S\, \mathbb{I}_{T+K+1,s'}\mathbb{I}_{T+\ell,s}}{ \hat{\mu}_{[T]}(s')+\kappa }\mathbb{I}_{i,s'}\mathbb{I}_{j,s}-1\Bigg|\Pi\right],
    \end{align*}
    where for any $i \in [T+K+1]$ and $s \in \Sc$, we denote $\mathbbm{1}_{S_i=s}$ by $\mathbb{I}_{i,s}$ for simplicity. 
    
    By standard concentration results of Hoeffding inequality for bounded random variables~\citep{hoeffding1963probability}, if $N={\mathrm{poly}}(\frac{1}{\epsilon},\log(\frac{1}{\delta}),\frac{1}{\kappa},S)$, then with probability at least $1-\delta$, we have 
    \begin{align*}
        \big|\bar{I}_{\chi^2}(S_i;S_j)-\hat{I}^\ell_{\chi^2}(S_i; S_j)\big| \leq \frac{\epsilon}{2}.
    \end{align*}
Then we only need to bound the difference between $\bar{I}_{\chi^2}(S_i;S_j)$ and $\tilde{I}^\ell_{\chi^2}(S_i; S_j)$.    Consider,
     \begin{align*}
    &\big|\bar{I}^\ell_{\chi^2}(S_i; S_j)-\tilde{I}^\ell_{\chi^2}(S_i; S_j)\big|\\
    & =\left|\mathbb{E}_{\Pi\sim P_\Pi}\left[\Eb_{S_{1:T+K+1}} \left[  \sum_{(s',s) \in \Sc^{2}} \frac{S\, \mathbb{I}_{T+K+1,s'}\mathbb{I}_{T+\ell,s}}{ \hat{\mu}_{[T]}(s')+\kappa }\mathbb{I}_{i,s'}\mathbb{I}_{j,s}-1\Bigg |\Pi\right]\right]-\Eb_{ \Pi\sim P_\Pi}\left[\sum_{(s,s')\in \Sc^{2}}\frac{\Pi^\ell(s'|s)\cdot P_{S_i|\Pi}(s',s)}{\mu_\Pi(s')}-1\right]\right|\\
    & \leq \sum_{(s',s) \in \Sc^{2}} \left|\mathbb{E}_{\Pi\sim P_\Pi} \left[\Eb_{S_{1:T+K+1}}\left[  \frac{S\, \mathbb{I}_{T+K+1,s'}\mathbb{I}_{T+\ell,s}}{ \hat{\mu}_{[T]}(s')+\kappa }\mathbb{I}_{i,s'}\mathbb{I}_{j,s}-  \frac{S\, \mathbb{I}_{T+K+1,s'}\mathbb{I}_{T+\ell,s}}{{\mu}_\pi(s')}\mathbb{I}_{i,s'}\mathbb{I}_{j,s}\Big |\Pi\right]\right]\right|\\
    & \leq \sum_{(s',s) \in \Sc^{2}} \left|\mathbb{E}_{\Pi\sim P_\Pi}\left[\Eb_{S_{1:T+K+1}} \left[  \frac{S({\mu}_\Pi(s')-\hat{\mu}_{[T]}(s')-\kappa)\mathbb{I}_{T+K+1,s'}\mathbb{I}_{T+\ell,s}\mathbb{I}_{i,s'}\mathbb{I}_{j,s}}{(\hat{\mu}_{[T]}(s')+\kappa){\mu}_\Pi(s')}  \Big|\Pi\right]\right]\right|.
\end{align*}
Note that the quantity inside the expectation is $O(\frac{1}{\kappa})$ almost surely. Therefore depending on whether   $\hat{\mu}_{[T]}(s')>\frac{1}{2}\mu_{\Pi}(s')$, we have
\begin{align*}
    &\big|\bar{I}^\ell_{\chi^2}(S_i; S_j)-\tilde{I}^\ell_{\chi^2}(S_i; S_j)\big|\\ 
    & \quad \leq \sum_{(s',s) \in \Sc^{2}} \left|\mathbb{E}_{\Pi\sim P_\Pi}\left[\Eb_{S_{1:T+K+1}} \left[  \frac{S({\mu}_\Pi(s')-\hat{\mu}_{[T]}(s')-\kappa)\mathbb{I}_{T+K+1,s'}\mathbb{I}_{T+\ell,s}\mathbb{I}_{i,s'}\mathbb{I}_{j,s}}{(\hat{\mu}_{[T]}(s')+\kappa){\mu}_\Pi(s')}\mathbbm{1}_{\hat{\mu}_{[T]}(s')>\frac{1}{2}\mu_{\Pi}(s')}  \Big|\Pi\right]\right]\right|\\
    &  \qquad + \sum_{(s',s) \in \Sc^{2}} \left|\Eb_{\Pi\sim P_\Pi} \left[ \Eb_{S_{1:T+K+1}}\left[  \frac{S({\mu}_\Pi(s')-\hat{\mu}_{[T]}(s')-\kappa)\mathbb{I}_{T+K+1,s'}\mathbb{I}_{T+\ell,s}\mathbb{I}_{i,s'}\mathbb{I}_{j,s}}{(\hat{\mu}_{[T]}(s')+\kappa){\mu}_\Pi(s')}\mathbbm{1}_{\hat{\mu}_{[T]}(s')\leq \frac{1}{2}\mu_{\Pi}(s')}  \Big|\Pi\right]\right]\right|\\
    & \quad \leq \sum_{(s',s) \in \Sc^{2}} \left|\Eb_{\Pi\sim P_\Pi}\left[\Eb_{S_{1:T+K+1}} \left[  \frac{S({\mu}_\Pi(s')-\hat{\mu}_{[T]}(s'))\mathbb{I}_{T+K+1,s'}\mathbb{I}_{T+\ell,s}\mathbb{I}_{i,s'}\mathbb{I}_{j,s}}{{\mu}_\Pi(s')^2}\Big|\Pi\right]  \right]\right| \nonumber\\*
    &\qquad+\frac{S^2}{\gamma\kappa}\sum_{s' \in \Sc}\Eb_{\Pi\sim P_\Pi}\left[\Pb_{S_{1:T+K+1}}\left(\hat{\mu}_{[T]}(s')\leq \frac{1}{2}\mu_{\Pi}(s')\Big|\Pi\right)\right]+\frac{S^3\kappa}{\gamma^2}.
\end{align*}
Furthermore 
\begin{align*}
    \Pb_{S_{1:T+K+1}}\left(\hat{\mu}_{[T]}(s')\leq \frac{1}{2}\mu_{\Pi}(s')\Big|\Pi\right) & \leq \Pb_{S_{1:T+K+1}}\left(|\hat{\mu}_{[T]}(s')-\mu_{\Pi}(s')|\geq \frac{1}{2}\mu_{\Pi}(s')\Big|\Pi\right)\\
    & \leq \frac{\Eb_{S_{1:T+K+1}}[(\hat{\mu}_{[T]}(s')-\mu_{\Pi}(s'))^2|\Pi]}{\frac{\mu_{\Pi}(s')^2}{4}}\\
    & \leq \frac{4LS^2 }{ T\gamma^2},
\end{align*}
where the last inequality follows from \Cref{Coro: approx stat dis} and \Cref{lem: bound of margins}.

Additionally,
\begin{align*}
    &\mathbb{E}_{S_{1:T+K+1}} \left[  \frac{({\mu}_\Pi(s')-\hat{\mu}_{[T]}(s'))\mathbb{I}_{T+K+1,s'}\mathbb{I}_{T+\ell,s}\mathbb{I}_{i,s'}\mathbb{I}_{j,s}}{{\mu}_\Pi(s')^2}  \Big|\Pi \right]\\
    & \quad \leq \frac{\sqrt{\Eb_{S_{1:T+K+1}}[({\mu}_\Pi(s')-\hat{\mu}_{[T]}(s'))^2|\Pi]\Eb_{S_{1:T+K+1}}[(\mathbb{I}_{T+K+1,s'}\mathbb{I}_{T+\ell,s}\mathbb{I}_{i,s'}\mathbb{I}_{j,s})^2|\Pi]}}{{\mu}_\Pi(s')^2}\\
    & \quad \leq \sqrt{\Eb_{S_{1:T+K+1}}[({\mu}_\Pi(s')-\hat{\mu}_{[T]}(s'))^2|\Pi]}\frac{\sqrt{{\mu}_\Pi(s'){\mu}_\Pi(s)}\sqrt{{\mu}_\Pi(s)\Pi^\ell(s'|s)}}{{\mu}_\Pi(s')^2}\\
    & \quad \leq \frac{\sqrt{S^3L}}{\sqrt{\gamma^3T}},
\end{align*}
where the last inequality follows from \Cref{Coro: approx stat dis} and \Cref{lem: bound of margins}.

As a result,
\begin{align}
    \big|\bar{I}^\ell_{\chi^2}(S_i; S_j)-\tilde{I}^\ell_{\chi^2}(S_i; S_j)\big| \lesssim  \frac{LS^4 }{ \kappa T\gamma^3} + \frac{\sqrt{S^5L}}{\sqrt{\gamma^3T}}+\frac{S^3\kappa}{\gamma^2}\lesssim \frac{LS^4 }{ \kappa T\gamma^3}+\frac{S^3\kappa}{\gamma^2},\label{eqn:above}
\end{align}
when $\kappa \leq \frac{\gamma^2\epsilon}{4S^3}$,  and $T={\mathrm{poly}}(S,L, \frac{1}{\gamma},\frac{1}{\kappa})$, \eqref{eqn:above} is upper bounded by $\frac{1}{2}\epsilon $ as desired. 
\end{proof}

\section{Proofs of Results in Section~\ref{sec: thm}}\label{App: Proof}
In this section, we present the proofs of our main theoretical results presented in Section \ref{sec: thm} concerning the training dynamics. Recall that the updated weights in \Cref{alg:meta} is $\theta(t)=\{Q^\ell(t)\}_{\ell=1}^K$,  and for any $\ell$, the estimated gradient $\widehat{\nabla}_{\theta}L=(\hat{\partial}L_{j,i}^\ell)_{j,i,\ell}$, where $\hat{\partial}L_{j,i}^\ell$ is the empirical estimate of the derivative $\frac{\partial L}{\partial Q^{\ell}_{j,i}}$ by replacing the population $f$-KG-MI ${I}^\ell_f(S_i; S_j)$ with its estimated version $\hat{I}^\ell_f(S_i; S_j)$. Then, the gradient ascent update writes
\begin{align*}
    Q_{j,i}^\ell(t) = Q_{j,i}^\ell(t-1) + \eta \hat{\partial} L_{j,i}^\ell(t-1).
\end{align*}

% The empirical version of the objective function is 
% \begin{align*}
%     \hat{L}=-\sum_{\ell=1}^K\sum_{i,j \in [T]} \hat{F}_{j,i}^{\ell} \attn_{j,i}^{\ell}
% \end{align*}

% Use empirical gradient directly.  

\subsection{Estimation  of Gradients }
In this subsection, we aim to obtain expressions for estimates of the  gradients $\hat{\partial}L_{j,i}^\ell$. The first step is to derive their  population versions.
\begin{lemma}\label{lem: grad expre pop}
   The gradient of the objective function $L$ (defined in Eq.~\eqref{Eq: obj f}) with respect to each element of $Q$ is given by
   \begin{align}
        \frac{\partial L}{\partial Q_{j,i}^{\ell}}= \frac{1}{KT}\attn_{j,i}^\ell \sum_{k=1}^{i-1}\attn_{k,i}^\ell (\tilde{I}^\ell_f(S_i; S_j)-\tilde{I}^\ell_f(S_i; S_k)). \nonumber
   \end{align} 
\end{lemma}
\begin{proof}
    Recall that $\attn^{\ell}=\attn(\mathrm{Mask}(Q^{\ell}))$, denote $\tilde{Q}^\ell=\mathrm{Mask}(Q^{\ell})$. If $j>i$, then $\tilde{Q}^\ell_{j,i}= -\infty$ and $e^{\tilde{Q}^\ell_{j,i}}=0$. As a result, $\attn^{\ell}_{k,m}=\frac{e^{\tilde{Q}_{k,m}^\ell}}{\sum_{k=1}^{T}e^{\tilde{Q}_{k,m}^\ell}}=\frac{e^{\tilde{Q}_{k,m}^\ell}}{\sum_{k=m}^{T}e^{\tilde{Q}_{k,m}^\ell}}$.
    Consider the gradient $\frac{\partial \attn^{\ell}_{k,m}}{\partial Q_{j,i}}$, for any $j < i$, we have
\begin{align*}
    \frac{\partial \attn^{\ell}_{k,m}}{\partial Q_{j,i}}=\attn^{\ell}_{k,m}\left\{\sum_{n=1}^{T}\attn^{\ell}_{n,m}\left[I_{k,j}-I_{n,j}\right]\cdot I_{m,i}\right\}.
\end{align*}
Note that, when $m \leq k$, $\attn^{\ell}_{k,m}=0$ for all $Q_{j,i}$, we have $\frac{\partial \attn^{\ell}_{k,m}}{\partial Q_{j,i}}=0$ for all $Q_{j,i}$.

\begin{enumerate}
    \item If $m\neq i$, then $I_{i,m}=0$, and as a result, $\frac{\partial \attn^{\ell}_{k,m}}{\partial Q_{j,i}}=0$.
    \item If $m = i$, then $I_{m,i}=1$,  and as a result, 
    \begin{align*}
        \frac{\partial \attn^{\ell}_{k,m}}{\partial Q_{j,i}} &=\attn^{\ell}_{k,m}\left\{\sum_{n=1}^{T}\attn^{\ell}_{n,m}\left[I_{k,j}-I_{n,j}\right]\right\}\\
        &=\attn^{\ell}_{k,m}\left\{\sum_{n=1}^{T}\attn^{\ell}_{n,m}I_{k,j}-\attn^{\ell}_{j,m}\right\}\\
        &=\attn^{\ell}_{k,m}\left\{I_{k,j}-\attn^{\ell}_{j,m}\right\}\\
        &=\attn^{\ell}_{k,i}\left\{I_{k,j}-\attn^{\ell}_{j,i}\right\}.
    \end{align*}
    \begin{enumerate}
        \item If $j\neq k$, then $I_{k,j}=0$, as a result, $ \frac{\partial \attn^{\ell}_{k,i}}{\partial Q_{j,i}}=-\attn^{\ell}_{k,i}\attn^{\ell}_{j,i}$.
        \item If $j = k$, then $\frac{\partial \attn^{\ell}_{k,i}}{\partial Q_{j,i}}=\attn^{\ell}_{j,i}\left\{1-\attn^{\ell}_{j,i}\right\}$.
    \end{enumerate}
\end{enumerate}

As a result, we get all the values of $\frac{\partial\attn^{\ell}_{k,m}}{\partial Q_{j,i}}$ for any fixed $\ell,i,j$ (in fact we only need to consider $j < i$), plug these values into the derivative $\frac{\partial L(\theta)}{\partial Q_{j,i}^{\ell}}$, we have
\begin{align*}
    \frac{\partial L(\theta)}{\partial Q_{j,i}^{\ell}}&=\frac{1}{KT}\sum_{k,m=1}^{T}\frac{\partial\attn_{k,m}^\ell}{\partial Q_{j,i}^\ell} \cdot \tilde{I}^\ell_f(S_m; S_k) \overset{\RM{1}}{=}\frac{1}{KT}\sum_{k=1}^{i-1}\frac{\partial\attn_{k,i}^\ell}{\partial Q_{j,i}^\ell} \cdot \tilde{I}^\ell_f(S_i; S_k)\\
    &= \frac{1}{KT}\left(\attn^{\ell}_{j,i}\left\{1-\attn^{\ell}_{j,i}\right\} \cdot \tilde{I}^\ell_f(S_i; S_j)-\sum_{k\neq j}^{i-1}\attn_{k,i}^\ell\attn_{j,i}^\ell\cdot \tilde{I}^\ell_f(S_i; S_k)\right)\\
    &= \frac{1}{KT}\left(\attn^{\ell}_{j,i}\cdot \tilde{I}^\ell_f(S_i; S_j) -\sum_{k=1}^{i-1}\attn_{k,i}^\ell\attn_{j,i}^\ell\cdot \tilde{I}^\ell_f(S_i; S_k)\right)\\
    & = \frac{1}{KT}\attn^{\ell}_{j,i} \sum_{k=1}^{i-1}\attn_{k,i}^\ell (\tilde{I}^\ell_f(S_i; S_j)-\tilde{I}^\ell_f(S_i; S_k)),
\end{align*}
where $\RM{1}$ follows from that $\frac{\partial\attn^{\ell}_{k,m}}{\partial Q_{j,i}}\neq 0$ only when $m \neq i$, and $\attn_{k,i}^\ell=0$ for all $k \geq i$.
\end{proof}

%Substitute the population $f$-KG-MI with the empirical $f$-KG-MI, we get the estimated gradients. 
We obtain the estimated gradients by replacing the population $f$-KG-MIs with the empirical $f$-KG-MIs.
%\begin{lemma}[Estimated Gradient]\label{lem: est gra}
 Thus, the empirical estimate of the gradient of the objective function $L$  (defined in Eq.~\eqref{Eq: obj f})  w.r.t.\ $Q$ is given by
   \begin{align}
        \hat{\partial}L_{j,i}^\ell= \frac{1}{KT}\attn_{j,i}^\ell \sum_{k=1}^{i-1}\attn_{k,i}^\ell (\hat{I}^\ell_f(S_i; S_j)-\hat{I}^\ell_f(S_i; S_k)). \nonumber
   \end{align}
%\end{lemma}

Next we control the difference of gradients for non-root nodes. 
\begin{lemma}\label{lem: grad comp n-r}
Suppose \Cref{assp: est of F} holds and $\epsilon<\frac{1}{4}\Delta^\ell(i)$ for all $\ell$ and $i$, then for any non-root node $i$, and any  $j \in [T]$ satisfying that $\attn_{j,i}^\ell \leq \attn_{j(i)^{\ell,*},i}^\ell$, we have
\begin{align*}
    \hat{\partial}L_{j(i)^{\ell,*}}^{\ell}-\hat{\partial}L_{j,i}^{\ell} \geq \frac{1}{2KT}\attn_{j(i)^{\ell,*}}^\ell\cdot (1-\attn_{j(i)^{\ell,*}}^\ell)\Delta^\ell(i)
\end{align*}
    
\end{lemma}
\begin{proof} Consider,
    \begin{align*}
        KT &\left(\hat{\partial}L_{j(i)^{\ell,*},i}^{\ell}-\hat{\partial}L_{j,i}^{\ell} \right)\\
        & = \attn_{j(i)^{\ell,*},i}^\ell \sum_{k=1}^{i-1}\attn_{k,i}^\ell (\hat{I}^\ell_f(S_i; S_{j(i)^{\ell,*}})-\hat{I}^\ell_f(S_i; S_k))-\attn^{\ell}_{j,i} \sum_{k=1}^{i-1}\attn_{k,i}^\ell(\hat{I}^\ell_f(S_i; S_j)-\hat{I}^\ell_f(S_i; S_k))\\ 
        % & = \attn_{j(i)^{\ell,*},i}^\ell \hat{I}^\ell_f(S_i; S_{j(i)^{\ell,*}})-\attn^{\ell}_{j,i}\hat{I}^\ell_f(S_i; S_j)
        % -\attn_{j(i)^{\ell,*},i}^\ell \sum_{k=1}^{i-1}\attn_{k,i}^\ell \hat{I}^\ell_f(S_i; S_k)+\attn^{\ell}_{j,i} \sum_{k=1}^{i-1}\attn_{k,i}^\ell \hat{I}^\ell_f(S_i; S_k)\\
        % & = \attn_{j(i)^{\ell,*},i}^\ell \hat{I}^\ell_f(S_i; S_{j(i)^{\ell,*}})- \attn_{j,i}^\ell \hat{I}^\ell_f(S_i; S_{j(i)^{\ell,*}})+\attn_{j,i}^\ell \hat{I}^\ell_f(S_i; S_{j(i)^{\ell,*}})-\attn^{\ell}_{j,i}\hat{I}^\ell_f(S_i; S_j)
        % \\
        % & \quad -\attn_{j(i)^{\ell,*},i}^\ell\sum_{k=1}^{i-1}\attn_{k,i}^\ell \hat{I}^\ell_f(S_i; S_k)+\attn^{\ell}_{j,i} \sum_{k=1}^{i-1}\attn_{k,i}^\ell \hat{I}^\ell_f(S_i; S_k)\\
        & = (\attn_{j(i)^{\ell,*},i}^\ell-\attn_{j,i}^\ell)(\hat{I}^\ell_f(S_i; S_{j(i)^{\ell,*}})-\sum_{k=1}^{i-1}\attn_{k,i}^\ell \hat{I}^\ell_f(S_i; S_k))+\attn_{j,i}^\ell (\hat{I}^\ell_f(S_i; S_{j(i)^{\ell,*}})-\hat{I}^\ell_f(S_i; S_j))\\
        & \geq (\attn_{j(i)^{\ell,*},i}^\ell-\attn_{j,i}^\ell)(1-\attn_{j(i)^{\ell,*},i}^\ell)(\Delta^\ell(i)-2\epsilon) + \attn_{j,i}^\ell(\Delta^\ell(i)-2\epsilon)\\
        & \geq \frac{1}{2}\attn_{j(i)^{\ell,*},i}^\ell(1-\attn_{j(i)^{\ell,*},i}^\ell)\Delta^\ell(i),
        \end{align*}
        where the first inequality follows from the definition of $\Delta^\ell(i)$ and \Cref{assp: est of F}, and the second inequality follows from the facts that $\attn_{j(i)^{\ell,*},i}^\ell>0$, $\Delta^\ell(i) > 0$, and $\epsilon \leq \frac{1}{4}\Delta^\ell(i)$.
\end{proof}

When there exists an estimation error, the gradients for the root nodes no longer stay at $0$. It thus necessitates that we bound the difference of gradients at the root nodes. 
\begin{lemma}\label{lem: grad comp r}
   Under \Cref{assp: est of F}, for any root node $i$, we have 
    \begin{align*}
         \max_{j<i}\partial\hat{L}_{j,i}^\ell-\min_{j<i}\partial\hat{L}_{j,i}^\ell \leq \frac{4\epsilon}{KT}\max_{j<i}\attn_{j,i}^\ell.
    \end{align*}    
\end{lemma}
\begin{proof}
If $i$ is a root node, denote $j_1=\arg\max_{j<i}\partial\hat{L}_{j,i}^\ell$ and $j_2=\arg\min_{j<i}\partial\hat{L}_{j,i}^\ell$. Then similar to  the proof of \Cref{lem: grad comp n-r}, we have
    \begin{align*}
        &\max_{j<i}\partial\hat{L}_{j,i}^\ell-\min_{j<i}\partial\hat{L}_{j,i}^\ell =\partial\hat{L}_{j_1,i}^\ell-\partial\hat{L}_{j_2,i}^\ell\\
        & = \frac{1}{KT}\left((\attn_{j_1,i}^\ell-\attn_{j_2,i}^\ell)\bigg(\hat{I}^\ell_f(S_i; S_{j_1})-\sum_{k=1}^{i-1}\attn_{k,i}^\ell \hat{I}^\ell_f(S_i; S_k)\bigg)+\attn_{j_2,i}^\ell \big(\hat{I}^\ell_f(S_i; S_{j_1})-\hat{I}^\ell_f(S_i; S_{j_2})\big)\right)\\
        & \leq \frac{1}{KT}\left(\left|\attn_{j_1,i}^\ell-\attn_{j_2,i}^\ell\right|\left|\sum_{k\neq j_1}^{i-1}\attn_{k,i}^\ell .(\hat{I}^\ell_f(S_i; S_{j_1})-\hat{I}^\ell_f(S_i; S_k))\right|+\attn_{j_2,i}^\ell \cdot 2\epsilon\right)\\
        & \leq \frac{2}{KT}\left(\left|\attn_{j_1,i}^\ell-\attn_{j_2,i}^\ell\right|\left(1-\attn_{j_1,i}^\ell \right)\epsilon+\attn_{j_2,i}^\ell \cdot \epsilon\right)\\
        & \leq \frac{4\epsilon}{KT}\max_{j<i}\attn_{j,i}^\ell,
    \end{align*}
where the first and second inequalities follow from \Cref{assp: est of F} so that for any $j<i$, we have $|\hat{I}^\ell_f(S_i; S_j)|=|\hat{I}^\ell_f(S_i; S_j)-\tilde{I}^\ell_f(S_i; S_j)|\leq \epsilon$.
\end{proof}

\subsection{Attention Concentration: Proof of \Cref{Thm: Attention Concentration}}

\begin{proof}[Proof of \Cref{Thm: Attention Concentration}]
Since the elements of the proof of \Cref{Thm: Attention Concentration} are used in the proof of \Cref{Thm: loss convergence}, we first present the proof of \Cref{Thm: Attention Concentration} before that of  \Cref{Thm: loss convergence}.

We first consider \textbf{non-root nodes}. For non-root node $i$, denote $\delta^\ell_i(t)={Q}_{j(i)^{\ell,*},i}^\ell (t)-\max_{j \neq j(i)^{\ell,*}}{{Q}_{j,i}^\ell (t)}$. Recall that the gradient ascent update rule is  $Q_{j,i}^\ell(t)=Q_{j,i}^\ell(t-1)+\eta\partial L_{j,i}^\ell(t-1)$. Thus, if $\attn_{j,i}^\ell(t-1) \leq \attn_{j(i)^{\ell,*},i}^\ell(t-1)$, then we have 
\begin{align}
    \delta^\ell_i(t)-\delta^\ell_i(t-1) &= {Q}_{j(i)^{\ell,*},i}^\ell (t)-\max_{j \neq j(i)^{\ell,*}}{{Q}_{j,i}^\ell (t)}-\Big({Q}_{j(i)^{\ell,*},i}^\ell (t-1)-\max_{j \neq j(i)^{\ell,*}}{{Q}_{j,i}^\ell (t-1)}\Big)\nonumber\\
    & = ({Q}_{j(i)^{\ell,*},i}^\ell (t)-{Q}_{j(i)^{\ell,*},i}^\ell (t-1))-\Big(\max_{j \neq j(i)^{\ell,*}}{{Q}_{j,i}^\ell (t)}-\max_{j \neq j(i)^{\ell,*}}{{Q}_{j,i}^\ell (t-1)}\Big)\nonumber\\
    &\geq ({Q}_{j(i)^{\ell,*},i}^\ell (t)-{Q}_{j(i)^{\ell,*},i}^\ell (t-1))-\max_{j \neq j(i)^{\ell,*}}({{Q}_{j,i}^\ell (t)}-{{Q}_{j,i}^\ell (t-1)})\nonumber\\
    & = \eta \Big(\partial L_{j(i)^{\ell,*},i}^{\ell}(t-1)-\max_{j \neq j(i)^{\ell,*}}\partial L_{j,i}^{\ell}(t-1) \Big)\nonumber\\
    & \geq \frac{\eta}{2KT} \attn_{j(i)^{\ell,*},i}^\ell(t-1)\cdot (1-\attn_{j(i)^{\ell,*},i}^\ell(t-1))\Delta^\ell(i), \label{Eq: Thm2: lower bound incrm}
\end{align}
where the last inequality follows from \Cref{lem: grad comp n-r}. 

Hence $\delta^\ell_i(t)\geq\delta^\ell_i(t-1)\geq 0$. As a result, if $\attn_{j,i}^\ell(t-1) \leq \attn_{j(i)^{\ell,*},i}^\ell(t-1)$, then  ${Q}_{j(i)^{\ell,*},i}^\ell (t)\geq\max_{j \neq j(i)^{\ell,*}}{{Q}_{j,i}^\ell (t)}$ and $  \attn_{j(i)^{\ell,*},i}^\ell(t) \geq \attn_{j,i}^\ell(t)$. Notice that $\attn_{j,i}^\ell(0) = \attn_{j(i)^{\ell,*},i}^\ell(0)$, as a result $\attn_{j,i}^\ell(t) \leq \attn_{j(i)^{\ell,*},i}^\ell(t)$ for all $t \geq 0$.
In addition,
\begin{align*}
    \attn_{j(i)^{\ell,*},i}^\ell(t) & = \frac{e^{{Q}_{j(i)^{\ell,*},i}^\ell(t)}}{\sum_{k=1}^{i-1} e^{Q^\ell_{k,i}(t)}}\\
    & = \frac{1}{\sum_{k=1}^{i-1} e^{Q_{k,i}^\ell(t)-Q^\ell_{j(i)^{\ell,*},i}(t)}}\\
    & \geq \frac{1}{\sum_{k\neq j(i)^{\ell,*}}^i e^{-\delta(t)}+1}\\ 
    & = \frac{e^{\delta_i^\ell(t)}}{i-1+e^{\delta_i^\ell(t)}}.
\end{align*}
Hence, the lower bound of $\attn_{j(i)^{\ell,*},i}^\ell(t)$ is monotonically increasing in $\delta_i^\ell(t)$.

We consider the dynamics into two  distinct phases:
    
\textbf{First phase: until $\attn_{j(i)^{\ell,*},i}^\ell(t)$ first hits $1/2$.}
Note that with the initialization of the all-zeros matrix, all  attention scores $\attn_{j,i}^\ell(0)$ are equal to $1/i$, and the increment $\delta_i^\ell(t)-\delta_i^\ell(t-1)$ of $\delta_i^\ell$ with each iteration   is always positive by \eqrefn{Eq: Thm2: lower bound incrm}. As a result, $\delta_i^\ell(t)$ is monotonically increasing in $t$. We now have that $\attn_{j(i)^{\ell,*},i}^\ell(t)$ is monotonically increasing in $t$ and  greater than $\attn_{j(i)^{\ell,*},i}^\ell(0)=\frac{1}{i}$. Furthermore, $\delta_i^\ell(1)-\delta_i^\ell(0)=\delta_i^\ell(1)\geq\frac{(i-1)\eta\Delta^\ell(i)}{2KTi^2}$.
Denote $\tau_1$ as the last iteration such that   $\attn_{j(i)^{\ell,*},i}^\ell(t)$ is no more than $1/2$, Mathematically, 
\begin{align*}
    \tau_1^{\ell,i} := \sup\Big\{t: \attn_{j(i)^{\ell,*},i}^\ell(t) \leq \frac{1}{2}\Big\}.
\end{align*}
We next upper bound $\tau_1^{\ell,i}$. Consider,
\begin{align}
    \delta_i^\ell(\tau_1^{\ell,i})&=\delta_i^\ell(0)+\sum_{t=1}^{\tau_1^{\ell,i}} \left(\delta_i^\ell(t)-\delta_i^\ell(t-1)\right) \nonumber\\
    & \overset{\RM{1}}{\geq} \sum_{t=1}^{\tau_1^{\ell,i}}\frac{\eta}{2KT} \attn_{j(i)^{\ell,*},i}^\ell(t-1)\cdot (1-\attn_{j(i)^{\ell,*},i}^\ell(t-1))\Delta^\ell(i)\nonumber\\
    & \overset{\RM{2}}{\geq} \frac{1}{4KT}\eta\Delta^\ell(i)\sum_{t=1}^{\tau_1^{\ell,i}}\attn_{j(i)^{\ell,*},i}^\ell(t-1)\geq \frac{\tau_1^{\ell,i}}{4KTi}\eta\Delta^\ell(i),\label{Eq: thm2-1}
\end{align}
where $\RM{1}$ follows from \eqrefn{Eq: Thm2: lower bound incrm}, and $\RM{2}$ follows from that $\attn_{j(i)^{\ell,*},i}^\ell(t) \leq \attn_{j(i)^{\ell,*},i}^\ell(\tau_1^{\ell,i}) \leq 1/2$ for all $t \leq \tau_1^{\ell,i}-1$.

Obvious, $\delta_i^\ell(\tau_1^{\ell,i}) \leq \log i$, otherwise
\begin{align*}
    \attn_{j(i)^{\ell,*},i}^\ell(\tau_1^{\ell,i}) \geq \frac{e^{\delta_i^\ell(\tau_1^{\ell,i})}}{i-1+e^{\delta_i^\ell(\tau_1^{\ell,i})}} >  \frac{i}{i-1+i} >\frac{1}{2},
\end{align*}
which contradicts the definition of $\tau^{\ell,i}_1$. As a result, $\tau_1^{\ell,i} \leq \frac{4KTi\log i }{\eta\Delta^\ell(i)}$. Let $\tau_1^{\ell,i}= \frac{4KTi\log i}{\eta\Delta^\ell(i)}$.

\textbf{Second Phase: $\attn_{j(i)^{\ell,*},i}^\ell(t)$ hits $1-\varepsilon_\mathrm{attn}$}.

Denote $\tau^{\ell,i}_2$ as the first  iteration such that $\attn_{j(i)^{\ell,*},i}^\ell(t)$ is larger than  $1-\varepsilon_\mathrm{attn}$. Mathematically, 
\begin{align*}
    \tau^{\ell,i}_2 := \inf\Big\{t: \attn_{j(i)^{\ell,*},i}^\ell(t) > 1-\varepsilon_\mathrm{attn}\Big\}.
\end{align*}
Similar to \eqrefn{Eq: thm2-1}, we have 
\begin{align*}
\delta_i^\ell(\tau^{\ell,i}_2-1)& =\delta_i^\ell(\tau^{\ell,i}_1)+\delta_i^\ell(\tau^{\ell,i}_2-1)-\delta_i^\ell(\tau^{\ell,i}_1)\\
& =\delta_i^\ell(\tau^{\ell,i}_1)+\sum_{t=\tau^{\ell,i}_1}^{\tau^{\ell,i}_2-2}\left(\delta_i^\ell(t+1)-\delta_i^\ell(t)\right)\\
& \geq \frac{\tau^{\ell,i}_1}{4KTi}\eta\Delta^\ell(i)+\sum_{t=\tau^{\ell,i}_1}^{\tau^{\ell,i}_2-2}\frac{\eta}{2KT} \attn_{j(i)^{\ell,*},i}^\ell(t-1)\cdot (1-\attn_{j(i)^{\ell,*},i}^\ell(t-1))\Delta^\ell(i)\\
& \geq \log i +\frac{\varepsilon_\mathrm{attn} (\tau^{\ell,i}_2-1-\tau^{\ell,i}_1)\eta \Delta^\ell(i)}{4KT},
\end{align*}
where the last inequality follows from \eqrefn{Eq: Thm2: lower bound incrm} and the result of Phase I that $\delta_i^\ell(\tau_1^{\ell,i}) \leq \log i$. 

Obviously, $\delta_i^\ell(\tau^{\ell,i}_2-1) \leq \log(\frac{i}{\varepsilon_\mathrm{attn}})$, otherwise, 
\begin{align*}
    \attn_{j(i)^{\ell,*},i}^\ell(\tau_2^{\ell,i}-1) \geq \frac{e^{\delta_i^\ell(\tau_2^{\ell,i})}}{i-1+e^{\delta_i^\ell(\tau_2^{\ell,i})}} >  \frac{\frac{i}{\varepsilon_\mathrm{attn}}}{i-1+\frac{i}{\varepsilon_\mathrm{attn}}} > \frac{1}{1+\varepsilon_\mathrm{attn}}>1-\varepsilon_\mathrm{attn},
\end{align*}
which contradicts   the definition of $\tau^{\ell,i}_2$.
As a result,
\begin{align*}
    \log i+\frac{\varepsilon_\mathrm{attn} (\tau^{\ell,i}_2-1-\tau^{\ell,i}_1)\eta \Delta^\ell(i)}{4KT} \leq \log\Big(\frac{i}{\varepsilon_\mathrm{attn}}\Big).
\end{align*}
As a result, $\tau^{\ell,i}_2-\tau^{\ell,i}_1 \leq \frac{4KT\log\frac{1}{\varepsilon_\mathrm{attn}}}{\varepsilon_\mathrm{attn}\eta\Delta^\ell(i)}+1$.
If we let $\tau^{\ell,i}_*=\frac{4KT\log\frac{1}{\varepsilon_\mathrm{attn}}}{\varepsilon_\mathrm{attn}\eta\Delta^\ell(i)}+\frac{4KTi\log i }{\eta\Delta^\ell(i)}+1$, by the monotonicity of $\attn_{j(i)^{\ell,*},i}^\ell(t)$ in $t$, we have for all $t \geq \tau^{\ell,i}_*$,
\begin{align*}
    \attn_{j(i)^{\ell,*},i}^\ell(t) > 1-\varepsilon_{\mathrm{attn}}.
\end{align*}

    Next, we consider \textbf{root nodes}. Let $r^\ell_i(t):=r(Q_i^\ell(t))=\max_j Q_{j,i}^\ell(t)-\min_j Q_{j,i}^\ell(t)$.
    We have 
    \begin{align*}
    \max_j\attn_{j,i}^\ell(t) & = \frac{e^{\max_j{Q}_{j,i}^\ell(t)}}{\sum_{k=1}^{i-1} e^{Q^\ell_{k,i}(t)}}\\
    & = \frac{e^{\max_j{Q}_{j,i}^\ell(t)-\min_j{Q}_{j,i}^\ell(t)}}{\sum_{k=1}^{i-1} e^{Q^\ell_{k,i}(t)-\min_j{Q}_{j,i}^\ell(t)}}\\
    & \leq \frac{e^{r^\ell_i(t)}}{e^{r^\ell_i(t)}+i-2}.
\end{align*}
Similarly,
\begin{align*}
    \min_j\attn_{j,i}^\ell(t) & = \frac{e^{\min_j{Q}_{j,i}^\ell(t)}}{\sum_{k=1}^{i-1} e^{Q^\ell_{k,i}(t)}}\\
    & = \frac{1}{\sum_{k=1}^{i-1} e^{Q^\ell_{k,i}(t)-\min_j{Q}_{j,i}^\ell(t)}}\\
    & \geq \frac{1}{(i-2)e^{r^\ell_i(t)}+1}.
\end{align*}
In addition,
\begin{align}
    r_i^\ell(t)-r_i^\ell(t-1) & =\max_j Q_{j,i}^\ell(t)-\min_j Q_{j,i}^\ell(t)-\Big(\max_j Q_{j,i}^\ell(t-1)-\min_j Q_{j,i}^\ell(t-1)\Big)\nonumber\\
    & \leq \max_j \left(Q_{j,i}^\ell(t)-Q_{j,i}^\ell(t-1)\right)-\min_j \left(Q_{j,i}^\ell(t)- Q_{j,i}^\ell(t-1)\right)\nonumber\\
    & \leq \max_j \left(\eta\partial\hat{L}_{j,i}^\ell(t-1)\right)-\min_j \left(\eta\partial\hat{L}_{j,i}^\ell(t-1)\right)\nonumber\\
    & \leq \frac{4\eta\epsilon}{KT}\max_{j} \attn_{j,i}^\ell, \label{Eq: A.14-1}
\end{align}
where the last inequality follows from \Cref{lem: grad comp r}.

Fix any $w<1$. Denote the first hitting time of $r_i^\ell(t)$ at $\log(1+w)$ as 
\begin{align*}
    t^{\ell,i}_*=\inf\big\{t: r_i^\ell(t)>\log(1+w)\big\}.
\end{align*}
Case (1): If $t^{\ell,i}_* > \tau^{\ell,i}_*$, then for all $t \leq\tau^{\ell,i}_*$, $r_i^\ell(t) \leq \log(1+w)$. As a result,
\begin{align}
    &\max_j\attn_{j,i}^\ell(t) \leq \frac{1+w}{1+w+i-2}\leq \frac{1+w}{i-1},\label{Eq: thm2-2}\\
    & \min_j\attn_{j,i}^\ell(t) \geq \frac{1}{(1+w)(i-2)+1}\geq \frac{1-w}{i-1}.\label{Eq: thm2-3}
\end{align}
Combine  Eq.~\eqref{Eq: thm2-2} and Eq.~\eqref{Eq: thm2-3}, we have the following inequality for any $j$:
\begin{align}
    \left|\attn_{j,i}^\ell(t)-\frac{1}{i-1 } \right| \leq\frac{w}{i-1}. \label{Eq: thm2-4}
\end{align}
% \leq \frac{8\eta \tau^{\ell,i}_* \epsilon}{KT(i-1)^2}

\noindent Case (2): If $t^{\ell,i}_* <\tau^{\ell,i}_*$, we can obtain an upper bound for $w$ as follows.

For all $t \leq t^{\ell,i}_*<\tau^{\ell,i}_*$, we have 
\begin{align*}
    \max_j\attn_{j,i}^\ell(t) \leq \frac{1+w}{1+w+i-2}\leq \frac{1+w}{i-1}.
\end{align*}
So 
\begin{align*}
    \log(1+w) \leq r_i^\ell(t) \leq t \cdot \frac{4 \eta \epsilon}{KT} \cdot \frac{1+w}{i-1}\leq t^{\ell,i}_* \cdot \frac{4 \eta \epsilon}{KT} \cdot \frac{1+w}{i-1},
\end{align*}
where the second inequality follows from \eqrefn{Eq: A.14-1}.

By using the fact that $w<1$ and $\log(1+w)\geq \frac{w}{2}$, we further have 
\begin{align*}
    w < \frac{16\eta \tau^{\ell,i}_* \epsilon}{KT(i-1)}
\end{align*}
As a result, if we let $w=\frac{16\eta \tau^{\ell,i}_* \epsilon}{KT(i-1)}$, then Case (2) will not occur for such a choice of $w$. In other words, if we choose $w=\frac{16\eta \tau^{\ell,i}_* \epsilon}{KT(i-1)}$,   Case~(1) holds, and consequently, \eqrefn{Eq: thm2-4} also holds. 
Finally, we have
\begin{align*}
    \bigg|\attn_{j,i}^\ell(t)-\frac{1}{i-1}\bigg| \leq \frac{w}{i-1} \leq \frac{16\eta \tau^{\ell,i}_* \epsilon}{KT(i-1)^2}= 64\left(\frac{1}{ \varepsilon_\mathrm{attn} }\log{\frac{1}{\varepsilon_\mathrm{attn}}} +{i\log i }\right) \cdot \frac{1}{(i-1)^2} \cdot \frac{\epsilon}{\Delta^\ell(i)},
\end{align*}
as desired.
\end{proof}

\subsection{Objective Convergence: Proof of \Cref{Thm: loss convergence}}
% \begin{lemma}[Objective Convergence]\label{lem: real-obj-converge}
%     Denote $L^*=\max_{\theta}L(\theta)$ and $\Delta:=\min_{\ell \in [K],i\in[T]}\Delta^\ell(i)$. If $\epsilon<\frac{1}{4}\Delta$, for any error $\varepsilon_\mathrm{attn}$, let $\tau_*=\frac{4KT\log{\frac{1}{\varepsilon_\mathrm{attn}}}}{\varepsilon_\mathrm{attn}\eta\Delta}+\frac{4KT^2\log(T)}{\eta\Delta}$, for any iteration $t\geq \tau_*$,
%     \begin{align*}
%         L(\theta(t))-L^* \leq \max\{I_{\max}\varepsilon_{\mathrm{attn}},\epsilon\}.
%     \end{align*}
% \end{lemma}

\begin{proof}[Proof of \Cref{Thm: loss convergence}]
    Denote $\theta^*= \arg\max_\theta L(\theta)$. We have:
    \begin{enumerate}
        \item for non-root node $i$, $\attn_{j(i)^{\ell,*},i}(\theta^*)=1$, and $\attn_{j,i}(\theta^*)=0$ for $j \neq j(i)^{\ell,*}$.
        \item for root node $i$, any $\theta$ satisfying $\sum_{j=1}^{i-1}\attn_{j,i}^\ell(\theta)=1$ is optimal as $\tilde{I}^\ell_f(S_i; S_j)=0$ for $j<i$. 
    \end{enumerate}
    Then $L^*=\frac{1}{KT}\sum_{\ell=1}^K\sum_{i,j \in [T]} \tilde{I}^\ell_f(S_i; S_j) \attn_{j,i}^{\ell}(\theta^*)= \frac{1}{KT}\sum_{\ell=1}^K\sum_{i \in \mathcal{R}} \tilde{I}_f^{\ell}(S_i;S_{j(i)^{\ell,*}})$.
    
    Denote the set of root nodes as $\mathcal{R}$. Then, we have 
\begin{align*}
    L^*-L(\theta(t))& =  \frac{1}{KT}\sum_{\ell=1}^K\sum_{i,j \in [T]} \tilde{I}^\ell_f(S_i; S_j) (\attn_{j,i}^{\ell}(\theta^*)-\attn_{j,i}^{\ell}(\theta(t)))\\
    & \leq \frac{1}{KT}\sum_{\ell=1}^K\left(\sum_{i \notin \mathcal{R}} \bigg(\tilde{I}_f^{\ell}(S_i;S_{j(i)^{\ell,*}})\varepsilon
    _\mathrm{attn}-\sum_{j\neq j(i)^{\ell,*}}\tilde{I}^\ell_f(S_i; S_j) \attn_{j,i}^{\ell}(\theta(t))\bigg)+\sum_{i \in \mathcal{R}}\epsilon\right)\\
    & \leq \frac{1}{KT}\sum_{\ell=1}^K\left(\sum_{i \notin \mathcal{R}} \tilde{I}_f^{\ell}(S_i;S_{j(i)^{\ell,*}})\varepsilon
    _\mathrm{attn}+\sum_{i \in \mathcal{R}}\epsilon\right) \leq \max\{I_{\max}\varepsilon_{\mathrm{attn}},\epsilon\} ,
\end{align*}
which is the desired convergence result.
\end{proof}

\subsection{Attention Concentration for KL-KG-MI: Proof of \Cref{coro: KL conver}}\label{appd: subs: KL}
\begin{proof}[Proof of \Cref{coro: KL conver}]
Now we consider the convex function $f(x)=x\log x$ in which case the $f$-KG-MI particularizes to the KL-KG-MI. Recall that, for each node $i$, the marginal distribution of the random variable $S_i$ is $\mu_\Pi$. As a result, $\tilde{I}^\ell_\mathrm{KL}(S_i; S_j)$  can be expressed as 
\begin{align*}
    \tilde{I}^\ell_\mathrm{KL}(S_i; S_j)=\Eb_{\Pi\sim P_\Pi}\left[\Eb_{S_j\sim \mu_{\Pi}(\cdot),S_i\sim\Pi^\ell(\cdot|s)}\left[\log\frac{P_{S_i,S_j|\Pi}(S_i,S_j)}{\mu_{\Pi}(S_i)\mu_{\Pi}(S_j)}\bigg|\Pi\right]\right]
\end{align*}
    As a result of \Cref{Thm: Attention Concentration}, we only need to show that for the KL-KG-MI, for each $\ell$ and $i$, $\arg\max_{j \in [T]}$\hspace{0.1pt} $\tilde{I}^{\ell}_{\mathrm{KL}}(S_i;S_j) = p(i)^\ell$. 
        \begin{align*}
        \operatorname*{argmax}_{j} \tilde{I}^{\ell}_\mathrm{KL}(S_i;S_j) 
          &= \operatorname*{argmax}_{j}\Eb_{\Pi\sim P_\Pi}\left[\Eb_{S_j\sim \mu_{\Pi}(\cdot),S_i\sim\Pi^\ell(\cdot|S_j)}\left[\log\frac{\Pb(S_i,S_j)}{\mu_{\Pi}(S_i)\mu_{\Pi}(S_j)}\bigg|\Pi\right]\right]\\
        &  = \operatorname*{argmin}_{j} - \Eb_{\Pi\sim P_\Pi}\left[\Eb_{S_j\sim\mu_{\Pi}(\cdot),S_i\sim\Pi^\ell(\cdot|S_j)}\left[\log\frac{\mu_X(S_j)\Pi^\ell(S_i|S_j)}{ P_{S_i,S_j|\Pi}(S_i,S_j)}\bigg|\Pi\right]\right]\\
        &  
        =  \operatorname*{argmin}_{j} \Eb_{\Pi\sim P_\Pi}[ D_\mathrm{KL}(\mu_\Pi(\cdot)\cdot\Pi^\ell(\cdot|\cdot)\| P_{S_i,S_j|\Pi}(S_i=\cdot,S_j=\cdot))]\\
        & 
        =  j\qquad \mbox{s.t.}\quad \mu_{\Pi}(\cdot)\Pi^\ell(\cdot|\cdot) = P_{S_i,S_j|\Pi}(S_i=\cdot,S_j=\cdot)\\
        & =  p(i)^\ell.
    \end{align*}
    We observe that in the case of (KL) mutual information, identifying the node $S_j$ maximizing $\tilde{I}^\ell_{\mathrm{KL}}(S_i; S_j)$ is equivalent to identifying the node $S_j$ whose joint distribution with $S_i$ is closest to $P_{S_i|\Pi}(\cdot)\,\Pi^\ell(\cdot\!\mid\!\cdot)$ in the sense of the KL divergence.
    \end{proof}

\end{appendices}

\bibliography{ref}
\bibliographystyle{ieeetr}

\end{document}